\documentclass{article}
\usepackage{kr}
\usepackage{pdfpages}
\usepackage{thm-restate}
\usepackage{svg}
\usepackage{multirow}
\usepackage{times}
\usepackage{soul}
\usepackage{url}
\usepackage[hidelinks]{hyperref}
\usepackage[utf8]{inputenc}
\usepackage[small]{caption}
\usepackage{graphicx}
\usepackage{amsmath}
\usepackage{amsthm}
\usepackage{booktabs}
\usepackage{algorithm}
\usepackage{algorithmic}
\usepackage{microtype}
\usepackage{mysc}
\usepackage{amssymb,amsmath,amsthm,mathtools}
\usepackage{todonotes}
\usepackage[parfill]{parskip}
\usepackage[shortlabels]{enumitem}
\usepackage{booktabs}
\usepackage[capitalize]{cleveref}
\usepackage{xspace}
\usepackage{url}
\usepackage{mathrsfs}
\usepackage{natbib}
\newtheorem{lemma}{Lemma}

\newtheorem{proposition}{Proposition}
\newtheorem{corollary}{Corollary}
\newcommand{\repeatcaption}[2]{%
  \renewcommand{\thefigure}{\ref{#1}}%
  \captionsetup{list=no}%
  \caption{#2}%
  \addtocounter{figure}{-1}% So that next figure after the repeat gets the right number.
}

\newtheorem{definition}{Definition}
\newcommand{\canonmodel}{\ensuremath{\Imc_\Kmc}\xspace}
\newcommand{\World}{{\Omega}}

\newcommand{\worldSize}{{\mathbf{s_\World}}}
\newcommand{\worldSizeScalar}{{s_\World}}

\newcommand{\sComp}[1]{{\overline{{#1}}}}

\newcommand{\headAssign}[1]{{{\sf Head}(#1)}}
\newcommand{\headAssignI}[2]{\ensuremath{{{\sf Head}_{#2}(#1)}}\xspace}
\newcommand{\inv}[1]{#1^-}
\newcommand{\tailAssign}[1]{{{\sf Tail}(#1)}}
\newcommand{\tailAssignI}[2]{\ensuremath{{{\sf Tail}_{#2}(#1)}}\xspace}
\newcommand{\bumpBoxAssign}[1]{{{\sf Bump}(#1)}}
\newcommand{\bumpBoxAssignI}[2]{\ensuremath{{{\sf Bump}_{#2}(#1)}}\xspace}
\newcommand{\posAssign}[1]{{{\sf pos}(#1)}}
\newcommand{\bumpAssign}[1]{{{\sf bump}(#1)}}

\newcommand{\cAssignI}[2]{\ensuremath{\cAssignSymb_{#1}(#2)}\xspace}

\newcommand{\cAssignSymb}{{\eta}}
\newcommand{\cAssign}[1]{{\cAssignSymb(#1)}}

\newcommand{\Lmcp}{\ensuremath{\mathcal{L}}\xspace}

\newcommand{\leqd}{\mathrel{{\leq_{d}}}}

\newcommand{\geqd}{\mathrel{{\geq_{d}}}}
\newcommand{\Rd}{{{\mathbb{R}}^{d}}}

\renewcommand{\NC}{\ensuremath{{\sf N_C}}\xspace}
\newcommand{\NCE}{\ensuremath{{\sf N_C^\exists}}\xspace}
\newcommand{\NCEn}{\ensuremath{{{\sf N}}^{\exists  }_{\sf C\sqcap}}\xspace}
\newcommand{\NCEneg}{\ensuremath{{{\sf N}}^{\exists  }_{\sf C\neg}}\xspace}
\renewcommand{\NR}{\ensuremath{{\sf N_R}}\xspace}
\newcommand{\NRm}{\ensuremath{{\sf N_R^-}}\xspace}
\renewcommand{\NI}{\ensuremath{{\sf N_I}}\xspace}
\newcommand{\convexSet}{{\ensuremath\mathcal{C}}\xspace}
\newcommand{\B}{{\sf Box}}

\newcommand{\hyperP}{\ensuremath{\boldsymbol{\lambda}}}
\newcommand{\hyperPs}{\ensuremath{\bar{\boldsymbol{\lambda}}}}

\newcommand{\modelName}{BoxLitE\xspace}
\newcommand{\geometric}{box\xspace}
\newcommand{\ontoo}{\Kmc\xspace}

\newcommand{\Sempty}{\ensuremath{I_{0}}\xspace} 
\newcommand{\Seq}{I_{=}} 
\newcommand{\Ssub}{I_{\subset}}
\newcommand{\Ssup}{I_{\supset }}
\newcommand{\Snsup}{I_{ \not\supset }}
\newcommand{\Sinter}{I_{ \cap }}

\newcommand{\PRa}{\mathcal{P}_{\mathsf{R}}}
\newcommand{\PRna}{\mathcal{P}^{\neg}_{\mathsf{R}}}
\newcommand{\BRa}{\mathcal{B}_{\mathsf{R}}}
\newcommand{\BRna}{\mathcal{B}^{\neg}_{\mathsf{R}}}
\newcommand{\BBRa}{\mathcal{B}_{\mathsf{R},\mathsf{I}}}
\newcommand{\BBRna}{\mathcal{B}^{\neg}_{\mathsf{R},\mathsf{I}}}
\newcommand{\BB}{\ensuremath{I_0}\xspace}
\newcommand{\PC}{\mathcal{P}_{\mathsf{C}}}
\newcommand{\BC}{\mathcal{B}_{\mathsf{C}}}
\newcommand{\PnC}{\mathcal{P}^{\neg}_{\mathsf{C}}}
\newcommand{\SC}{I }
\newcommand{\Qasub}{\mathcal{S}_{\subset}}
\newcommand{\Qnasub}{\mathcal{S}^{\neg}_{\subset}}

\newcommand{\BRavalue}{\ensuremath{0}\xspace}
\newcommand{\BRnavalue}{\ensuremath{-1}\xspace}

\newcommand{\PCvalue}{\ensuremath{-1}\xspace}
\newcommand{\PnCvalue}{\ensuremath{1}\xspace}
\newcommand{\BCvalue}{\ensuremath{0}\xspace}

\newcommand{\epsilonVec}{\ensuremath{\boldsymbol{\epsilon}}\xspace}
\newcommand{\epsilonMax}{\ensuremath{\epsilon_\mathit{max}}\xspace}

\newcommand{\POC}[2]{{\mathcal{K}_{#1}^{#2}}}
\newcommand{\intervalSkip}{1.25}
\newcommand{\subSkip}{0.25}
\def\dist{\mathop{\mathsf{dist}_e}}
\def\sdist{\mathop{\mathsf{sdist}}}

\newcommand{\norm}[1]{\|#1\|}
\renewcommand{\Re}{\mathbb{R}}
\newtheorem{claim}
{Claim}
\newcommand{\element}{\ensuremath{c}\xspace}
\newcommand{\anotherelement}{\ensuremath{e}\xspace}
\newcommand{\inProd}[2]{\langle #1 , \, #2 \rangle }
\newcommand{\PRavalue}{\ensuremath{-1 + \epsilon}\xspace}
\newcommand{\PRnavalue}{\ensuremath{1 - \epsilon}\xspace}

\title{BoxLitE: A Faithful Knowledge Base Embedding Based on Convex Optimization}
\author{%
Bruno F. Louren\c{c}o$^1$\and
Hesham Morgan$^2$ \and
Ana Ozaki$^{3}$\and
Aleksandar Pavlovi\'{c}$^4$\and
Emanuel Sallinger$^2$ \\
\affiliations
$^1$The Institute of Statistical Mathematics, Japan\\
$^2$TU Wien, Austria\\
$^3$University of Oslo, Norway\\
$^4$University of Applied Sciences Campus Vienna, Austria\\
\emails
bruno@ism.ac.jp,
hesham.morgan@tuwien.ac.at,
anaoz@uio.no,
aleksandar.pavlovic@hcw.ac.at,
sallinger@dbai.tuwien.ac.at
}

\newif\ifFullVersion
\FullVersiontrue %\FullVersionfalse

\begin{document}
%\includepdf{paper_title_page.pdf}
\maketitle

\begin{abstract}
%\todo{
%-1. Finish the main body first.
%0. Look through the main body and think about what needs to be reworked.
% 1. Look at KR reviews
% 1.3. Add path proof
%1.3. Rework convex formulation (loss terms, distance function, scoring function)
%1.5. Add negative sampling
%1.9. Convex Box: Describe which parts of BoxLitE need to be removed
%2. Rework Introduction
%3. Appendix Changes: 
%3.1 Add something about dimensionality constraints (disjointness for concepts per dimension - proof this based on the weak faithfulness proof, by showing that these settings may correspond to a dimension in the proof, thus it should not destroy anything) 
% 3.2 change Weak Faithfulness (proof Claim 16 in a proper way)
%3.2 Add proof of distance function 
%}
Knowledge base (KB) embeddings aim
at combining the  
capability of classical knowledge graph embeddings to generalize the information present in facts, the ABox, with conceptual knowledge represented  in an ontology language, the TBox. 
Several authors have recently explored the idea of mapping    concepts to \emph{convex regions} in a vector space. 
 This is  useful to represent    hierarchies, typically present in TBoxes,  since
more general concepts can be mapped to larger regions, containing those regions 
associated with more specific concepts.
However, the power of convexity is rarely leveraged during the actual learning tasks. 
Here, we 
introduce \modelName, a  {KB embedding}  {model}
for DL-Lite$^\Hmc$  that allows for  {convex optimization}.
We show  that for any satisfiable DL-Lite$^\Hmc$ KB, there is a \modelName embedding that is a weakly faithful model. As a proof of concept, we show how to formulate the KB embedding task as a convex optimization problem and how to obtain embeddings with such desirable faithfulness property. % in practice.  %using   CVXPY and Gurobi. %Our approach allows us to enforce axioms in the TBox while leaving 
%In contrast to commonly used unconstrained optimization approaches, our approach allows for the use of 
%
%
%we can constrain the model so as to enforce TBox axioms to hold.
%enforce strong faithfulness using convex constraints.
%(independently of having loss $0$)
%We consider the notion of weak and strong faithfulness
%and show that
%provides theoretical faithfulness properties, and guarantees some of these faithfulness properties in practice. 
\end{abstract}

\section{Introduction}\label{sec:intro}

%\todo{Rewise the intro to newer works}

%\todo{Maybe revise the intro based on the property that any satisfied fact has a negative score and any unsatisfied fact has a positive score aligning with the semantics of our geometric models.}
%already too much to say

%\todo{Storyline:
%0.5. KB embeddings are important
%0.7 Region-based KB embeddings are important
%0. DL-Lite is important for KGs
%1. There hasn't been any DL-Lite$^H$ KG embedding in general
%2. In addition, SGD does not allow for enforcing T-Box constraints %skip, some works enforce
%3. SGD is not used in a theoretically sound
%number 3 a bit too big to argue
%we now put number 5 earlier
%4. We propose to use convex optimization, where TBox axioms where TBox axioms can be naturally represented as constraints
%5. Problems: While many geometric embedding models use convex shapes, their objective functions are non-convex, due to: (1) negative sampling, (2) design of regularization terms, (3) ontology language.
%6. Finally, most KB embeddings methods do not have property XYZ (KR24), we close these gaps
%6. Contributions
%}

%%%0.5. KB embeddings are important
Knowledge base (KB) embeddings 
%are intended to 
combine the  
capability of %classical 
knowledge graph embeddings to  %generalize patterns present in facts 
perform inductive reasoning 
for link prediction 
with %the capability of %performing 
 deductive reasoning, using  logic expressions present in an ontology~\citep{DBLP:conf/kr/Bourgaux0KLO24}. 
%using ontologies.
%conceptual knowledge represented  in an ontology language, the TBox. 
%This results in
%That is, they are 
%a machine learning technique, called KB embedding model (KBE), that is particularly interesting to the KR community, with a strong focus on ontologies. 
%the ontology is respected by the embedding model.
%Faithfulness come
%Moreso as we as a community, and particularly in the \textit{reasoning, learning and decision making} track are targeting works combining aspects of KR and machine learning (ML) research, including the integration of learning and reasoning at modeling or solving. 
%aim
%at combining the inductive reasoning capability of classical knowledge graph embedding models with  conceptual knowledge, represented in an ontology language. 
Several authors have recently explored the idea of mapping concepts in a KB to  regions in a vector space~\citep{GemoetricKGEs,AMW2023,ReshufflE,FHK26}. %In a vector space,
Region-based embeddings are important for KBs as they can naturally represent hierarchies: more general concepts can be mapped to larger regions, containing those regions 
associated with more specific concepts. 

Although concepts  are usually mapped to \emph{convex regions} in KBEs, e.g.,  balls, boxes, and cones \citep{DBLP:conf/ijcai/KulmanovLYH19,BoxE,BoxEL,ExpressivE,OLW20}, the power of convexity is rarely leveraged during the actual learning tasks. 
%In fact, in the cone semantic approach~\citep{OLW20}, convex optimization is cited as one of the motivations for using cones.
%(see Section~4.2 therein).  
While being convex is not synonymous with ``easy'', convexity brings a number of theoretical and practical benefits, in particular: all local minima must be global. Under convexity, there are a number of efficient  algorithms for many classes of convex problems such as linear programs and second-order cone programs (SOCPs) \citep{NN94,BtN01}. 
%certain geometric quantities, e.g., \cite[Section~3.8]{Re01}. 

%DBLP:conf/aib/BouraouiGS22 can be added if there is space
For KBs expressed in description logic, most of the literature work on region-based embeddings focuses on the \EL ontology language
%s such as   \EL and %$\mathcal{ALC}$
~\citep{YCS25,Box2EL,FAITHEL,BoxEL,DBLP:conf/ijcai/KulmanovLYH19,EmEL,DBLP:journals/corr/abs-2202-14018,DBLP:journals/tgdk/LacerdaO024}, while some consider $\mathcal{ALC}$ \citep{OLW20,LOW22}.
%%%%%0. DL-Lite is important for KGs
However, in practice, ontologies present in large-scale KBs tend to use features %the constructs 
in the DL-Lite$^\Hmc$ ontology language~\citep{dllite-jair09}, due to its simple but versatile
expressivity, featuring low computational complexity~\citep{dllite-jair09}. 
%have simple logical constructs
%such as Wikidata
%1. There hasn't been any DL-Lite$^H$ KG embedding in general
Despite its broad use, % of DL-Lite$^\Hmc$, 
only a few works investigate logics in the DL-Lite family in a KB embedding context~\citep{app122010690,FamilyDataset,DBLP:conf/dlog/BourgauxOP21} and none provide a region-based embedding implementation.  % for DL-Lite. %,
% and quasi-chained rules, 
%where concepts are mapped into regions in a vector space. %not sure if we should cite box2el since it seems it is only on arxiv at the moment
%This is useful to represent    hierarchies, typically present in ontologies,  since
%3. SGD is not used in a theoretically sound
%our link prediction and theoretically sound
% \todo{wip}
%CosE 
%2. In addition, SGD does not allow for enforcing T-Box constraints
%The requirement for the regions to be    convex  is often motivated   %for practical purposes.
%\todo{add a paragraph motivating geometric embeddings directly}
%
%
%,  with concepts  mapped to \emph{convex regions}
%
%Convexity and geometric embeddings have an interesting story.
%\todo{wip}
%Convexity and region-based embeddings have an interesting story.
%However, the mere usage of convex shapes does not bring any practical benefit if one does not appropriately leverage the power of convexity during optimization and learning tasks. 
%To the best of our knowledge, % it is safe to say that
%Most  previous   embedding approaches using convex shapes use \emph{nonconvex} optimization problems during the learning phase and are   vague on how \emph{exactly} could convexity be leveraged. 
%This seems to be a missed opportunity.

In this paper, we  
introduce \modelName, a  {KB embedding}  {model} 
for DL-Lite$^\Hmc$ KBs  that allows for  {convex optimization}. 
In particular, %We also 
we study  the notion of  \textit{weakly faithful models}~\citep{OLW20,DBLP:conf/kr/Bourgaux0KLO24}
in the optimization task. This notion states that axioms that hold in the embedding are \emph{consistent} with the KB %(weakly faithful) 
and \emph{entailments of the KB are satisfied} in the embedding.
%(it is a model).   %\todo{add weak}
%, which can be seen as a counterpart to the  notion of canonical models.
%Faithfulness is a property of embedding models that comes with two flavours: \emph{strong} faithfulness,\todo{B: This parts needs to be adjusted as we no longer have strong faithfulness stuff.} which says that axioms that hold in the embedding are a logical consequence of the KB; and \emph{weak} faithfulness, which says axioms that hold in the embedding are consistent with the KB.  
In detail,
\begin{itemize}
    \item we select DL-Lite$^\Hmc$ an \textbf{ontology language} that allows to exploit the \textbf{advantages of convex optimization};
    %:   languages that allow conjunction on the left-hand side of concept inclusions, such as   \EL and $\mathcal{ALC}$ lead to issues (discussed in Secton~\ref{sec:discussion}), hence we consider ;
    %and its extension with role hierarchies, called DL-Lite$^\Hmc$;%(which extends  DL-Lite with role hierarchies);   
    \item we \textbf{design, implement, and evaluate}  \modelName, a KB embedding approach that allows for convex optimization; 
    \item we prove that every satisfiable DL-Lite$^\Hmc$ has a \modelName embedding  that is a \textbf{weakly faithful model};
    \item we introduce a novel and efficient form of \textbf{negative sampling based on existential concepts};
%    \todo{Add, we introduce a novel and efficient form of negative sampling on existential concepts}
    %and uses %well-established 
%and 
%robust 
%optimization tools such as  CVXPY \citep{DB16,AVDB18} and Gurobi \citep{Gurobi23};

%    \item to ensure that  {instances of \modelName (called \geometric interpretations $\cAssignSymb$)} reflect the knowledge of the KB,
%we consider the notion of  \textbf{weak and strong faithfulness}~\citep{OLW20}. 
%In a nutshell, a strongly faithful embedding model $\cAssignSymb$
%satisfies exactly
%the axioms entailed by a KB; more interestingly, the conceptual knowledge part of it,  %technically 
%called the `TBox'; while weak faithfulness allows new axioms to hold in  {$\cAssignSymb$}, as long as they do not contradict the KB.  %knowledge available. 
%This is important since we  want the  embedding models to generalize from the data in the KB, the `ABox';

%\item \todo{Say something about maybe entailed + weak faithfulness}
% \item we show that every DL-Lite$^\Hmc$  KB has a \textbf{strongly faithful} \modelName   model; 

%\todo{rethin the following}
%\item interestingly, we show how to extend \modelName's problem formulation such that for any satisfiable DL-Lite$_{core}$ KB, there is a \modelName model that is weakly KB faithful. %\todo{@Bruno: Do you have any feeling about ``òptimal \geometric interpretations''.} 
% is \textbf{strongly faithful} w.r.t. the TBox part of the KB and \textbf{weakly faithful} w.r.t.\ its ABox; % using convex optimization. 
%We use convex regions as this allows to optimize our model using convex optimization, which 
\item in contrast to commonly used unconstrained optimization approaches,
our \modelName approach can enforce the \textbf{TBox axioms using convex constraints} while leaving the ABox terms in the objective function.
%, for KB completion. %\todo{Somewhere in the last two paragraphs it would be nice to use the BoxLitE name}
%allowing for generalization.
%does not solely guarantee to find optimal solutions () %of the field) 
%but also 
%Briefly speaking, weak faithfulness requires the model to satisfy the KB and 
%also other axioms of the language as long as they are consistent with the KB, while
%strong faithfulness require that the model expresses exactly those axioms that are a logical consequence of the KB. 
%We first show that every DL-Lite$^\Hmc$ KB (which extend DL-Lite with role hierarchies) has a strongly faithful 
%model in the class of geometric models we introduce. 
\end{itemize}
\textbf{Main Contribution. }
Our main contribution lies on theoretical grounds, %, with far from trivial results, 
establishing the existence of  {a faithful model for DL-Lite$^\Hmc$
%box interpretations 
and the proposal of a novel approach that computes embeddings for DL-Lite$^\Hmc$ KBs.} \modelName's implementation and the empirical results serve as a proof of concept that our theoretical results translate into practical settings.
%\textbf{Goals.}
One of the motivations for this work is to
try to solve  link prediction  
%for DL-Lite$^\Hmc$ KBs 
using theoretically sounds techniques from a convex optimization perspective. 
%In particular, 
We 
would like to check how much performance can we get from a purely convex approach. In this way, our work stands in contrast to previous approaches that   relied {on nonconvexity,
yielding %which correspond to 
more complicated optimization problems. }
%We provide more details on \todo{...}

There are three primary sources of nonconvexity in contemporary KB embedding methods (see more details in Section~\ref{sec:discussion}): 
%\begin{itemize}
  %  \item 
    %$(a)$ 
    the way \emph{negative sampling} is %which is used in several works, e.g.,~
    used by most works \citep{TransE,RotatE,DistMult,ComplEx,SimplE,TuckER,BoxEL,BoxE,SpeedE};
%\item 
%$(b)$ 
the \emph{design of loss terms} involving operations that do not preserve convexity; 
%\item 
%$(c)$ 
and the choice of the \emph{ontology language}, in particular, languages that allow conjunction on the left-hand side of concept inclusions, such as   \EL and $\mathcal{ALC}$, 
%, and quasi-chained rules, 
are likely to lead to nonconvexity.
%\end{itemize}
This motivated us to consider   
%DL-Lite$_{core}$ and 
DL-Lite$^\Hmc$~\citep{dllite-jair09}, 
%and its extension with role hierarchies, called DL-Lite$^\Hmc$, 
which is a  simple but practically useful and well-known 
ontology language without conjunction on the left-hand side.

%\textbf{Differentiability.} 
On the algorithmic side, 
%we point out that 
{ADAM} \citep{KB15}, the method of choice of several works, is often used in theoretically unsound ways. 
For example, ADAM requires the objective function to be differentiable; however, it is not uncommon to see this requirement being ignored. 
For instance, in \cite[Sections~4.3, 4.5]{BoxEL} %\todo{We may need to update the reference to BoxEL} 
and \cite[Section~4] %and Appendix~F.2]
{BoxE}, the authors describe nondifferentiable objective functions that are optimized via ADAM\footnote{{ In \citep{BoxEL}, for example, loss terms for concepts assertions are expressed using the 2-norm and the regularization term considered therein use a maximum of functions. Both correspond to terms that are not differentiable in general, e.g., 
the 2-norm function is not differentiable at the origin, which can be checked directly by the definition of Fr\'echet differentiability.
}}. 
%\todo{argue why}
 The rationale for that is not explained in the papers, but a reasonable guess seems to be that these functions are typically differentiable almost everywhere (in a measure theory sense), so there may be an underlying belief that iterates  are unlikely to reach a point of nondifferentiability and the algorithm may still work in practice.   
 {Besides, widely used tools such as PyTorch may %will happily 
attempt to differentiate
%\todo{Is it ok to add ``'' or is it too spicy?} 
nondifferentiable functions for the user by, e.g., selecting subgradients/supgradients when available~\citep{pytorchgrad2026}. 
 } %which is often used in implementations, happily deals with nondifferentiability by selecting an an element of subgradient (if it exists) or having undefined behavior. }
 %{\color{red} Furthermore, pytorch, which is often used in implementations, happily deals with nondifferentiability by selecting an an element of subgradient (if it exists) or having undefined behavior. }
Unfortunately, there are well-known examples in the optimization literature   showing that when a gradient method is applied to a nondifferentiable function, it may fail to find an optimal solution, even if the function is differentiable at the points generated by the method, e.g., see \cite[Section~8.1.2]{Beck17}. In the optimization community, some works explore the convergence properties of SGD methods when applied to nondifferentiable functions \citep{PWE20,DDKL19}, but, as far as we know, similar results have not been proven for ADAM.   % by Wolfe.
%We should also mention %in passing 
%that it has been shown that the main convergence theorem in the original ADAM paper \cite[Theorem~4.1]{KB15}, which incidentally %by the way, 
%was stated for convex functions only, is flawed and there are known counterexamples~\cite[Section~3]{RKK18}.  
%In order to actually get convergence guarantees for a given problem, tuning %of 
%the moment parameters may be required %, as %was recently pointed out~
%\citep{ZCSSL22}.
%This is a stark reminder that, as popular as ADAM is, there is no such thing as ``proof by acclamation'' if we care about correctness.
%In summary, when it comes to the actual optimization and learning tasks, all of the nice convexity structure of KB embeddings typically goes unused. %, leaving room for improvement. % in the way optimization tools are used. 
% and  optimization tools  often end up being applied 
% %is often 
% on a shaky theoretical ground. 

{When performance is good, one may be tempted to overlook these issues, but when it is not, it may be hard to know what is to blame. 
Is it the method, the parameter choice, the optimization algorithm, or a combination of those?}
{Our proposed approach \emph{does} include nondifferentiable terms as well, but, in contrast, our method of choice for solving the underlying optimization problem is suitable for handling nondifferentiability (Section~\ref{sec:discussion}). In this sense, our approach is conceptually sound from an optimization point of view.}

\textbf{Organization}. 
 Our paper is organized as follows. Section~\ref{sec:basicDef} provides basic definitions. Section~\ref{sec:boxSemantics} defines the semantics of our \modelName approach. Section~\ref{sec:faithful} studies \modelName's faithfulness properties. Section~\ref{sec:opt} formulates \modelName's convex optimization problem and shows faithfulness properties that are 
%changed guarantee to ensure : we can discuss more if you prefer the word guarantee
ensured by the problem formulation. Section~\ref{sec:exp} discusses \modelName's proof of concept implementation and experiments. Section~\ref{sec:conclusion} provides a discussion on optimization in KB embeddings and concludes our paper.
Omitted proofs and additional details   are given in the appendix.

\section{Basic Definitions: DL-Lite  Ontologies}\label{sec:basicDef}
% \subsection{Lattice}

% \begin{definition}
% A triple $\langle L, \land, \lor \rangle$ is called a lattice if the following conditions are satisfied:

% \begin{enumerate}
%     \item $L$ is a non-empty set.
%     \label{lat:nonEmpty}
%     \item $\land$ (meet) and $\lor$ (join) need to:
%     \begin{enumerate}
%     \label{lat:operations}
%         \item be binary operations $L \times L \xrightarrow{} L$;
%         \label{lat:binary}
%         \item be idempotent: $a \circ a = a$, for $\circ \in \{\land, \lor\}$ and $a\in L$;
%         \label{lat:idem}
%         \item be commutative: $a \circ b = b \circ a$, for $\circ \in \{\land, \lor\}$  and $a,b\in L$;
%         \label{lat:commu}
%         \item be associative:  $(a \circ b) \circ c = a \circ (b \circ c)$, for $\circ \in \{\land, \lor\}$ and $a,b,c\in L$;
%         \label{lat:assoc}
%         \item satisfy the absorption law $a \land (a \lor b) = a \lor (a \land b) =a $, for $a,b\in L$.
%         \label{lat:absorp}
%     \end{enumerate}
% \end{enumerate}
% \label{def:lattice}
% \end{definition}

%\subsection{DL-Lite  Ontologies} %and Canonical Model}

%We first define the syntax and semantics of 
%DL-Lite$^\Hmc$~\citep{dllite-jair09} and the notion of canonical model.
%We then 
%define our class of convex geometric models for dealing with DL-Lite$^\Hmc$. % KBs.
%define canonical models,  used to establish strong faithfulness.

%\paragraph{Syntax of DL-Lite$^\Hmc$} 
Let $\NC$, $\NR$, and $\NI$ 
be \emph{finite}, non-empty, %countably infinite 
mutually disjoint sets of \emph{concept}, \emph{role}, and \emph{individual} 
names, respectively. We denote by $\NRm$ the set $\NR\cup\{R^-\mid R\in \NR\}$,  by $\NCE$ the set $\NC\cup \{\exists R \mid R \in \NRm\}$,   and by $\NCEneg$ the set $\NCE\cup\{\neg E\mid E\in \NCE\}$.
 DL-Lite$^\Hmc$ role and concept
inclusions are   of the form $S \sqsubseteq T$ and $B \sqsubseteq C$,
resp., where $S,T\in \NRm$ are roles\footnote{A \emph{role} is a role name or the inverse of a role name.} and  $B$, $C$ are
concepts built as follows: 
%through the   rules
\(S~::=~R~\mid~R^-,\;   \ B~::=~A~\mid~\exists S,\;\ C~::=~B~\mid~\neg B,\)
with $R \in \NR$ and $A \in \NC$. 
%A DL-Lite$^\Hmc$ axiom is a DL-Lite$^\Hmc$ role or concept inclusion. 
A DL-Lite$^\Hmc$  \emph{TBox} (we may also use the less technical term \emph{ontology}) is a (finite) set 
of DL-Lite$^\Hmc$ concept and role inclusions. 
%%%Commented def of DL-Lite$^\Hmc_{bool}$
%A DL-Lite$^\Hmc_{bool}$ ontology admits concept inclusions
%$C\sqsubseteq D$ with $C,D$ of the form
%\[C,D \coloneqq A\mid (C\sqcap D)\mid \neg{C}\mid \exists %S\mid (C\sqcup D)\]
%where $A\in\NC$ and $S$ is as above.
\emph{Assertions} are %axioms 
of the form $D(a)$ (called \emph{concept assertions}) or $R(a,b)$ (called \emph{role assertions}), 
where $D\in\NCE$, $R\in\NR$, and $a,b\in\NI$.
A DL-Lite$^\Hmc$ \emph{knowledge base} (KB)
 is a pair $(\Tmc, \Amc)$
 where \Tmc is a DL-Lite$^\Hmc$ TBox  and 
 \Amc is a (finite) set of assertions, called 
 \emph{ABox}.  
 %{DL-Lite$_{core}$ is the fragment of DL-Lite$^\Hmc$ that has only concept inclusions (no role inclusions). 
 %}
% We denote by ${\sf Ind}(\Amc)$ the set of all individual names occurring in an ABox \Amc.
Following the KBE literature (e.g. \citep{BoxEL}) and to simplify our presentation,
we assume %without further notice 
that  DL-Lite$^{\Hmc}$ TBoxes  are  in \emph{named form}, meaning that they contain only concept inclusions %moved def of abbrev above
where one of the concepts is a concept name.
% \todo{Ana: I left just \NC, \NR,\NI. we have little time and now it seems a bit too late to change without possibly leaving many typos}
 %and we denote by ${\sf sig}(\Kmc)$ the set of concept and role names occurring in \Kmc. %\todo[inline]{B: Is ${\sf sig}(\Kmc) = \NC \cup \NR$ or ${\sf sig}(\Kmc) = \NCE \cup \NRm$? A: ${\sf sig}(\Kmc)$ is a subset of concept names \NC and role names \NR}
The semantics is given as usual by interpretations $\Imc=(\Delta^\Imc,\cdot^\Imc)$ (see appendix). 
We call a DL-Lite$^\Hmc$ \emph{axiom} an expression that is
a role inclusion (RI), a concept inclusion (CI), or an assertion.
We write $\Imc\models \alpha$ if \Imc satisfies
an axiom $\alpha$. An interpretation \Imc satisfies a 
KB $\Kmc=(\Tmc, \Amc)$,   written $\Imc\models \Kmc$,  if it satisfies the axioms in \Tmc and \Amc. We say that \Kmc is \emph{satisfiable} if such an interpretation \Imc exists.
Also,  \Kmc \emph{entails} an axiom $\alpha$, written $\Kmc\models\alpha$, iff, for all interpretations \Imc, if %we have that 
$\Imc\models\Kmc$ then $\Imc\models\alpha$. 
We say that an axiom $\alpha$  is \emph{consistent} with a KB \Kmc if there is an interpretation \Imc such that $\Imc\models \Kmc\cup\{\alpha\}$.
%\todo{Complete this sentence?} thank you
%\todo{add entailment d5ef}
%\todo{B:rigorously speaking, I think ``axioms of \Amc'' were not defined. Are they just the assertions in \Amc?} They were defined but for tbox we did it twice so it was confusing, now I removed the redundancy
%Concept inclusions of the form $B\sqsubseteq\neg C$ can be equivalently written as
%$B\sqcap C\sqsubseteq\bot$ with
%$(B\sqcap C)^\Imc$ interpreted as %$B^\Imc\cap C^\Imc$ and %$\bot^\Imc:=\emptyset$.

%Also, we assume w.l.o.g. that inverse roles in  RIs  only appear on the right side (since $R^-\sqsubseteq S$ is equivalent to $R\sqsubseteq S^-$). 
%\todo[inline]{B:Although it may be clear, I suggest explicitly mentioning the CI and RI stand for concept inclusion and relation inclusion when these abbreviations are first used}

\textbf{Canonical Model.}  
%\todo{add ref to can model paper}
%We provide a polynomial size canonical model definition  for DL-Lite$^\Hmc$.
%This means that,
%if a KB $\Kmc$ in such language is satisfiable, then
%there is a polynomial size interpretation that is a model
%of $\Kmc$ and satisfies exactly those inclusions and assertions that are entailed by $\Kmc$, where the polynomial is based on  the size of
%the finite sets $\NC$, $\NR$, $\NI$ of
%symbols used to construct $\Kmc$.  
Our construction is based on previous definitions of canonical models for DL-Lite$^\Hmc$ (e.g.,~\citep{DBLP:conf/kr/KontchakovLTWZ10}). Though here, we ensure that \emph{inclusions} hold in the model \emph{only if} they are entailed by the ontology.
%canonical model for DL-Lite 
%\todo[inline]{B: I am confused about the phrasing of the definition of concept satisfiability. Should it be as follows? We say that a concept $C$ is satisfiable w.r.t \Kmc 
%if there exists $a_f\in \NI\setminus {\sf Ind}(\Amc)$ such that $(\Tmc,\Amc\cup\{C(a_f)\})$ is satisfiable}
Before we provide the definition of the canonical model, 
we introduce the following notions.
Assume $\Kmc=(\Tmc,\Amc)$ is a satisfiable DL-Lite$^\Hmc$ KB. %\todo{B: should we add what it means for $\Kmc$ to be satisfiable? } Done above
We say that a concept $C$ is \emph{satisfiable w.r.t.  \Kmc} 
if there is an interpretation \Imc such that $\Imc\models\Kmc$ and
$C^\Imc\neq\emptyset$. 
%Let ${\sf Ind}(\Amc)$ denote the 
%set of individual names occurring in 
%an ABox \Amc and 
We write concepts of the form $B\sqcap C$ (with the \emph{conjunction} operator)   only to falsify
 concept inclusions of the form $B\sqsubseteq\neg C$. In an interpretation \Imc, the meaning of $(B\sqcap C)^\Imc$  is  $B^\Imc\cap C^\Imc$.
Let  $\NCEn$ be the set $\NCE\cup\{D\sqcap E\mid D, E\in\NCE\}$.
Assume   
$\Delta_\Kmc := \{c_D\mid D\in\NCEn,\ D \text{ is satisfiable w.r.t. }\Kmc  \}$
 is   disjoint from $\NI$.
Given a role $S$, we write $\overline{S}$
as the result of switching between a role name and its inverse:
 $\overline{S}=S^-$ if $S\in\NR$ and
$\overline{S}=R$ if $S=R^-$ with $R\in\NR$. 
%We are   ready to define our notion of   canonical model for DL-Lite$^\Hmc$.
%Ana: I am commenting this now based on Bruno's response via skype
%\todo[inline]{B: I don't understand what is happening in the equation above. Is it a definition? Ana:Yes. What is {\sf sig}(\Kmc)? Ana: added
%What the index in $c_A$ indicates?Ana:   it was  just to say it is one per concept. 
%Is it some $c$ such that $(\Tmc,\Amc\cup\{A(c)\})$ is satisfiable? Ana: I changed the def of concept sat wrt to KB. the reason I changed is because when NI is infinite these two conditions are equivalent. but this is not clear in our case when NI is finite and it could be that all elements of NI occur in K. It is subtle
%}
%A role $R$ is  \emph{generating} in \Kmc if there exist $a \in {\sf Ind}(\Amc)$ and $R_0, \ldots, R_n = R$  such that the following conditions hold:
%\begin{itemize}
 %   \item[\textbf{(ag)}] $\Kmc \models \exists R_0(a)$ but $R_0(a, b) \not\in \Amc$ for all $b \in {\sf Ind}(\Amc)$ 
 %   (written $a \rightsquigarrow c_{R_0}$),
%\item[\textbf{(rg)}] for $i< n$, $\Tmc \models \exists R^-_i \sqsubseteq \exists R_{i+1}$ and $R^-_i\neq R_{i+1}$ (written $c_{R_i} \rightsquigarrow c_{R_{i+1}}$).
%\end{itemize}

 \begin{definition}[Canonical Model]
 \label{def:canonicalModel} 
    The canonical model $\Imc_{\Kmc}$ for a satisfiable DL-Lite$^\Hmc$ KB $\Kmc =(\Tmc,\Amc)$ is:
   \begin{itemize}
       \item $\Delta^{\Imc_\Kmc} := 
        \NI \cup \Delta_\Kmc$,  $ \qquad$ $a^{\Imc_\Kmc} := a$, for all $a \in \NI$,
       % \item 
\item $A^{\Imc_{\Kmc}} := \{a \in \NI \mid \Kmc \models A(a)\}\cup \\
\{c_D \in \Delta_\Kmc \mid \Kmc\models  D \sqsubseteq A\}$ for all $A\in\NC$,
\item $R^{\Imc_{\Kmc}}:= \{(a,b)\in \NI\times\NI\mid  \Kmc\models R(a,b)\}\cup  \\ 
\{(a,c_{\exists S})\in \NI\times \Delta_\Kmc \mid  \Kmc\models \exists \overline{S}(a),\ \Kmc\models \overline{S}\sqsubseteq R\}\cup  \\
\{(c_{\exists S},a)\in \Delta_\Kmc\times \NI \mid  \Kmc\models \exists \overline{S}(a),\ \Kmc\models S\sqsubseteq R\}\cup  \\
\{(c_{\exists S},c_{\exists \overline{S}})\in \Delta_\Kmc \times \Delta_\Kmc \mid \Kmc\models S\sqsubseteq R\} \cup \\  
%\{(c_{\exists S^-},c_{\exists S})\in \Delta_\Kmc \times \Delta_\Kmc \mid \Kmc\models S^-\sqsubseteq R\} \cup \\
\{(c_D,c_{\exists S})\in \Delta_\Kmc \times \Delta_\Kmc {\mid} \Kmc\models D\sqsubseteq \exists \overline{S},\ \Kmc\models \overline{S}\sqsubseteq R\}\cup \\ \{(c_{\exists S},c_{D})\in \Delta_\Kmc \times \Delta_\Kmc \mid \Kmc\models D\sqsubseteq \exists \overline{S},\ \Kmc\models S\sqsubseteq R\}$, for all $R\in\NR$.
    \end{itemize}

\end{definition}
%\todo{change below if signature is a subset of all symbols}
\begin{restatable}{theorem}{theoremcanonical}
\label{thm:canonicalModel}
            Let \Kmc be a satisfiable DL-Lite$^\Hmc$ KB  and let $\Imc_{\Kmc}$ be the canonical model of  \Kmc.
    Then,  for all DL-Lite$^\Hmc$ axioms $\alpha$, % over   ${\sf sig}(\Kmc)$, 
    we have that $\Imc_{\Kmc} \models \alpha$ iff $\Kmc \models \alpha$.
    %\begin{itemize}
    %\item $\Imc_{\Kmc}$ satisfies   $\Kmc$, 
    %\item if  $c_{D}\in C^{\Imc_{\Kmc}}$ then   $\Kmc\models D\sqsubseteq C$,
    %and 
    % \item   for all DL-Lite$^\Hmc$ axioms $\alpha$ over   ${\sf sig}(\Kmc)$,   $\Imc_{\Kmc} \models \alpha$ iff $\Kmc \models \alpha$.
    %\end{itemize}
%
\end{restatable}

%\section{Our Convex Geometric Models}
\section{{\modelName} Semantics} %to improve later
\label{sec:boxSemantics}
Here we introduce a semantics for DL-Lite$^\Hmc$  inspired by box-based geometric models~\citep{BoxE} and suitable for defining a KB embedding model that can be optimized using convex optimization.
 %,
 %, %DL-Lite$^\Hmc$ KBs in vector spaces, 
 %while keeping the optimization problem convex.}
% Our semantics is .
 %,  differing in some aspects that we discuss in the following. We %will %better to avoid future, just write in present
%In particular, 
%We define DL-Lite$^\Hmc$ concept and role embeddings 
We define a geometric model
that uses   axis-aligned hyper-rectangles, called~\emph{boxes}. %\emph{boxes}. 

%\todo{Add a footnote about needing epsilon to be fixed with $0<\epsilon<\epsilonMax$, where   $\epsilonMax = 0.5$ and $\epsilonMax \leq \worldSizeScalar/8$ and that it is necessary for representing the strictly less operator in CVXPY.} this is said later
Let $\epsilon$ and $\worldSizeScalar$ be fixed but arbitrary \footnote{In our work and, in particular, in the implementation, we take $0<\epsilon<\epsilonMax$, where   $\epsilonMax = 0.5$ and $\epsilonMax \leq \worldSizeScalar/8$. 
	Making $\epsilon$ and  $ \worldSizeScalar$ learnable parameters would not enhance the representation capabilities of our model but  allow for infinitely many equivalent solutions of our learned embeddings which only differ in scale. 
	% and that it is necessary for representing the strictly less operator in CVXPY
	%Thus, we choose $\worldSize$ to be a constant instead of a learnable parameter.
} positive constants with $\epsilon \leq \worldSizeScalar$. Let $\epsilonVec$ and $\worldSize$ denote the $d$-dimensional vectors whose value in each dimension is equal to $\epsilon$ and $\worldSizeScalar$ respectively.
%DONE // was: \todo{Make any vector bold and replace $\vec{0}$ by $\mathbf{0}$}
Henceforth, we denote elementwise comparison operators with $\leqd$ and $\geqd$. %, $\led$, and $\ged$. %\todo{Ana: do we ever use $\led$ and $\ged$?}
Using $\epsilonVec$ and $\worldSize$, we define 
%the semantics of DL-Lite$^\Hmc$-concepts on the area of 
an axis-aligned hyper-rectangle, called the \emph{universe box}:
%\todo{Remove epsilon, use strict inequality. Add footnote about CVXPY not being able to handle strong inequalities and this is why we can add an arbitrary small positive epsilon to emulate strict inequalities}
%we decided to keep as is so that it is closer to the implementation
% I (Ana) removed the epsilon from universe box because it cannot be empty. ind need to be in there
%\[\World = \{\mathbf{x} \in \Rd \mid -\worldSize + \epsilonVec \leqd \mathbf{x} \leqd \worldSize - \epsilonVec\}\]
\[\World = \{\mathbf{x} \in \Rd \mid -\worldSize   \leqd \mathbf{x} \leqd \worldSize  \}\]
%\todo{B:The only bounded vector subspace in a Banach space is the trivial one containing only the origin. Maybe bounded is not the right word? Also, $\Omega$ is not a vector space.}
in our $d$-dimensional Euclidean space,  where $d>0$. 
The universe box 
is useful to establish basic properties of boxes (Theorem~\ref{thm:properties}) that are needed in our faithfulness proofs. 
% \todo{check if this theorem is indeed used} used in theo 3
%allows 
%us to define complement
%boxes with favorable properties, which are %described in Theorem~\ref{thm:properties}.
\begin{definition}\label{def:box}
A \emph{box} is an element of the set %of 
%\emph{center boxes} 
$\B$, defined as:
%
%\todo[inline]{discuss: it may be better to leave epsilon and explain}
\begin{equation*}
\begin{aligned}
  \B\coloneqq \{&\{\mathbf{x} \in \Rd \mid \mathbf{L} + \epsilonVec \leqd \mathbf{x} \leqd \mathbf{U} - \epsilonVec\}\mid \\
  &\,\,\mathbf{0} \leqd (\mathbf{U} - \mathbf{L}) \leqd 2\worldSize,  \ \mathbf{L},\mathbf{U}\in \Rd\}.  
\end{aligned}
\end{equation*}
For a non-empty box, 
%We call $\mathbf{L}$ and $\mathbf{U}$ in the definition of a non-empty box $\mathbf{X}$  in $\B$ 
the \emph{lower} and \emph{upper} bounds are denoted $\mathbf{L}$ and $\mathbf{U}$ (resp.) of $\mathbf{X}$. 
If $\mathbf{X}$ is the empty box then we set $\mathbf{L}:=\mathbf{U}:=\mathbf{0}$.
We  denote with $\mathbf{x}[i]$ the $i$-th dimension of a vector $\mathbf{x}$. If $\mathbf{X}$ is a box with lower and upper bounds $\mathbf{L}$ and $\mathbf{U}$, resp., then    $\mathbf{X}[i]$ is the pair $(\mathbf{L}[i],\mathbf{U}[i])$. 
\end{definition}

%\color{red}
\textbf{Box Definition Intuition.} \color{black}
The definition of $\B$ allows boxes to have a width of   $\leq 2 \worldSize$ \color{black} and to span outside of the universe $\World$, while we require vectors associated with individual names to be within $\World$ (as we will see in Definition~\ref{def:boxinterpretation}). This design is motivated (in the spirit of \citet{BoxE}) by associating any individual name with a position and a bump vector, where the embeddings of a pair of individual names $(a, b)$ is retrieved by translating  (``\emph{bumping}'') the position of $a$ with the bump of $b$ and vice versa. Thus, if both the position and bump vectors are within the universe box, i.e., bounded by $\worldSize$, the embeddings of any individual pair is bounded by $2 \worldSize$. As we associate any concept and role name with a set of boxes that shall contain the embeddings of individual pairs, it is sufficient that box widths are bounded by $2 \worldSize$.
As for \emph{$\epsilonVec$}'s intuition, most convex optimization tools do not allow strict inequalities or handle them less efficiently. Yet, we need them to define empty boxes for concepts/roles that are unsatisfiable.
So we employ in \cref{def:box} \emph{nonstrict inequalities} and add a small positive \emph{$\epsilonVec$} to emulate strict ones.
%\begin{remark}\upshape

%We differ from~\citeauthor{BoxE}~\shortcite{BoxE}

%\textcolor{red}{while we require vectors associated with individual names to be in the universe (Definition~\ref{def:boxinterpretation}).}
% \todo{say why we require that?}
%\end{remark}
%\todo{Aleks: Note this definition does not allow for ``geometrically'' empty concepts/roles, which make sense from a practical perspective.}
%$$\B \coloneqq \{\{\mathbf{x} \mid \mathbf{L} \leqd \mathbf{x} \leqd \mathbf{U}, \mathbf{x} \in \Rd \}\mid  \mathbf{0} \leqd \mathbf{U} - \mathbf{L} \leqd 2 \cdot\worldSize \}.$$ 
%$$\B \coloneqq \{{\sf box} \mid {\sf box} = \{\mathbf{x} \mid \mathbf{L} \led \mathbf{x} \led \mathbf{U}, \mathbf{x} \in \Rd \}, \mathbf{0} \leqd \mathbf{U} - \mathbf{L} \leqd 2 \worldSize \}.$$ 
%with $\worldSize \in \Rd$ and $\cpoint \in \vec{0}$.

%\textbf{Discussion.} 

%, which allows for an elegant definition of the complement of a box, as we will see later in this section. %not clear what elegant means here

% \todo{Boxes can span out of the universe on purpose - add this here + validation for why we use a bounded universe (complement!) - without loss of generality we can define the universe to be finite, as the number of concepts and roles is finite! }

% \todo{Discuss motivation: We use boxes for the complement etc. because we can solve the problem via convex optimization.}

%
%Note that $\World$ is the largest element of $\B$.
\begin{definition}%[Box Interpretation]
\label{def:boxinterpretation}
A \emph{\geometric interpretation} $\eta$
is a function that maps:
\begin{itemize}
    \item 
%each concept name $A\in\NC$
%to a box $\cAssign{A}$, each role name $R\in\NR$ to three boxes $\cAssign{R} = (\headAssign{R}, \tailAssign{R}, \bumpBoxAssign{R})$ and each individual name $a\in\NI$ to two vectors $\cAssign{a} = (\posAssign{a}, \bumpAssign{a})$. We  discuss each of these cases in detail in the following.
%, defined as follows.
%A \emph{center box} $\cAssign{A_i}$ is defined as: 
%
%Let $L,U$ be 
%\textbf{Individuals.} %Let $\maxBump$ be a fixed but arbitrary element of $\Rd$ such that $\maxBump \in \mathbb{R}^d_+$. 
%A \emph{\geometric interpretation} $\eta$ %maps 
each individual name $a\in\NI$ to two vectors $\eta(a) = (\posAssign{a}, \bumpAssign{a})$, namely, a position $\posAssign{e} \in \World$ and a bump $\bumpAssign{e} \in \World$;
%\citep{BoxE,Box2EL}; commented ref since we already mention at beg of 2.2
%\todo{say for all inv pos, bump are in the universe box}
% with  The size of the bumps is restricted to be smaller than a maximal bump $\maxBump$, i.e., $|b(e)|_d \leqd \maxBump$ \todo{Reference Box2EL and BoxE}. % for any $e\in\NI$.
%
%\textbf{Concepts.}
%A \emph{\geometric interpretation} $\eta$ %maps 
\item each concept name $A\in\NC$ to a box $\cAssign{A} \in \B$; % \todo{check if the box of A needs to be in somega}
% \todo{check bump notation}
%\textbf{Roles.}
%A \emph{\geometric interpretation} $\eta$ %maps 
\item each role name  $R\in\NR$ to three boxes  $\cAssign{R} = (\headAssign{R}, \tailAssign{R}, \bumpBoxAssign{R}$), which we call $R$'s {head} $\headAssign{R} $, {tail} $\tailAssign{R} $ and {bump box} $\bumpBoxAssign{R} $. %\todo{Question to Ana: Is there an equivalent to top and bottom for roles?}
\end{itemize}
%We extend the mapping function $\eta$ to arbitrary DL-Lite$^\Hmc$  role expressions  
%Now, let us define the interpretation of $\sInter$, $\sComp{\cdot}$, and $\sUnion$ 
%as follows:
%\begin{align*}
%    \cAssign{\neg{R}} &\coloneqq \{\sComp{Head(R)}, \sComp{Tail(R)}, \roleWorldSize - \bumpBoxAssign{R}\} \\
 %   \cAssign{\inv{R}} &\coloneqq (\tailAssign{R}, \headAssign{R}, \bumpBoxAssign{R}), \\
  %  \cAssign{\exists R} &\coloneqq  \{\mathbf{L^H_R} - \mathbf{U^B_R} + \epsilonVec \leqd x \leqd \mathbf{U^H_R}  - \mathbf{L^B_R} - \epsilonVec \}. 
%\end{align*}
%\todo{Is it possible to define the negation of a box in a way that ensures that RQ1 is satisfied?}
% The concept $\top$ is mapped to $\World$ %,\setminus\{\cpoint\}$
% in symbols, $\eta(\top)=\World$, and $\eta(\bot)=\emptyset$. 
%The concept $\bot$ is mapped to the empty set, in symbols, $\eta(\bot)=\emptyset$. 
We extend the mapping function $\eta$ to arbitrary DL-Lite$^\Hmc$ concept and role expressions  
%Now, let us define the interpretation of $\sInter$, $\sComp{\cdot}$, and $\sUnion$ 
as follows:
\begin{align*}
    \cAssign{\neg{C}} &\coloneqq \sComp{\cAssign{C}}, \quad \cAssign{\inv{R}} \coloneqq (\tailAssign{R}, \headAssign{R}, \bumpBoxAssign{R}), \\
     %\\
    \cAssign{\exists R} &\coloneqq  \{\mathbf{x} \in \Rd \mid \mathbf{L^H_R} - \mathbf{U^B_R} + \epsilonVec {\leqd} \mathbf{x} {\leqd} \mathbf{U^H_R}  - \mathbf{L^B_R} - \epsilonVec \} 
\end{align*}
where $\mathbf{L}^X_\mathbf{R}$, $\mathbf{U}^X_\mathbf{R}$ with $X \in \{\mathbf{H}, \mathbf{T}, \mathbf{B}\}$ are the lower and upper bounds of $\headAssign{R}$, $\tailAssign{R}$, and $\bumpBoxAssign{R}$; and where $\sComp{\cAssign{C}}$ represents $\cAssign{C}$'s complement box. Furthermore, for any $C\in\NCE$, the complement box $\sComp{\cAssign{C}}$ is defined as follows: 
\begin{equation*}
     \sComp{\cAssign{C}} \coloneqq 
     \{ \mathbf{x} \in \Rd \mid (-\worldSize - \mathbf{L_C} + \epsilonVec) \leqd \mathbf{x} \leqd (\worldSize - \mathbf{U_C} - \epsilonVec)\}
\end{equation*}%
with $\mathbf{L_{C}}$ and $\mathbf{U_{C}}$ being $\cAssign{C}$'s lower and upper bounds.
%{(note that $\mathbf{L_{\exists R}}:=\mathbf{L^H_R} - \mathbf{U^B_R}$ and $\mathbf{U_{\exists R}}:=\mathbf{U^H_R}  - \mathbf{L^B_R}$)}.
\end{definition}
%\todo{Exists R is still wrongly defined: Head Bump with max bump.}

% Finally, we require that for all $\cAssign{C} \in \B$ we have that $\cAssign{C} \cap \sComp{\cAssign{C}} = \emptyset$ (\textbf{RQ1}). 

%\todo{Ana: it seems strange to have RQ1 without proving that a model that satisfies it exists and if it exists then it is strange to not already define the model in a way that ensures this is satisfied (?)}
% \todo{Aleks: Yes you are completely right, calling this a requirement is weird. I believe RQ1 is related to the following observation: There are first-order formulae that are inconsistent. Similarly, there are \geometric{} interpretations of our ``geometric'' language that are inconsistent. I tried below to formulate RQ1 not as a requirement, but as a consistency criteria below, which is hopefully more natural. What do you think? :)}
% \todo{Ana: agree, it is   better to express this as a condition for consistency. thanks  }
% \todo{Aleks: Add whether we can define a complement that is always consistent: This might be possible, but so far we did not manage to do that in a manner that allows for convex optimization. Recall that while cone semantics has only consistent interpretations, it cannot be optimized with any form of continuous optimization (be it gradient descent or convex optimization), but requires some form of nonconvex discrete optimization, which is much harder to optimize in practice (to my understanding of Bruno's explanations so far :) ).}
%ok then lets leave as is thanks

%\color{red}

\textbf{Box Interpretation Intuition.}
At first sight, an intuitive definition for the complement of a box would be the set complement. However, we cannot use this notion as it leads to non-convexity. Thus, a different convex-preserving definition of box complements is required that satisfies important properties of the usual complement, as shown in Theorem~\ref{thm:properties}. %A similar issue is discussed by \citet{OLW20}, where the authors also invented an alternative notion of complement. Our definition of a complement box has similar properties,  %satisfying the following key properties 
%proven in Theorem~\ref{thm:properties}. 
Next, the box interpretation of inverse roles $\cAssign{\inv{R}}$ swaps the head and tail boxes of $\cAssign{R}$. As for the existential boxes, 
we define $\cAssign{\exists R}$ as $\headAssign{R}$ enlarged by the boundaries of $\bumpBoxAssign{R}$, to ensure that it contains the embeddings of all subjects of $R$ assertions, following $\exists R$’s semantics.

\begin{restatable}{theorem}{thmprop}%[Complement]
\label{thm:properties} For any \geometric{} interpretation $\cAssignSymb$:
\begin{enumerate}[align=left,label=\roman*),leftmargin=*] %[$(i)$]
    \item for all $C \in \NCE$, 
    %there is a complement %Ana: commented since this is not just existential but something defined
    $\sComp{\cAssign{C}} \in \B$;
    \label{lat:comp1}
    \item for all $C \in \NCE$,   $\sComp{\sComp{\cAssign{C}}} = \cAssign{C}$;
    \label{lat:comp2}
    \item for all $C, D \in \NCE$, if $\cAssign{C} \subseteq \cAssign{D}$ then $\sComp{\cAssign{D}} \subseteq \sComp{\cAssign{C}}$.
    \label{lat:comp3}
%    \item For all $A, B \in \NC$: if $\cAssign{A} \subseteq \cAssign{B}$ and $\cAssign{A} \subseteq \cAssign{\neg B}$ then $\cAssign{A} = \emptyset$. 
%    \label{lat:comp4}
\end{enumerate}
\end{restatable}

\textbf{Box Consistency.} A \geometric{} interpretation is \emph{box consistent} if for all %$\cAssign{C} \in \B$ 
{$C\in\NCE$}
we have that $\cAssign{C} \cap \sComp{\cAssign{C}} = \emptyset$.  
%\color{red}
%In Section~\ref{sec:ProblemFormulation}, we define convex constraints that ensure this property. 
\color{black}

%\textbf{Truth of Assertions.} Definition \ref{def:BoxSatisfaction} defines the truth value of an assertion $\phi$ in a \geometric interpretation $\eta$ --- denoted as $\eta \models \phi$. 

% \textbf{Truth of Concept Assertions.} A concept assertion $A(a)$ is true in a \geometric interpretation $\eta$ --- denoted as $\eta \models A(a)$ --- if:

% \begin{align}
% \mathbf{L_A} \led \posAssign{a} \led \mathbf{U_A}
% \label{eq:ConceptTruth}
% \end{align}

% At an intuitive level, this means that $\eta \models A(a)$ if $\posAssign{a}$ is within the relative interior of $\eta(A)$, which we denote henceforth as $\posAssign{a} \in \eta(A)$. Note that $A(a)$ is false in $\eta$ --- denoted as $\eta \not\models A(a)$  --- if $\eta(a) \in \eta(\neg{A})$. Otherwise the truth state of $A(a)$ in $\eta$ is unknown. 

% \textbf{Truth of Role Assertions.} A role assertion $R(a,b)$ is true in $\eta$ --- denoted as $\cAssignSymb \models R(a,b)$ if the following constraints hold: 

% \begin{align}
% \posAssign{a} + \bumpAssign{b} &\in \headAssign{R} 
% \label{eq:roleTruthHead} \\
% \posAssign{b} + \bumpAssign{a} &\in \tailAssign{R} 
% \label{eq:roleTruthTail} \\
% \bumpAssign{a} &\in \bumpBoxAssign{R}
% \label{eq:roleTruthHeadBump} \\
% \bumpAssign{b} &\in \bumpBoxAssign{R}
% \label{eq:roleTruthTailBump}
% \end{align}

% In contrast to concept assertions, a role assertion $R(a,b)$ is false in $\cAssignSymb$ -- denoted as $\cAssignSymb \not\models R(a,b)$ --- if at least one of the constraints \ref{eq:roleTruthHead}-\ref{eq:roleTruthTailBump} is violated. Otherwise the truth state of $R(a,b)$ is unknown.

\begin{definition}
%[Box Satisfaction]
\label{def:BoxSatisfaction}
%We say that 
A \geometric{} interpretation $\cAssignSymb$ satisfies 
\begin{itemize}
    \item  %role assertion 
    \(R(a, b)\) with $R \in \NR$ and $a,b\in\NI$
    %\cup \{\exists R, \exists R^- \mid R \in \NR\}$ \todo{Ana: the part $\{\exists R, \exists R^- \mid R \in \NR\}$ is not needed here right?}
    iff 
    \begin{align*}
        i) \; \posAssign{a} + \bumpAssign{b} &\in \headAssign{R} 
        %\label{eq:roleTruthHead} 
        \\
        ii) \; \posAssign{b} + \bumpAssign{a} &\in \tailAssign{R} 
        %\label{eq:roleTruthTail} 
        \\
        iii) \; \bumpAssign{a},\bumpAssign{b} &\in \bumpBoxAssign{R}
        %\label{eq:roleTruthHeadBump}
        %\\
        %\bumpAssign{b} &\in \bumpBoxAssign{R}
        %\label{eq:roleTruthTailBump}
    \end{align*}

    \item   \(R \sqsubseteq S\) with $R, S \in \NRm$ iff $\; \bumpBoxAssign{R} \subseteq \bumpBoxAssign{S}$  
    \begin{align*}
            \headAssign{R} %&
            \subseteq \headAssign{S}\qquad 
        %\label{eq:RoleInclusionHead} 
        %\\
            \tailAssign{R} %&
            \subseteq \tailAssign{S}
        %\label{eq:RoleInclusionTail} 
        %\\
         %   
        %\label{eq:RoleInclusionBump}
    \end{align*}
        \item   \(C \sqsubseteq D\), $C \in \NCE $ and $D \in\NCEneg$ iff $\; \cAssign{C} \subseteq \cAssign{D}$.
    %\todo{Ana: we should also include the case $D=\neg\exists R^{(-)}$}
    %\begin{align}
     %   \cAssign{C} \subseteq \cAssign{D}
     %   \label{eq:ConceptInclusion}
    %\end{align}
\end{itemize}

%\todo[inline]{Add an explanation with a figure of the head tail and bump box, where without bump box exist R included in exist S and exist $R^-$ included in exist $S^-$.}
Regarding  concept assertions %\(C(a)\) 
with $C \in \NCE$ and $a\in\NI$, 
%we have that
$\eta$ satisfies \(C(a)\)  if  
$ \posAssign{a} \in \cAssign{C} $
   % \begin{align}
    %\posAssign{a} \in \cAssign{C} 
    %\label{eq:ConceptTruthone}
    %\end{align}
    and 
    $\eta$ falsifies \(C(a)\)  if  
    $\posAssign{a} \in \sComp{\cAssign{C}}$.
    %\begin{align}
    %\posAssign{a} \in \sComp{\cAssign{C}}. 
    %\label{eq:ConceptTruthtwo}
    %\end{align}
    Otherwise the   status of \(C(a)\)
    is \emph{unknown}\footnote{Geometric models for languages that have negation normally need to allow for the `unknown' truth status. This happens in the Cone Semantics~\citep{OLW20}, tailored for $\mathcal{ALC}$, which has negation. Since DL-Lite$^\Hmc$ allows for $B\sqsubseteq \neg C$ we use complement boxes and   the `unknown' truth state.}.
%Furthermore, we say that a \geometric{} interpretation does not %satisfy
%\todo{this will need some revision}
%\todo[inline]{next label is defined twice}
%\begin{itemize}
 %   \item a concept assertion \(C(a)\) with $C \in \NC \cup \{\exists R, \exists R^- \mid R \in \NR\}$ iff 
  %  \begin{align}
   % \posAssign{a} \in \cAssign{\neg C} 
    %\label{eq:NotConceptTruth}
    %\end{align}
    %\item a role assertion \(R(a, b)\) with $R \in \NR \cup \{\exists R, \exists R^- \mid R \in \NR\}$ iff at least one of Equations \ref{eq:roleTruthHead}-\ref{eq:roleTruthTailBump} is not satisifed.
    %\item a concept inclusion \(C \sqsubseteq D\) with $C \in \NC \cup \{\exists R, \exists R^- \mid R \in \NR\}$ and $D \in \{E, \neg E \mid E \in \NC\} \cup \{\exists R, \exists R^- \mid R \in \NR\}$ iff Equation \ref{eq:ConceptInclusion} is not satisifed.
    %\item a role inclusion \(R \sqsubseteq S\) with $R, S \in \{R, R^- \mid R \in \NR\}$ iff at least one of Equations \ref{eq:RoleInclusionHead}-\ref{eq:RoleInclusionBump} is not satisifed.
%\end{itemize}

We write \(\cAssignSymb \models \alpha\) to indicate that \(\cAssignSymb\) satisfies an axiom \(\alpha\), and write \(\cAssignSymb \not\models \alpha\) otherwise. 
Also, if \(\cAssignSymb \models \alpha\) for every axiom \(\alpha\) in a KB \(\ontoo\) then we say that \(\cAssignSymb\) satisfies \(\ontoo\) or, equivalently, that \(\cAssignSymb\) is a model of \(\ontoo\), in symbols, \(\cAssignSymb \models \ontoo\).  
\end{definition} 

\begin{figure}[]
    % \vskip 0.2in
    \begin{center}
    \centerline{\includegraphics[width=250pt]{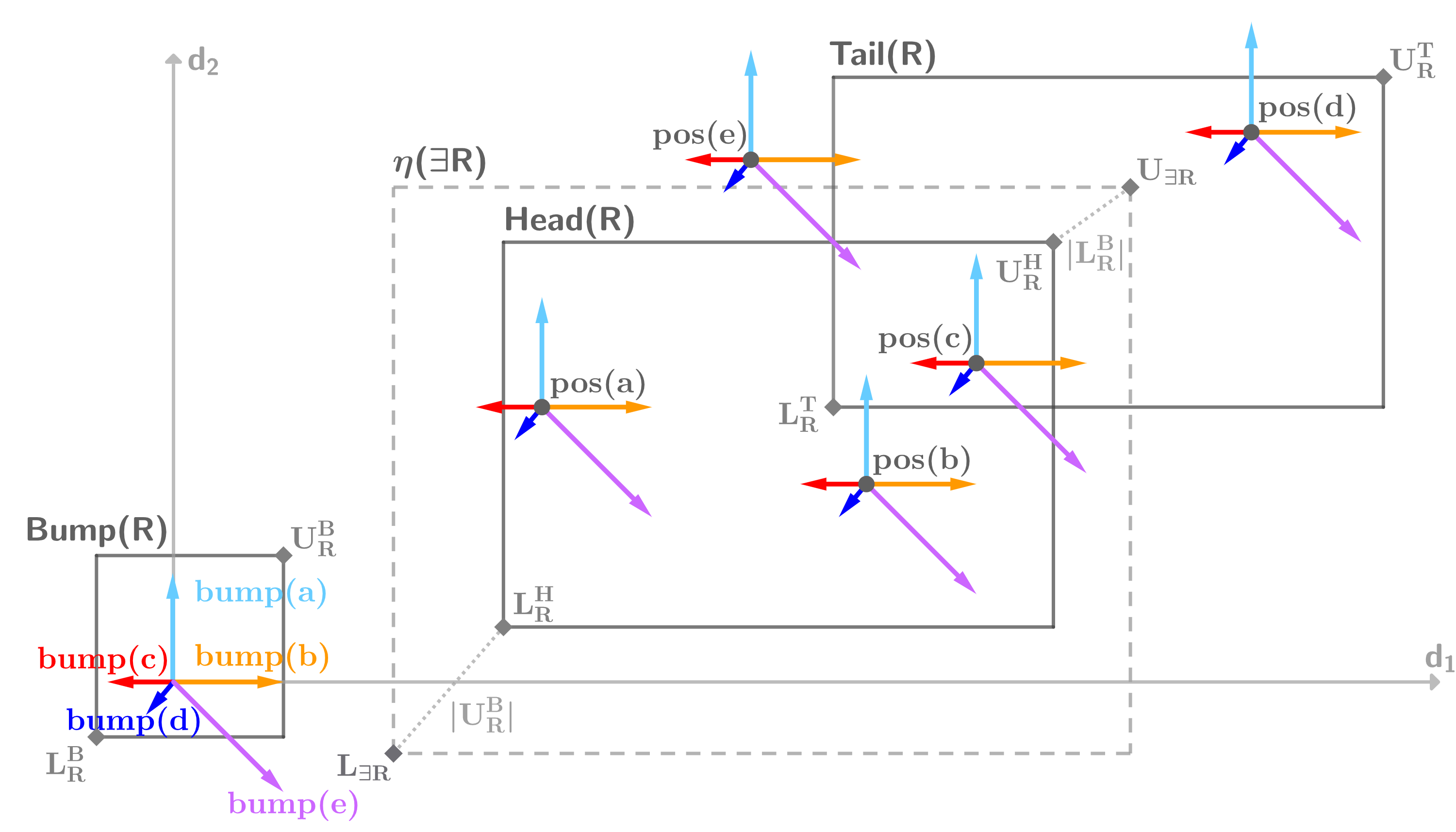}}
    \caption{A two-dimensional box interpretation $\eta$ is shown. It maps a role $R \in \NR$ to the $\headAssign{R}$,  $\tailAssign{R}$, and $\bumpBoxAssign{R}$ boxes. Additionally $\eta(\exists R)$ is visualized. The depicted box interpretation satisfies the assertions \{ \(R(a, b)\), \(R(b, d)\), \(R(b, c)\), \(R(c, d)\), \(R(c, c)\), \(\exists R(a)\), \(\exists R(b)\), \(\exists R(c)\) \} over $\{a,b,c,d,e\} \in \NI$.} 
    \label{fig:BoxLitE_Components}
    \end{center}
    \vskip -0.2in
\end{figure}

%\color{red}
%\textbf{Existential Box Intuition.}
The most intricate part of Definition~\ref{def:BoxSatisfaction} corresponds to role assertions, so we explain this in more details.
Items $i)$ and $ii)$ are as in~\cite{BoxE}. Item $iii)$ is helpful to embed existential concepts. % we add the membership of the bumps inside the bump boxes .
Recall that $\cAssign{\exists R}$ is %defined as the 
$\headAssign{R}$ enlarged by the boundaries of $\bumpBoxAssign{R}$. %This is the case for the following reasons. The satisfaction of $R(a,b)$ in \geometric interpretations is tied to the satisfaction of the Conditions~$i)$ and~$iii)$ of Definition~\ref{def:BoxSatisfaction}. Condition $i)$ states that $\posAssign{a}$ translated by $\bumpAssign{b}$ must lie within $\headAssign{R}$. By Condition $iii)$, the translation $\bumpAssign{b}$ is bounded by $\bumpBoxAssign{R}$. 
Thus, if $R(a,b)$ is satisfied in a \geometric interpretation, $\posAssign{a}$ may lie at most the size of $\bumpBoxAssign{R}$ away from $\headAssign{R}$ (as visualized in Figure~\ref{fig:BoxLitE_Components}).  %, motivating the definition of $\cAssign{\exists R}$.
\color{black}

With this, we have defined \modelName's semantics in terms of box interpretations. 
%\textbf{BoxLitE vs. BoxE.} \color{black}
In contrast to \citep{BoxE}, we $(i)$ define \emph{complement boxes}, which allow the formulation of convex constraints that guarantee the satisfaction of concept disjointness axioms in our embedding solutions and $(ii)$ associate roles with additional \emph{bump} boxes that constrain the maximal bump of an individual that is accepted by a role embedding. %  Point $(ii)$  
The latter 
is important %$(a)$
{to distinguish between axioms of the form $R\sqsubseteq S$ and $\{\exists R\sqsubseteq \exists S,\exists R^-\sqsubseteq \exists S^-\}$.} 
%\color{red} 
%Next, the intuition of these aspects is discussed in detail.
% (i) The complement box of $A$ is indeed a box. (ii) $\sComp{A}$'s complement box is $A$. % (i.e., the complement box operation is an involution). 
% (iii)  If a box $B$ is a subset of a box $A$, then the complement of box $A$ is a subset of the complement of box $B$. 
% \color{black}
In the next section, we will define faithfulness properties of geometric models. Furthermore, we will analyse which faithfulness properties for DL-Lite$^\Hmc$ can be satisfied by \modelName's box interpretations.

% \textcolor{red}{Note: If $\sComp{\cAssign{A}} \subseteq \cAssign{B}$ and $\cAssign{A} \subseteq \cAssign{B}$, then $\cAssign{A}$ is not necessarily $\World$! But this is not a problem, as such constraints are not in our language (negation can always be on the lefthand side!) - Also given ``unknown'' areas this sounds rather intuitive - we only know a subset of all negative facts for A and all positive ones for A, thus including both of them does not necessarily give the universe (the union operation however would do that!)}

\section{Model Faithfulness}\label{sec:faithful}
In this section, we study model faithfulness~\citep{OLW20}, which is a property that is useful to show that geometric models correctly represent the knowledge present in KBs. For KB completion, one is particularly interested in preserving the conceptual knowledge in TBoxes while allowing new assertions to hold.
%in the embedding. 

\begin{definition}\label{def:faithfulness}(Adapted~\citep{OLW20,DBLP:conf/kr/Bourgaux0KLO24})
    Let \(\Lmcp\) be an ontology language and  
    %description logic  
     \(\ontoo\)   a satisfiable KB in \(\Lmcp\).   A \geometric interpretation   \(\cAssignSymb\) is 
   % Then,
    \begin{itemize}
        \item 
        \emph{weakly KB faithful} for \Lmc and \Kmc if for every KB axiom $\alpha$ in \Lmc, \ $\cAssignSymb\models \alpha$ implies that $\alpha$ is consistent with \Kmc;
   %     \(\cAssignSymb\) is  \emph{weakly KB faithful   w.r.t. \(\ontoo\in \Lmcp\)} iff  
    %    for every  KB axiom  \(\alpha\) in \Lmc, if \(\cAssignSymb \models\alpha\) %\todo{$\models?$} 
     %   then \(\ontoo \cup \{\alpha\}\) is satisfiable in the     semantics of \(\Lmcp\);
        \item %$\cAssignSymb$ is 
        \emph{KB entailed}
       % \(\cAssignSymb \models\ontoo\)
         for \Lmc and \Kmc if for every KB axiom $\alpha$ in \Lmc that is entailed by \Kmc, \ $ \cAssignSymb\models \alpha$. 
        %\todo{remove strong faith}
        %and
        %\item \(\cAssignSymb\) is a \emph{strongly KB faithful model of \(\ontoo\) w.r.t.\ \(\Lmcp\)} iff \(\cAssignSymb \models\ontoo\) and, for every KB axiom   \(\alpha\) in \(\Lmcp\), \(\cAssignSymb \models  \alpha\) iff \(\ontoo \models \alpha\).
        %in the semantics of \(\Lmcp\).
    \end{itemize}
    %From now on, 
    We may omit ``weakly'' and just say ``faithful''. {We may also omit \Lmc and/or \Kmc if they are clear from the context.} If a \geometric $\eta$ interpretation satisfies both of the  properties we  say that $\eta$ is a KB \textbf{faithful model}. 
    These notions can be adapted  for the case in which \(\Lmcp\) is a  description logic 
    with TBox and ABox axioms {and for the case where}
    we consider only TBox axioms (or only  ABox axioms)   for    TBox (or ABox) faithful models.
    % (accordingly).
\end{definition}
%\citep{DBLP:conf/kr/Bourgaux0KLO24}.
%\todo{mention that this corresponds to weak and entailed from kr24 paper. since we only consider weak we could say here that from now on whenever we talk about faith we mean weak faith}
% \todo[inline]{How to show weak/strong faithfulness?}
% \todo[inline]{Does DL-Lite form some kind of lattice structure (or other algebraic structure), such that we solely have to prove that our geometric interpretation follows this structure? (as for e.g. boolean algebras)}

\newcommand{\posAssignI}[2]{\ensuremath{\mathsf{pos}_{#1}(#2)}\xspace}
\newcommand{\bumpAssignI}[2]{\ensuremath{\mathsf{bump}_{#1}(#2)}\xspace}
\newcommand{\lowerboundI}[2]{\ensuremath{\mathbf{L}_{#1}(#2)}\xspace}
\newcommand{\upperboundI}[2]{\ensuremath{\mathbf{U}_{#1}(#2)}\xspace}

%From the KB completion 
%\todo{Should we call this ``KB completion'' instead of KG completion? Did we introduce this term? I changed, thanks, I think (hope, because of space) it may be understandable what we mean by the references of embeddings with DL ontologies} 
%point of view, strong faithfulness can be too restrictive, since we would like the model to generalize and thus to  make additional axioms hold, as long as they do not contradict the conceptual knowledge present in the KB.
%This holds in particular for ABox axioms, and that is why it also makes sense to differentiate between KB faithfulness and the more specific notions that include only TBox or only ABox axioms. 
\cref{lem:tboxfaith} and \cref{lem:kbfaith} give basic conditions for    faithfulness  
%when the ontology language is 
in DL-Lite$^\Hmc$. 
\begin{restatable}{proposition}{lemweakfaithemptyabox}\label{lem:tboxfaith}
Let \Tmc be a DL-Lite$^\Hmc$ TBox.
If $\eta\models\Tmc$
%is a TBox entailed \geometric interpretation for a DL-Lite$^\Hmc$ TBox \Tmc and $\eta$, 
then $\eta$ is  (weakly) TBox faithful.
\end{restatable}
Proposition~\ref{lem:tboxfaith} follows from the fact that, in DL-Lite$^\Hmc$, TBoxes are always consistent, that is, inconsistencies only happen when considering a TBox and an ABox.
%Then, for Proposition~\ref{lem:tboxfaith}, we only need to show that, under the assumptions of the proposition, the box interpretation is TBox entailed.} 
%
%\todo[inline]{B: In Lemma~\ref{lem:kbfaith}, to match the definition, should it be ``weakly KB faithful'' instead or it does not matter?}changed thanks
\begin{restatable}{proposition}{lemweakfaith}\label{lem:kbfaith}
%If $\eta$ is a box consistent model for a DL-Lite$^\Hmc$ KB \Kmc 
Let \Kmc be a DL-Lite$^\Hmc$ KB and
$\eta$ a box interpretation that is box consistent. 
     If $\eta\models \Kmc$ %{is KB entailed}  
%and $\eta \models \Kmc$, 
then $\eta$ is (weakly)  KB faithful. 
\end{restatable}
%
%\todo[inline]{Add proof}
%
%
We are now ready to state \cref{thm:faith_bool_alc}, which is the main result of this section. 
%Although this result is  interesting from a theoretical point of view, in practice, we  are typically only interested in weak KB faithfull models.
%or strong TBox faithfulness (but not strong ABox faithfulness).
%since we want to be able to infer new facts). 
%
%todo create mapping
%

%\todo{remove this theorem}
\begin{restatable}{theorem}{thmfaithfulness}\label{thm:faith_bool_alc}
%\todo{DL-Lite}
There exists a suitable $s_\Omega$ such that every satisfiable  %\(\mathcal{L}\)-ontology
    DL-Lite$^\Hmc$ KB \Kmc
    %\(\ontoo\) 
      has a box interpretation $\eta_{\Imc_\Kmc}$ that is %strongly 
      a  
      KB faithful model and \geometric{} consistent.
      %{that is box consistent}. 
    %\(\cAssignSymb_{\Imc_{\ontooC}}\).
\end{restatable}
\begin{proof}[Sketch]
%\todo{In fact our proof provides a stronger guarantee, add sentence before and cite cone sem neg logic paper}
The proof strategy of \cref{thm:faith_bool_alc} consists of first defining a mapping that translates
classical finite interpretations into \geometric interpretations. Then, given a DL-Lite$^\Hmc$ KB \Kmc, we construct the canonical model $\Imc_\Kmc$ in \cref{def:canonicalModel} and use this mapping to create 
a \geometric interpretation $\eta_{\Imc_\Kmc}$ (using  $s_\Omega=4$) that is a %weakly 
KB faithful model of \Kmc. 
In fact, our proof provides a stronger guarantee: the embedding   is KB entailed and strongly KB faithful~\citep{OLW20,DBLP:conf/kr/Bourgaux0KLO24}.
%, add sentence before and cite cone sem neg logic paper
%We provide a detailed proof in the supplemental material.
\end{proof}%
% \todo[inline]{This way of framing our problem is great - i.e., any ontology has a strongly faithful \geometric{} interpretation \(\cAssignSymb_{\Imc_{\ontooC}}\). Our geometric interpretation is not distributive in meet and join in general, but it might be the case that at least there is always a solution that preserves distributivity!}
%
Informally, Theorem \ref{thm:faith_bool_alc}'s proof proposes for every satisfiable DL-Lite$^\Hmc$ KB $\Kmc$ an algorithm for constructing a box interpretation (i.e., a \modelName embedding)   that is %strongly
  KB faithful and box consistent. Using this algorithm, we can compute an upper bound for the minimal number of dimensions that a box interpretation requires to satisfy certain faithfulness properties, leading to Corollaries~\ref{cor:ConsistencyDim} and \ref{corollary:FaithfulnessDim}.

\begin{restatable}{corollary}{corFaithfulSolelyABox}
Let $\ontoo$ be a DL-Lite$^\Hmc$ KB
    %\(\ontoo\) %without role disjointness 
    %and 
    % \todo{Ana: do we need satisfiable here?}
    with an empty ABox. Then for any $d \geq d_{\min}$ there is a \geometric interpretation $\cAssignSymb$ with dimensionality $d$ such that 
%\begin{enumerate}[align=left,label=\roman*),leftmargin=*] %[$(i)$]
 %   \item 
    $d_{\min} = |\NC| + 3|\NR|$  and $\cAssignSymb$ is a    {TBox} faithful model of $\ontoo$.
%    \item $d_{\min} = |\NC| + |\NR| (2 + (|\NC| + 2|\NR|)^2)$  and $\cAssignSymb$ is a strongly  {TBox} faithful model of $\ontoo$.
%\end{enumerate}
\label{cor:ConsistencyDim}
\end{restatable}

%\todo{remove this corollary}
\begin{restatable}{corollary}{corFaithfulEnitreOntoo}
\label{corollary:FaithfulnessDim}
Let $\ontoo$ be a satisfiable DL-Lite$^\Hmc$ KB. %\todo{Ana: do we need satisfiable here?}
    %\(\ontoo\). 
    Then for any $d \geq d_{\min}$ there is a \geometric interpretation $\cAssignSymb$ with dimensionality $d$ such that:
%\begin{enumerate}[align=left,label=\roman*),leftmargin=*]  %[$(i)$]
    %\item 
    $d_{\min} = |\NC| + |\NR|(2 + |\NI| + 2|\NR|)$  and $\cAssignSymb$ is a    {KB} faithful model of $\ontoo$.
  %  \item $d_{\min} = |\NC| + |\NR| (2 + |\NI| + (|\NC| + 2|\NR|)^2)$  and $\cAssignSymb$ is a strongly  {KB}
  %  faithful model of $\ontoo$.
%\end{enumerate}
% Let $\ontoo$ be a DL-Lite$^\Hmc$-ontology
%     \(\ontoo\).
% Then for any $d \geq {\color{red}|\NC| + 2|\NR| + |\NR| (|\NC| + 2|\NR| + |\NI|)}
% $ there is a \geometric interpretation $\cAssignSymb$ with dimensionality $d$ such that  $\cAssignSymb$ is a strongly faithful model of $\ontoo$.
\end{restatable}
 
% \todo[inline]{Aleks: Please note that this corollary is very similar to Theorem \ref{thm:faith_bool_alc}, except that we additionally state the dimensionality.}

With this we have finished the theoretical analysis of \modelName's faithfulness properties. What now remains to show is how our \modelName approach can be translated into a convex optimization problem that  \emph{ensures} certain faithfulness properties. We will investigate this in the next section.

% \todo{Delete below!}

% Corollary \ref{corollary:FaithfulnessDim} directly follows from the proof of Theorem \ref{thm:faith_bool_alc}. Specifically, our translation of the canonical model of $\ontoo$ to $\cAssignSymb_{\Imc_{\ontoo}}$ constructs (1) one dimension $i_C$ for any concept $C \in \NCE$, (2) one dimension $i_{R,a}$ for any pair of roles $R \in \NR$ and individuals $a \in \Delta_\ontoo \cup \NI$. Thus, the constructed embedding has the following number of dimensions: 
% \begin{align*}
% |\NCE| + |\NR| (|\Delta_\ontoo| + |\NI|)& =\\
% |\NC| + 2|\NR| + |\NR| (|\NC| + 2|\NR| + |\NI|)&
% \end{align*}

% \textcolor{green}{
% \begin{align*}
% |\NCE| + |\NR| (|\Delta_\ontoo| + |\NI|)& =\\
% |\NC| + 2|\NR| + |\NR| ((|\NC| + 2|\NR|)^2 + |\NI|) &=\\
% |\NC| + |\NR| (2 + |\NI| + (|\NC| + 2|\NR|)^2)
% \end{align*}
% } \todo{Aleks: The green lines above are my understanding of the new dimensionality.}

% Similarly to Corollary \ref{cor:ConsistencyDim},    $\cAssignSymb_{\Imc_{\ontoo}}$ remains strongly faithful in $\ontoo$ when its dimensionality is increased by copying an arbitrarily chosen dimension of $\cAssignSymb_{\Imc_{\ontoo}}$ multiple times. Thus, we have shown that for any $d \geq (|\NR| + 1) (|\NC| + 2|\NR|)$ there is a \geometric interpretation $\cAssignSymb$ with dimensionality $d$ such that $\cAssignSymb$ a strongly faithful model of $\ontoo$.

\section{\modelName's Convex Optimization Problem}
\label{sec:opt}
%\todo{Change to current new optimization formulation}
%This section presents how 
Here, we formulate the representation of KBs with \modelName as a constrained convex optimization problem \citep{BV04,Rt97} and investigate how the formulation relates to faithfulness properties. 
In Section \ref{sec:ProblemFormulation}, we describe  how TBox axioms are translated to constraints. In Section \ref{sec:Scoring}, we discuss the objective and scoring functions. 
%how we calculate %the position and bump vectors of ABox axioms is reflected in 
%the scoring function and optimized in the objective function. 
In Section \ref{sec:WeakFath}, we establish  that our problem formulation ensures the KB  faithful and TBox entailed properties for DL-Lite$^\Hmc$. %Furthermore, we show that the final convex optimization problem can be cast as a as a \emph{second-order cone program} (SOCP) \citep{LVBL98}, \cite[Lecture~3]{BtN01}. 
 In summary, the TBox corresponds to a part of the constraints of the convex optimization problem and the ABox corresponds to a part of the objective function of the problem.
In this way, a solver can handle both the ABox and TBox simultaneously by minimizing the objective function subject to the constraints.

%Inspired by the definition of the canonical model, Section \ref{sec:OP_StrongFaithful} shows how \modelName's problem formulation can be modified to ensure strong TBox faithfulness.
%moved the comment on the appendix

\subsection{From TBox Axioms to Constraints}
%associated to TBox axioms}
\label{sec:ProblemFormulation}

% Given $d, \epsilon, \worldSize$, a \geometric interpretation $\eta$ is determined by the parameters in the geometric model (lower/upper bounds of boxes, positions of individuals, etc.).
% Suppose we concatenate all the parameters of the \geometric interpretation into a single vector $z$ of dimension 
% $n := (2|\NI|+2|\NC|+6|\NR|)d$, where $|S|$ denotes the cardinality of a set $S$. 
% In this way, we can reconstruct $\eta$ from $z$ and vice-versa.
Given a \geometric interpretation $\eta$, we   concatenate all of $\eta$'s parameters into a single vector $z$ of $n := (2|\NI|+2|\NC|+6|\NR|)d$ dimensions\footnote{For each individual in $\NI$ we have two vectors, one for the position and one for the bump. For each concept in $\NC$ we also have two vectors, one for the upper corner and one for the lower corner of the box. For each role in $\NR$ we have 3 boxes: head, tail, and bump; each box needs the upper and lower corner vectors.}. We call this vector $z \in \Re^n$ an \emph{embedding solution}. Conversely, given such a $z \in \Re^n$, we can reconstruct $\eta$. However, not all $z$ will lead to an $\eta$ that satisfies Definition~\ref{def:boxinterpretation}. Therefore, we devise constraints that $z$ must satisfy, such that the corresponding $\eta$ has desirable properties. %For example,  $\eta$ must be box consistent and the positions of individuals must be within the universe box. We call these constraints the \emph{box consistency} and \emph{universe constraints}.
%and %define them formally 
%provide details in the supplemental material. 

In particular, given an arbitrary DL-Lite$^\Hmc$ TBox $\Tmc$, %without role disjointness, 
we want to ensure that any feasible solution for a \geometric{} interpretation (i) is box consistent, i.e., for all $C \in \NCE$ it holds that $\cAssign{C} \cap \cAssign{\neg C} \neq \emptyset$, (ii) ensures that any box width is positive and bounded by $2\worldSize$,  %consistent (i.e., $\cAssignSymb \not\models \bot$)
(iii) ensures that any individual embedding is within the universe box, and (iv) guarantees that any TBox axiom is satisfied in the embedding solution. In the following, we define convex constraints for Points (i)--(iv). % and ensure that, given a $\Tmc$, Points (i) and (ii) are satisfied in any feasible solution for the \geometric{} interpretation.

\textbf{Box Consistency Constraints.} Intuitively, Point (i) ensures that the set of feasible solutions only contains those solutions that do not predict any contradictions. 
However, enforcing box consistency directly is non-convex, due to the need of computing intersection boxes (see Section~\ref{sec:discussion}). Thus, we need to define a convex alternative that guarantees non-overlap of $\cAssign{C}$ and $\cAssign{\neg C}$. \color{black}
%, e.g., $A(a)$ and $\neg A(a)$. %commented since $\neg A(a)$ is not part of our current syntax
%; and $(ii)$ ensures that these consistency constraints together with the translated TBox inequalities (see Section~\ref{sec:TBox_Translation}) do not make the set of feasible solutions empty. 
%Property (2) states that we do not constrain our embedding to much with the consistency inequalities, i.e., for any axiom $\phi$ that is consistent with $\Tmc$, we have that our embedding solution together with the consistency inequalities and the inequalities derived by $\phi$ remains satisfiable. 
We ensure Point (i) by   reserving one dimension $i_C$ per $C \in \NCE$ and   defining one constraint per dimension $i_C$, formally described in (\ref{eq:concept_cons}). By reserving one dimension $i_C$ per $C \in \NCE$, the minimal number of dimensions of our method depends on the TBox. %ontology. 
In particular, our method requires at least $|\NC| + 2|\NR|$ dimensions \footnote{We have $2|\NR|$   because of  $\exists R$ and $\exists R^-$ for each $R\in\NR$.} for a TBox 
%an ontology 
with $|\NC|$ concept names and $|\NR|$ role names.
%\todo{Aleks: Make the dimensions more explicit here: This means that the minimal number of dimensions for an ontology is equal to ... }
\begin{equation}
        \frac{\mathbf{L_{C}}[i_C] + \mathbf{U_{C}}[i_C]}{2} \leq-\frac{\worldSizeScalar}{2}
        \label{eq:concept_cons}
\end{equation}
% \todo[inline]{B: I put the part below in red because it should be carefully checked.}
% \todo[inline]{Aleks: This looks very good :) - I am just wondering whether we should put this at another place or make it to a small lemma.}
The inequality in \eqref{eq:concept_cons}  ensures box consistency (see  Theorem~\ref{thm:propembedding} in the appendix).

\textbf{Box Width Constraints.} %\todo[inline]{Ana: we said above that concept boxes span outside the universe, this is confusing}
As mentioned in Point (ii), our method assumes that the width of any box  
%$b \in \B$ with , 
with lower and upper bounds $\mathbf{L}$ and $\mathbf{U}$ is %$(i)$
positive and %$(ii)$ 
bounded by $2\mathbf{\worldSize}$, i.e., $\mathbf{0} \leqd \mathbf{U} - \mathbf{L} \leqd 2\mathbf{\worldSize}$. % We can assume without loss of generality the more restrictive property that boxes of positive concepts are within the universe box $\World$. 
The following inequalities ensure this requirement:
\begin{equation}\label{eq:conUn}
\begin{aligned}
      \mathbf{0} \leqd \mathbf{U_C} - \mathbf{L_C} \leqd 2 \mathbf{\worldSize} \\
      \mathbf{0} \leqd \mathbf{U}^X_\mathbf{R} - \mathbf{L}^X_\mathbf{R} \leqd 2 \mathbf{\worldSize}, 
    % \forall C \in \NCE\colon &\mathbf{L_C} \geqd -\mathbf{\worldSize} \\
    % \forall C \in \NCE\colon &\mathbf{U_C} \leqd \mathbf{\worldSize},
\end{aligned}
\end{equation}
% \todo{Should I use $\NC$ or $\NCE$ above?}
where $C \in \NC$, $R \in \NR$, 
and $X \in \{\mathbf{H}, \mathbf{T}, \mathbf{B}\}$.
% \todo[inline]{check if here it is \NRm or \NR}
%where $\NCE \coloneqq \NC \cup \{\exists R, \exists R^- | R \in \NR\}$.
% Additionally, we can assume without loss of generality that all bump boxes are bounded by $\worldSize$:

% \begin{align*}
%     \forall R \in \NR\colon &\mathbf{L^B_R} \geqd -\mathbf{\worldSize} \\
%     \forall R \in \NR\colon &\mathbf{U^B_R} \leqd \mathbf{\worldSize}
% \end{align*}
%
\textbf{Universe Constraints.} %\todo{say this before}
As mentioned in Point (iii), our method assumes that positions and bumps of individual embeddings are within the universe box $\World$ (see \cref{def:boxinterpretation}). The following inequalities ensure this property: %The following inequalities ensure this basic property:
% \begin{align*}
%     \forall a \in \NI:& -\mathbf{\worldSize} \leqd \posAssign{a} \leqd \mathbf{\worldSize}\\
%     \forall a \in \NI:& -\mathbf{\worldSize} \leqd \bumpAssign{a} \leqd \mathbf{\worldSize}
% \end{align*}
\begin{equation}
\begin{aligned}\label{eq:indUn}
    %\forall a \in \NI\colon
    & 
    -\mathbf{\worldSize} + \boldsymbol{\epsilon} \leqd \posAssign{a} \leqd \mathbf{\worldSize} - \boldsymbol{\epsilon}\\
    %\forall a \in \NI\colon
    &
    -\mathbf{\worldSize} + \boldsymbol{\epsilon} \leqd \bumpAssign{a} \leqd \mathbf{\worldSize} - \boldsymbol{\epsilon},
\end{aligned}
\end{equation}
where $a\in\NI$.
% Let us first recall the embeddings of complex concepts. 

\textbf{TBox Axiom Constraints.} 
Finally, to guarantee the satisfaction of TBox axioms in feasible solutions (Point (iv)), first we need to define the boundaries of concept and role embeddings. Based on the definitions in  Section~\ref{sec:boxSemantics}, the boundaries of  $\eta(\neg{C})$, $\eta(\exists R)$, and $\eta(\inv{R})$ are defined as follows: 
\begin{align*}
    \mathbf{L_{\neg{C}}} &\coloneqq -\worldSize - \mathbf{L_{C}},\,\, \mathbf{U_{\neg{C}}} \coloneqq \worldSize - \mathbf{U_{C}},\\
    \mathbf{L_{\exists R}} &\coloneqq  \mathbf{L^H_R} - \mathbf{U^B_R}, \,\,   \mathbf{U_{\exists R}} \coloneqq \mathbf{U^H_R} - \mathbf{L^B_R}, \\
    \headAssign{\inv{R}} &\coloneqq \tailAssign{R},\,\,   \tailAssign{\inv{R}} \coloneqq \headAssign{R}, \\
    \bumpBoxAssign{\inv{R}} &\coloneqq \bumpBoxAssign{R}.
\end{align*}

Let $B\in \NCE$, $C \in \NCEneg$, and $R,S\in \NRm$. To satisfy concept and role inclusions within any embedding solution $z$ (Definition~\ref{def:BoxSatisfaction}), we employ the following linear inequalities: 
\begin{itemize}
    \item $B \sqsubseteq C$ corresponds to $\mathbf{L_C} \leq_d \mathbf{L_B}$ and $\mathbf{U_B} \leq_d \mathbf{U_C}$, 
    \item $R \sqsubseteq S$ corresponds to $\mathbf{L}^X_\mathbf{S} \leq \mathbf{L}^X_\mathbf{R}, \mathbf{U}^X_\mathbf{R} \leq \mathbf{U}^X_\mathbf{S}$ with $X \in \{\mathbf{H}, \mathbf{T}, \mathbf{B}\}$. 
\end{itemize}

\textbf{Problem Size of the Translation.}
We point out that the size of \modelName's problem formulation increases linearly w.r.t. the size of the KB \Kmc and  $|\NC\cup\NR\cup\NI|$. Let $\mathcal{T}$ be $\Kmc$'s TBox,   the number of inequalities, as described in this section, gives us 
%$O(4|\NI| + 3|\NC| + 8|\NR| + 2|\mathcal{T}|)$
$O((|\NI| + |\NC| + |\NR| + |\mathcal{T}|)d)$
constraints.

\subsection{Optimization: Ranking Assertions}
\label{sec:Scoring}
% \todo{Think about calling this section "Scoring: Truth of ABox axioms"}
%\todo{talk about likelihood instead of plausibility? - e.g., more likely ABox axioms}
When ranking assertions, we would like the distance between the position of an individual % $a$
w.r.t to a box associated with a concept name %$A$ 
to affect the score of the corresponding assertion: %$A(a)$: 
the smaller the distance to the elements of the box, the higher the score. The   same intuition also applies for roles, but needs to take into account the way role assertions are satisfied in the embedding model~(Definition~\ref{def:BoxSatisfaction}).
%Similarly, for roles, we would like 
%and bump vectors of  individuals to affect the scores 
%give higher scores to those assertions where the distance  
%that if those individuals mapped to vectors 
%Here we explain how 
%the plausibility of missing ABox axioms is reflected by the scoring function and how this leads to the objective of our optimization problem.
%In KB completion, the goal is typically to rank assertions by plausibility. Merely checking whether an assertion is satisfied by an embedding, however, yields only a binary outcome and does not allow for fine‑grained ranking. To solve this,
%
%problem, we have carefully designed the following 

To express this intuition,
%For this, 
we employ a \emph{signed distance} function.  
The signed distance assigns nonpositive values to the assertions satisfied in the embedding, and positive values to those that are violated. Moreover, it reflects \emph{how}   an assertion is satisfied (or violated): for concept assertions, individuals with embeddings deeper inside the   box have a smaller negative value than those closer to the border, while   individuals with  embeddings outside the box  have positive values, and those values are higher the farther away the embeddings of the individuals are from the box. 

%To express this intuition,
%we design a distance function that varies the score based on . 

%\subsection{Optimization: Plausibility of Assertions}
%\label{sec:Scoring}
% \todo{Think about calling this section "Scoring: Truth of ABox axioms"}
%\todo{talk about likelihood instead of plausibility? - e.g., more likely ABox axioms}
%Here we explain how the plausibility of missing ABox axioms is reflected by the scoring function and how this leads to the objective of our optimization problem.
%In KB completion, the goal is typically to rank assertions by plausibility. Merely checking whether an assertion is satisfied by an embedding, however, yields only a binary outcome and does not allow for fine‑grained ranking. To solve this,
%
%problem, we have carefully designed the following 
%we design a distance function that provides graded plausibility scores. 

% base our scoring function on two distance functions, one for computing whether an assertion is satisfied (box distance) and one for ranking the plausibility of  assertions that are satisfied in the embedding solution (center distance). % --- that we use for both optimization and scoring. 
% However, the concept and role assertion loss terms above do not provide a \emph{complete} ranking but assign a loss of zero to the embeddings of any assertion that is (according to Definition \ref{def:BoxSatisfaction}) satisfied in our embedding solution.

% \todo[inline]{say below L and U are bounds for box B}

\textbf{Signed Distance.} {  We move on to the formal definition of signed distance to a set.} %The signed distance to a  set assigns a positive value to elements that are not in the set and a nonpositive value to elements that are in the set.
%In contrast to the usual Euclidean distance, its value may be negative.
%Let $S \subseteq \Re^n$ be a set. 
First, the \emph{Euclidean distance   to a set $S\subseteq \Re^n$} is defined as the function $\dist:\Re^n \to \Re$ such that 
$\dist(y,S) \coloneqq \inf\{\norm{x-y}_2 \mid x \in S\}, \ \forall y \in \Re^n.$
Let $S^c$ be $\Re^n\setminus S$.
%be the complement of $S$ in $\Re^n$. 
Then, the \emph{signed distance}  (also called \emph{oriented distance})   to $S$ is %defined as
\begin{equation}\label{eq:sdist_def}
\sdist(y,S) \coloneqq 
\begin{cases}
\dist(y,S) & \text{ if } y \not \in S \\
-\dist(y,S^c) & \text{ if } y  \in S,
\end{cases}  
\end{equation}
e.g., see Chapter~7 of \citet{DZ11}.
%In what follows, 
Given $y \in \mathbb{R}^n$, we 
denote by $y^+ \in \Re^n$ the vector that corresponds to replacing the negative components of $y$ by $0$. %so that $y^+ \in \Re^n$ 
%satisfies $(y^+)_i = \max(y_i,0)$ for every $i \in \{1,\ldots,n\}$.
%Similarly, we denote by $y^-$ the nonpositive part, 
 E.g., $(1,-2,3,-4)^+ = (1,0,3,0)$.
 %for $y = (1,-2,3,-4)$ we have 
%$y^+ = (1,0,3,0)$.
%With that, we have the following proposition. 
%Here, we make use of the signed distance w.r.t. %with respect  
%Let $\Re^n_-$ be $\{y \in \Re^n \mid y_i \leq 0, 1\leq i \leq n\}$.
\begin{restatable}{proposition}{thmsign}
The signed distance to  $\Re^n_-$ satisfies
\begin{equation}\label{eq:sdist_def_rn}
\sdist(y,\Re^n_-) = 
\begin{cases}
\norm{y^+}_2 & \text{ if } y \not \in \Re^n_- \\
\max_{i\in \{1,\ldots,n\}} y_i & \text{ if } y  \in \Re^n_-,
\end{cases}  
\end{equation}
where $\Re^n_-=\{y \in \Re^n \mid y_i \leq 0, 1\leq i \leq n\}$.
\end{restatable}

%which satisfies
%\begin{equation}\label{eq:sdist_def_main}
%\sdist(y,\Re^n_-) = 
%\begin{cases}
%\norm{y^+}_2 & \text{ if } y \not \in \Re^n_- \\
%\max_{i\in \{1,\ldots,n\}} y_i & \text{ if } y  \in \Re^n_-,
%\end{cases} 
%\end{equation}
%see Appendix~\ref{app:signed} for more details.  

Let $\mathbf{B}$ be a box with bounds $\mathbf{L}$ and $\mathbf{U}$. 
We define ${\sf dist}(\mathbf{B}, \mathbf{x})$ 
as the signed distance of 
the concatenated vector 
$(\mathbf{L}+\epsilonVec-\mathbf{x})\oplus (\mathbf{x}-\mathbf{U}+\epsilonVec)$ to 
$\Re^{2d}_{-}$.
That is, 
\[
{\sf dist}(\mathbf{B}, \mathbf{x}) \coloneqq 
\sdist((\mathbf{L}+\epsilonVec-\mathbf{x})\oplus(\mathbf{x}-\mathbf{U}+\epsilonVec),\Re^{2d}_-).
\]
We now define the loss 
%and regularization
terms using ${\sf dist}$.

\textbf{Concept Assertion Loss.} %Based on $dist$, 
We define the loss for concept assertions $D(a)$ as follows:
\begin{equation*}
\mathcal{L}_{\sf concept}(D, a) \coloneqq {\sf dist}(\cAssign{D}, \posAssign{a}).
\end{equation*}

Intuitively,  minimizing the  loss $\mathcal{L}_{\sf concept}(D, a)$ of a concept assertion $D(a)$ pushes an individual's position embedding $\posAssign{a}$ into $D$'s concept embedding $\cAssign{D}$.

\textbf{Role Assertion Loss.} Similarly, the loss $\mathcal{L}_{\sf role}(S, a, b)$ of a role assertion $S(a, b)$ pushes the translated individual embedding $\posAssign{a} + \bumpAssign{b}$ and, respectively,  
$\posAssign{b} + \bumpAssign{a}$ into $S$'s head box $\headAssign{S}$ and $S$'s tail box $\tailAssign{S}$. Furthermore, it pushes the bumps of $a$ and $b$ into $S$'s bump box $\bumpBoxAssign{S}$. Based on the distance function ${\sf dist}$, we define the loss for role assertions $S(a,b)$ as follows:
%
% \todo{Change to maximum}
\begin{equation*}
\begin{aligned}
\mathcal{L}_{\sf role}(S, a, b) \coloneqq %\frac{1}{2} \big( 
\max\big(&{\sf dist}(\headAssign{S}, \posAssign{a} + \bumpAssign{b}), \\
&{\sf dist}(\tailAssign{S}, \posAssign{b} + \bumpAssign{a}), \\
&{\sf dist}(\bumpBoxAssign{S}, \bumpAssign{a}), \\
&{\sf dist}(\bumpBoxAssign{S}, \bumpAssign{b}) \big).
\end{aligned}
\end{equation*}%

% \todo{Aleks: Add regularization terms here}

% \textbf{Regularization Terms.} 
% To allow for convex optimization our model cannot use standard negative sampling procedures, such as self-adversarial negative sampling \citep{RotatE}, since they are neither convex nor determistic. Thus, \modelName needs a different way to favour predicting low scores for false assertions and high scores for true assertaions. 

% Thus, given a box $\mathbf{B} \in \B$, we define a prediction sparsity term $\mathcal{R}_{\sf width}(\mathbf{B})$ that penalizes large widths for box embeddings $\mathbf{B}$, shrinking the boxes and making it less likely that high scores are predicted for an assertion by chance. 

%\todo{Ana please check this sentence:}
%\todo{small change, removed plausib, likelyhood}
In practice, KBs typically contain only positive assertions and no explicit negative ones. Therefore, if we were to minimise only the loss terms of the assertions, \modelName would tend to assign high scores to all potential assertions,  not distinguishing between true and false assertions. A common way to address this issue in KB embedding methods is \emph{negative sampling} \citep{TransE,RotatE,BoxE}. In that approach, given a role assertion $R(h, t)$ contained in the KB, one constructs a \emph{corrupted} assertion by exchanging either the head $h$ or tail $t$ of the role assertion by a randomly chosen individual from $\NI$. Since the number of assertions that hold is usually much smaller than the number of all possible assertions that can be made, most corrupted assertions tend to be false. 
%Since the ABox of KBs is typically sparse, these corrupted assertions are assumed to be false with high probability. 
KB embedding models are then trained to assign high scores to ABox assertions in the training data and low scores to corrupted ones.
%\todo{this may change, depending on neg samp}
However, negative sampling %has  drawbacks:
%(i) it is inherently stochastic, since enumerating all negative assertions is computationally infeasible;
%(ii) 
%it 
leads to nonconvex loss terms, which are often incompatible with convex optimization (see Section~\ref{sec:discussion}); and
%(iii) 
it is computationally expensive, as competitive performance often requires sampling hundreds or even thousands of negative assertions per ABox assertion \citep{BoxE,DensE,ExpressivE}. Next, we adapt negative sampling to concepts in \NCE, leading to convex regularization terms that are fast to compute and 
%addresses issues (i) and (iii). We address (ii) by 
push individual embeddings into complement boxes as required.
%\todo{this point ii we may want to move it out from here and put in another part talking about convex opt}
%Instead, we enforce sparsity in the set of predicted   assertions by introducing two convex regularization terms.

% Given the set of assertions known to be true (i.e., those within the ABox) and the set of corrupted assertions, KB embedding models are typically optimized to assign high scores to ABox assertions and low scores to corrupted ones. However, negative sampling $(i)$ is stochastic (as sampling any possible assertion would be computationally infeasible), $(ii)$ leads to nonconvex loss terms, meaning that we cannot use them directly for convex optimization, $(iii)$ lowers the efficiency of current KB embedding models drastically, as typically $100$ to $1000$ negative assertions need to be sampled per positive assertion to achieve competitive KB completion performance \citep{BoxE,DensE,ExpressivE}. We enforce sparsity in our predicted plausible assertions by adding the following two convex regularization terms. 

\textbf{Negative Concept Regularization.} 
%First, observe that the number of assertions entailed by a KB is typically much smaller than the set of all possible assertions that can be formed from elements of $\NI$, $\NC$, $\NR$. 
%As discussed above, without additional measures, the objective function would likely optimize our embedding model to assign high plausibility scores to many unlikely assertions. To counteract this effect, we introduce a regularization term that penalizes overly plausible concept assertions. 
Our approach builds on three observations: (1) directly pushing individual embeddings outside a box (i.e., into its geometric complement) leads to nonconvex optimization terms; (2) \modelName introduces convex boxes representing the  complement of concept embeddings, which we can use to regularize the scores; and (3) \modelName does not offer convex representations for the complement of role embeddings, yet a role's domain and range can be expressed as existential concept embeddings. Based on these observations, we introduce a negative concept regularization term that for any potential concept assertion $D(a) \not \in \Amc$ with $D \in \NCE$ and $a \in \NI$ pushes $\posAssign{a}$ into $\sComp{\cAssign{D}}$, reducing the score of $D(a)$. This regularization term keeps plausibility scores for arbitrary assertions low, while the assertion loss terms selectively increase the scores of ABox assertions explicitly contained in the KB:
% any concept $D \in \NCE$ we regularize any individual embeddings to be pushed into its complement box. 
% Thereby, we introduce an additional force in our objective function that tries to push all
% \todo{Work on this tomorrow}
% \todo{Remove Abox from negatively sampled assertions}
\begin{align}
    \label{eq:negSampling}
    \mathcal{L}_{\sf negative}(D,a)\hspace{2pt}\coloneqq \mathcal{L}_{\mathsf{concept}}(\neg D,a).\hspace{2pt}\hspace{2pt}
\end{align}
\textbf{Box Width Regularization.} 
A second strategy to keep scores for arbitrary assertions within a reasonable range is to regularize the box size of concept and role embeddings.
Specifically, for any concept name $D \in \NC$ we regularize the size of its concept box $\cAssign{D}$; and for any role name  $S \in \NR$, we regularize the size of its head $\headAssign{S}$, tail $\tailAssign{S}$, and bump box $\bumpBoxAssign{S}$. Formally, we define the box regularization term $\mathcal{R}_{\sf width}(\mathbf{B})$ of a box $\mathbf{B}$ with bounds $\mathbf{U}$ and $\mathbf{L}$, as:
\begin{align}\label{eq:rb}
\mathcal{R}_{\sf width}(\mathbf{B}) \coloneqq \| \mathbf{U} - \mathbf{L} \|_2.
\end{align}
\textbf{Objective Function.} We now assemble these loss and regularization terms into our objective function: % $\mathcal{L}(\eta)$:
\begin{equation*}
\begin{split}
 %\coloneqq 
 &\hspace{-14pt}\max\left(\max_{D(a) \in \Amc}\mathcal{L}_{\sf concept}(D, a) , %\hspace{-10pt} 
\max_{S(a,b) \in \Amc}\mathcal{L}_{\sf role}(S, a,b)\right) \hspace{2pt}+ \\
\lambda_1 &\hspace{-7pt}\max_{D(a)\in  N_C^\exists\setminus \Amc} \mathcal{L}_{\sf negative}(D,a)\hspace{2pt}\hspace{2pt}+ \\
\lambda_2 \Big(&\sum_{D \in \NC } \mathcal{R}_{\sf width}(\cAssign{D})  \hspace{2pt}+ \\
&\sum_{S \in \NR} \mathcal{R}_{\sf width}(\headAssign{S})+ 
\mathcal{R}_{\sf width}(\tailAssign{S}) \Big)
\hspace{2pt}+\\
\lambda_3 
&\sum_{S \in \NR}\mathcal{R}_{\sf width}(\bumpBoxAssign{S}).
%&\lambda_i \Big(\sum_{D(a) \in \Amc} \mathcal{R}_I(\cAssign{D}, \posAssign{a})\hspace{2pt}+ \\
%&\hspace{13pt}\sum_{S(a,b) \in \Amc}\big( \mathcal{R}_I(\headAssign{S}, \posAssign{a} + \bumpAssign{b})\hspace{2pt}+\\ 
%&\hspace{55pt}\mathcal{R}_I(\tailAssign{S}, \posAssign{b} + \bumpAssign{a})\hspace{2pt} %+
% &\hspace{55pt}\mathcal{R}_{\sf width}(\bumpBoxAssign{S}, \bumpAssign{a})\hspace{2pt}+\\
% &\hspace{55pt}\mathcal{R}_{\sf width}(\bumpBoxAssign{S}, \bumpAssign{b})
%\big)\Big)
\end{split}
\end{equation*}
%\todo[inline]{add sentence about width}
%The terms in the first line encode the penalty of not satisfying the assertions in the ABox used as training data.

%\todo[inline]{the comment below raises more questions than answers.}
%We use two separate hyperparameters, $\lambda_2$ and $\lambda_3$, to regularize box widths. This distinction is theoretically motivated, as the head and tail boxes lie in the same embedding space (entity position embeddings translated by bump embeddings), whereas the bump box is defined in a different space (bump embeddings only). In addition, our experiments provide empirical evidence that using different values for $\lambda_2$ and $\lambda_3$ leads to prediction performance gains.

\textbf{Scores.}
To rank the assertions, we define two scoring functions, one for concept and one for role assertions. 
%Following the literature \citep{BoxE,ExpressivE,SpeedE,ReshufflE}, we assign higher scores for more plausible and lower scores for less plausible assertions.

%\textbf{Concept Assertion Score.} 
% \todo{Remove the center distance terms}
%Based on the loss, we define 
The score $s(D, a)$ of concept assertions $D(a)$ is:
\begin{align*}
  s(D, a) \coloneqq & 
  -  \mathcal{L}_{\sf concept}(D, a)
%  -{\sf dist}(\cAssign{C}, \posAssign{a}).    
\end{align*}
%
%
%\todo{@Aleks: Change Score - The nonconvex score seems to perform better!}
%\textbf{Role Assertion Score.} Based on the loss, we define 
%
and the score $s(S, a, b)$ of role assertions $S(a,b)$ is: 
% \todo{Remove the center distance terms}
%
\begin{align*}
  s(S, a, b) \coloneqq & -\mathcal{L}_{\sf role}(S, a, b).
 % -\max\big( &{\sf dist}(\headAssign{R}, \posAssign{a} + \bumpAssign{b}),\\ 
  %&{\sf dist}(\tailAssign{R}, \posAssign{b} + \bumpAssign{a}), \\
  %&{\sf dist}(\bumpBoxAssign{R}, \bumpAssign{a}),\\
  %&{\sf dist}(\bumpBoxAssign{R}, \bumpAssign{b}\big).
\end{align*}
The scoring functions of the concept and role assertions are the negative of the corresponding loss functions. 
The intuition for this is that solving the optimization problem, i.e., minimizing the concept and role assertion loss for ABox assertions, maximizes their score.

\subsection{DL-Lite$^\Hmc$   KB Faithfulness}
\label{sec:WeakFath}
%As we shall see in Section \ref{sec:WeakFath}, 
Translating TBox axioms to linear inequalities that are used as convex constraints in the optimization problem (see Section~\ref{sec:ProblemFormulation}), guarantees any \modelName embedding solution $z$ to satisfy the concept inclusions in the  TBox  and to be KB faithful for DL-Lite$^\Hmc$.
Let $\convexSet_{\Kmc} \subseteq \Re^n$ be the set of $z$'s such that the constraints of Section~\ref{sec:ProblemFormulation} are satisfied. That is, $z \in \convexSet_{\Kmc}$ if and only if:
$(a)$ for each TBox axiom the corresponding inequalities are satisfied; and $(b)$ the box consistency and universe constraints are satisfied, see the appendix for details.
%\begin{enumerate}[$(a)$]
  %  \item 
 %   \item the box consistency and universe constraints are satisfied (see Section~\ref{sec:consistency}), i.e., individuals within the universe (see \eqref{eq:indUn} and \eqref{eq:conUn}), {and concepts satisfy the inequality in~\eqref{eq:concept_cons}}.
%\end{enumerate}
Also, let
$f_{\hyperP}:\Re^n\to \Re$ be the function that maps $z$ to the objective value  for a given choice of 
nonnegative hyperparameters $\hyperP = (\lambda_1,\lambda_2 ,\lambda_3) \in \Re^3_+$. 
%\todo{B: I believe we no longer have the hyperparameter $\mu$, so I removed it}%, where $\lambda_1$ is the hyperparameter associated to the loss terms and $\lambda_2,\lambda_3,\lambda_4, \lambda_5, \mu$ are associated to regularization terms.
%With that, we have the following theorem.

%\todo{@Bruno: I would greatly appreciate your feedback on whether this is enough detail for Theorem 4. I placed your highly detailed explanation in the supplemental material for the interested reader and for comparisons to this section.}

% Before we state the next theorem, we recall that a set in $\Re^n$ is said to be \emph{polyhedral} if it  can be written as the set of solutions of finitely many linear equalities/inequalities.
%In particular, polyhedral sets are convex.

\begin{restatable}{theorem}{theoremConvex}\label{theo:convex}
Let $\Kmc$ be a satisfiable DL-Lite$^\Hmc$ KB.
For nonnegative $\hyperP$ the following optimization problem is convex.
%\begin{equation}
%\min_{z \in \Re^n} \quad  f_{\hyperP}(z), \qquad  \textup{subject to}\quad z \in \convexSet_{\Kmc}.  \label{eq:opt_ont}
%\end{equation}
\begin{align}
\min_{z \in \Re^n} \;  f_{\hyperP}(z), \quad
\textup{subject to} \; z \in \convexSet_{\Kmc}. %\notag
  \label{eq:opt_ont}
\end{align}
In particular, $f_{\hyperP}$ is a convex function,  $\convexSet_{\Kmc}$ is a polyhedral set and the following items hold.
\begin{enumerate}[align=left,label=\roman*),leftmargin=*]  %[$(i)$]
    % \item For $d \geq |\NC|+2|\NR|$, $\convexSet_{\Kmc}$ is nonempty.
    % \item For $d \geq |\NI|+|\NC|+2|\NR|+|\NR||\NI|$ and $\lambda_{i}=\lambda_{n} = 0$, the optimal value is zero.
    \item For $d$ as in Corollary~\ref{cor:ConsistencyDim}, and $\worldSizeScalar$ as in Theorem~\ref{thm:faith_bool_alc},
%    $\worldSize \coloneqq (4,\ldots, 4) \in \Rd$ and $\epsilon \in (0,\epsilonMax]$ (see Section~\ref{sec:boxSemantics}),
$\convexSet_{\Kmc}$ is nonempty.
   \item Any embedding solution $z \in \convexSet_{\Kmc}$ corresponds to a box consistent  interpretation that is  TBox faithful.
   %\geometric interpretation.    
    \item If $\lambda_1=0$ and there is an optimal solution $z^*$ such that $f_{\hyperP}(z^*) \leq 0$  then the \geometric interpretation corresponding to $z^*$ is  KB faithful. 
    
    \item Suppose that $\lambda_{1}=\lambda_{2}=\lambda_{3} = 0$.
    For $d$ as in Corollary~\ref{corollary:FaithfulnessDim},
    %$d \geq |\NC| + 2|\NR| + |\NR| (|\NC| + 2|\NR| + |\NI|)$,
    and $\worldSizeScalar$ as in Theorem~\ref{thm:faith_bool_alc}, there is an optimal solution $z^*$ s.t. 
   $f_{\hyperP}(z^*) \leq 0$ holds.
    
%    \item {\color{red}If the optimal value is zero, then the geometric model is TBox strongly faithful. \todo{This requires adding witnesses in model that come from the canonical model}}
\end{enumerate}
\end{restatable}
Informally, Theorem~\ref{theo:convex} states that we can find a \geometric interpretation for a {satisfiable  $\Kmc$}  via convex optimization over a polyhedral set.  
Any $z \in \convexSet_{\Kmc}$ (whether optimal or not)   corresponds to a \geometric interpretation $\eta$ that is     TBox faithful and  box consistent. 
Item~$i)$ gives a bound on the minimum required $d$ to ensure that $\convexSet_{\Kmc}$ is nonempty, but this estimate seems to be conservative.
Items~$iii)$ and $iv)$ imply that  if we wish to find a solution that is   KB faithful, this can be done by setting the hyperparameters associated to regularization terms to $0$, setting $d$ to be sufficiently large and solving \eqref{eq:opt_ont}. 

Finally, it turns out that %we remark that more can be said about \eqref{eq:opt_ont} and, in fact, Theorem \ref{theo:socp} shows that 
\eqref{eq:opt_ont}  can be formulated as a \emph{second-order cone program} (SOCP) \citep{LVBL98}, \cite[Lecture~3]{BtN01}.

\begin{restatable}{theorem}{theoremSOCP}\label{theo:socp}
For nonnegative $\hyperP$, the problem in \eqref{eq:opt_ont} can be reformulated as an equivalent SOCP.
\end{restatable} 
Theorem \ref{theo:socp} is important because it shows that \eqref{eq:opt_ont} can be
 solved efficiently via high-quality open-source solvers such as SeDuMi \citep{Sedumi99} and SDPT3 \citep{TTK03} or commercial solvers such as Gurobi \citep{Gurobi23} and MOSEK \citep{mosek}. 
 The conversion of a problem as in \eqref{eq:opt_ont} to a SOCP that can be handled by the aforementioned solvers, although tedious, can be automated by modelling tools for convex optimization such as CVXPY \citep{AVDB18}. %, which we use in our experiments in Section~\ref{sec:exp}.

%In particular, \emph{interior-point methods} (IPMs) offer a blend of practical performance and theoretical guarantees \citep{NN94,Re01} and are  implemented in both commercial \citep{Gurobi23,mosek} and open-source solvers \citep{Sedumi99,TTK03}. 
%The precise statement of the theoretical guarantees ensured by IPMs is beyond the scope of this paper, but in simplified terms, under certain mild technical conditions,  IPMs can find solutions that are \emph{globally optimal} up to a user-defined accuracy. 
%Furthermore, the number of iterations is a polynomial of the logarithm of the desired accuracy \textcolor{red}{and the size of the problem~\cite[Section~3.8]{Re01}. }

\section{Proof of Concept}\label{sec:exp}
%\todo{values need to be updated}
Here, we present empirical evidence for the theoretical foundations established so far.
% in the previous sections. 
% In detail, i
%In Section \ref{sec:Exp_Strong_Faithfulness}, we experimentally test whether \modelName's problem formulation ensures strong DL-Lite$_{core}$ TBox faithfulness on a prototypical KB. Next, 
%In Section \ref{sec:Exp_Assertion_Experiments}, 
We evaluate \modelName's performance 
on subsets of the Family dataset \citep{FamilyDataset}, providing first results for its  scalability
and reasoning capabilities. 
In our experiments, we only consider KBs with satisfiable concepts.
%TODO: Add the Github link

%for simplicity of code.
\textbf{Reproducibility.} We implemented \modelName's %(convex) second-order cone 
optimization problem 
% (see Section \ref{sec:opt} for details) 
in Python 3.12 using CVXPY \citep{DB16,AVDB18} for modeling and MOSEK \citep{mosek} for solving it. 
%\todo{update spec with Hesham computer spec}
{In our evaluation we include experiments with other KBE approaches using classical stochastic gradient descent (SGD). We use PyKEEN 1.11.1 \citep{ali2021pykeen} in these experiments.}
We ran each of our experiments on an
Apple Mac Mini Desktop Computer with M4 Chip with 10 Core CPU and 10 Core GPU: 16GB (Shared Memory).
%an Intel(R) Xeon(R) Silver 4314 CPU @ 2.40GHz of our internal cluster, using up to $32$ threads. 
More details can be found  %and list additional details on our hardware, libraries, experimental setup, hyperparameters, and definitions of metrics 
in the appendix and the code is available at \url{https://github.com/AleksVap/BoxLitE}.
%Finally, we include our code in the supplementary material to facilitate the reproducibility of the results presented next.
%\subsection{
%{Performance}
%Scalability, Representation, % can we comment this?
%and Reasoning
%}
%\label{sec:Exp_Assertion_Experiments}
%\todo{suggestion: Performance (one line titles usually look better)}
We
%briefly analyze the scalability, representation, and reasoning capabilities of our \modelName model. In particular, we want to answer the 
focus on the following questions.

\begin{enumerate}[(Q1),leftmargin=*]
   
    \item 
    %How well does our model learn to 
    %reason ontologically, i.e., 
    %predict assertions that can be inferred from the KB 
    What is the reasoning performance
    and 
    what is the effect of the regularization terms on the results?
    \item How does increasing the dataset size affect the prediction performance, compilation, and solving time?
     \item How well does our method compare with classical  embedding methods based on stochastic gradient descent?
\end{enumerate}

%\todo{ % Note that we have strong inequalities "<" for box consistency and "<=" for box definitions
%}
\textbf{Experimental Setup.} To answer each of these questions, we have created a set of datasets (F\_v1-4) of varying sizes from the family dataset \citep{FamilyDataset}. We derived these datasets by sampling $k$ assertions of the family dataset's ABox with forest fire sampling \citep{ForestFireSampling}, a popular sampling technique for large graphs. Furthermore, since the family dataset solely provides role assertions in its ABox, we selected all concept inclusions in the family dataset that only include roles and extended the TBox by the disjointness axiom $\exists {\sf hasFather^-} \sqsubseteq \neg \exists {\sf hasMother^-}$. We list the TBox of the created datasets in Figure~\ref{fig:familyTBox}.

% \textbf{Properties.} % Furthermore, Table \ref{tab:FamilyRoleProperties} provides the number of assertions per role in the ABox of the Family10k benchmark.

\textbf{Evaluation Setup.} To evaluate \modelName's %link prediction 
performance, we created a set of inferred role assertions by $(i)$ adding any role assertion that logically follows from each dataset and $(ii)$ removing any assertion that occurs in the ABox.
%training part of the dataset.
%ABox\todo[inline]{not clear   what is the ABox, is it all the facts in the dataset or just the part used for training?}. 
We randomly split this set %of inferred assertions 
into a validation set 
% \todo[inline]{in table 1 we have train, val, test, say explicitly the split between train, val, test?}
(20\%), used for model selection, and a test set (80\%), used for evaluating the performance of the selected model. % We evaluate \modelName's performance on %$(i)$ the ABox, revealing how well the embedding solutions fit the training data (for Q1) and $(ii)$ the link prediction task. %testing how well the embedding solutions learn to reason over the ontology. 
We use the standard evaluation setting for KB completion \footnote{%We optimize our model on the train, select the best model using the validation, and evaluate its performance on the test set. 
The evaluation of a KB embedding model typically needs a set of true and corrupted role assertions. True role assertions $R(a,b)$ of the KB are corrupted by replacing $a$ or $b$ by any $c \in \NI$ such that the corrupted assertion is not within the KB. %KB embeddings are optimized to score true assertions higher than false ones, thereby estimating a given assertion's plausibility. 
The performance of KB embedding models is typically measured using the filtered versions \citep{TransE} of the \emph{mean reciprocal rank} (MRR)
%, the average of inverse ranks ($1/\textit{rank}$) 
and H@k, the proportion of true assertions within the predicted assertions whose rank is at maximum k.} \citep{BoxE,AMW2024,BoxEL}.

\textbf{Dataset Properties.} Table \ref{tab:FamilyDatasetProperties} lists the number of assertions and individuals of the train, validation, and test sets of F\_v1-4. We sampled each of these datasets individually. Thus, although datasets F\_v1-4 gradually increase in size, they are   different from each other. %, which may lead to non-monotonic performance changes, as we will see next.

\begin{figure}[h!]
    \centering
    \begin{align*}
    &{\sf relative^-} \sqsubseteq {\sf relative} &{\sf hasSibling} \sqsubseteq {\sf relative}\\
    &{\sf hasChild} \sqsubseteq {\sf relative} &{\sf hasParent} \sqsubseteq {\sf relative}\\
    &{\sf hasFather} \sqsubseteq {\sf hasParent} &{\sf hasMother} \sqsubseteq {\sf hasParent}\\
    &{\sf spouse^-} \sqsubseteq {\sf spouse} &{\sf hasSibling^-} \sqsubseteq {\sf hasSibling} \\
    &{\sf spouse} \sqsubseteq {\sf relative} &\exists {\sf hasFather^-} \sqsubseteq \neg \exists {\sf hasMother^-}
    \label{onto:extension1}    
    \end{align*}
    \caption{TBox of datasets F\_v1-4.}
    \label{fig:familyTBox}
\end{figure}

\begin{table}[t]
%\resizebox{\columnwidth}{!}{%
\centering
\begin{tabular}{crrrr}
\toprule
Name & \#Train & \#Val & \#Test & \#Individuals \\
\midrule
% F\_v1 & 109 & 44 & 176 & 67   \\
F\_v1 & 300 & 110 & 440 & 155  \\
F\_v2 & 500 & 209 & 837 & 233 \\
F\_v3 & 1001 & 432 & 1731 & 368 \\
F\_v4 & 3014 & 1226 & 4908 & 895 \\
%F\_v5 & 5002 & 1694 & 6776 & 1061 \\
% F\_v6 & 5002 & 1694 & 6776 & 1061 \\
\bottomrule
\end{tabular}
%}
\caption{Dataset properties: Number of train, validation, and testing assertions, and individuals.}
\label{tab:FamilyDatasetProperties}
\end{table}

\begin{table}[t]
\begin{tabular}{cccccc}
\toprule
Dataset & Model    & \multicolumn{1}{c}{MRR}  & \multicolumn{1}{c}{H@1}  & \multicolumn{1}{c}{H@3}  &
%\multicolumn{1}{c}{H@5}  &
\multicolumn{1}{c}{H@10} \\
\midrule
\multirow{7}{*}{F\_v1}   & BoxLitE1  & \multicolumn{1}{c}{.632}     & \multicolumn{1}{c}{.435}     & \multicolumn{1}{c}{.800}     & \multicolumn{1}{c}{.962}     \\
        & BoxLitE2 & \multicolumn{1}{c}{.249}     & \multicolumn{1}{c}{.134}     & \multicolumn{1}{c}{.284}     & \multicolumn{1}{c}{.475}     \\
        & BoxLitE3 & \multicolumn{1}{c}{.698}     & \multicolumn{1}{c}{.524}     & \multicolumn{1}{c}{.842}     & \multicolumn{1}{c}{.984}     \\
        %\hline
        & BoxLitE & \multicolumn{1}{c}{.720}     & \multicolumn{1}{c}{.545}     & \multicolumn{1}{c}{.870}     & \multicolumn{1}{c}{.979}     \\
       % \hline
        & BoxE     & \multicolumn{1}{c}{.826} & \multicolumn{1}{c}{.719} & \multicolumn{1}{c}{.918} & \multicolumn{1}{c}{.986} \\
        & RotatE   & \multicolumn{1}{c}{.474}     & \multicolumn{1}{c}{.295}     & \multicolumn{1}{c}{.584}     & \multicolumn{1}{c}{.805}     \\
        & ComplEx  & \multicolumn{1}{c}{.322} & \multicolumn{1}{c}{.206} & \multicolumn{1}{c}{.359} & \multicolumn{1}{c}{.535} \\
        \midrule
\multirow{7}{*}{F\_v2}   & BoxLitE1  &    \multicolumn{1}{c}{.428} & \multicolumn{1}{c}{.268} & \multicolumn{1}{c}{.519} & \multicolumn{1}{c}{.726}  \\
        & BoxLitE2 &      \multicolumn{1}{c}{.144} & \multicolumn{1}{c}{.072} & \multicolumn{1}{c}{.142} & \multicolumn{1}{c}{.273}  \\
        & BoxLitE3 &     \multicolumn{1}{c}{.541} & \multicolumn{1}{c}{.342} & \multicolumn{1}{c}{.668} & \multicolumn{1}{c}{.897}  \\
       % \hline
        & BoxLitE &    \multicolumn{1}{c}{.549} & \multicolumn{1}{c}{.352} & \multicolumn{1}{c}{.685} & \multicolumn{1}{c}{.905}  \\
       % \hline
        & BoxE     & \multicolumn{1}{c}{ .432} & \multicolumn{1}{c}{.337} & \multicolumn{1}{c}{.461} & \multicolumn{1}{c}{.604} \\
        & RotatE  & \multicolumn{1}{c}{.209} & \multicolumn{1}{c}{.147} & \multicolumn{1}{c}{.218} & \multicolumn{1}{c}{.314} \\
        & ComplEx  & \multicolumn{1}{c}{.341} & \multicolumn{1}{c}{.265} & \multicolumn{1}{c}{.357} & \multicolumn{1}{c}{.488} \\
                \midrule
\multirow{7}{*}{F\_v3}   & BoxLitE1  &  \multicolumn{1}{c}{.364}                        &    \multicolumn{1}{c}{.239}                      & \multicolumn{1}{c}{.396}                         &        \multicolumn{1}{c}{.626}                  \\  %0 0.3 0
        & BoxLitE2 &    \multicolumn{1}{c}{.116}                        &  \multicolumn{1}{c}{.055}                        &  \multicolumn{1}{c}{.108}                        &   \multicolumn{1}{c}{.224}                       \\ %0.1	0	0.003
        & BoxLitE3 &     \multicolumn{1}{c}{.414}                     &       \multicolumn{1}{c}{.280}                   &        \multicolumn{1}{c}{.472}                  &         \multicolumn{1}{c}{.681}                 \\ %0.003 0.1 0
        & BoxLitE &    \multicolumn{1}{c}{.433}                      &  \multicolumn{1}{c}{.288}                        & \multicolumn{1}{c}{.509}                         &  \multicolumn{1}{c}{.708}                        \\
        & BoxE     & \multicolumn{1}{c}{.894} & \multicolumn{1}{c}{.827} & \multicolumn{1}{c}{.946} & \multicolumn{1}{c}{.998} \\
        & RotatE   & \multicolumn{1}{c}{.177} & \multicolumn{1}{c}{.104} & \multicolumn{1}{c}{.199} & \multicolumn{1}{c}{.298} \\
        & ComplEx  & \multicolumn{1}{c}{.184}                     & \multicolumn{1}{c}{.119}                     & \multicolumn{1}{c}{.205}                     & \multicolumn{1}{c}{.284}                     \\
                \midrule
\multirow{7}{*}{F\_v4}   & BoxLitE1  &  \multicolumn{1}{c}{.339} &  \multicolumn{1}{c}{.204}  &  \multicolumn{1}{c}{.415}  &  \multicolumn{1}{c}{.592} \\
        & BoxLitE2 &   \multicolumn{1}{c}{.059}  &  \multicolumn{1}{c}{.023} &  \multicolumn{1}{c}{.54} &  \multicolumn{1}{c}{.110} \\
        & BoxLitE3 &   \multicolumn{1}{c}{.409} &  \multicolumn{1}{c}{.272}  & \multicolumn{1}{c}{.483}&  \multicolumn{1}{c}{.638} \\
        & BoxLitE &    \multicolumn{1}{c}{.444} & \multicolumn{1}{c}{.249} &  \multicolumn{1}{c}{.571}  & \multicolumn{1}{c}{.763}\\
        & BoxE     & \multicolumn{1}{c}{.626} & \multicolumn{1}{c}{.444} & \multicolumn{1}{c}{.772} & \multicolumn{1}{c}{.929} \\
        & RotatE  & \multicolumn{1}{c}{.245} & \multicolumn{1}{c}{.129} & \multicolumn{1}{c}{.300} & \multicolumn{1}{c}{.450} \\
        & ComplEx  & \multicolumn{1}{c}{.134} & \multicolumn{1}{c}{.087} & \multicolumn{1}{c}{.156} & \multicolumn{1}{c}{.203} \\
        \bottomrule
\end{tabular}
\caption{Test Results on F\_v1-4. Average of 3 runs for SGD methods. Standard deviation nearly $0$ in all cases.}
\label{tab:results}
\end{table}

%\textbf{(Q1) Representation Performance.} Table \ref{tab:ExpResults} reveals the results of our experiments on F\_v1-4 for embedding solutions under low-dimensional settings ($d=32$), as commonly done in the KB embedding literature \citep{MuRP,RotH,HConE,Rot2L}. The table reveals that the prediction performance on the train set is quite high, given the low embedding dimensionality. In particular, given an individual and a role, the best embedding solutions found by Gurobi predict in $80\%$ of the cases the correct individual within the first $10$ results (see $H@10$). Thus, although the embedding dimensionality is quite low, our model learns to represent the training data quite well. 

\textbf{(Q1)  Performance.} %Furthermore, 
In Table \ref{tab:results} we present the link prediction results on the test set for \modelName with the three regularization terms that appear in the objetive function (Section~\ref{sec:Scoring}). To study the effect of each of these terms, we also performed experiments in which we remove them. We denote by BoxLitE$i$ the version of \modelName without the regularization term multiplied by $\lambda_i$.
The removal of each of the regularization terms  decreases the overall performance of the model, with the removal of the term associated with $\lambda_2$ being the one that   most negatively impacts the results. %\todo{say something to try to explain this}  %, with few exceptions
%On the test set, the performance 
%on link prediction in more than $70\%$ of the cases
%within the first $10$ results. 

% $(b)$ in contrast to traditional benchmarks such as \citep{WN18RR,FB15k237,YAGO3-10}, our test set size is much larger than our training set size (Table ), leading to a harder evaluation setting than in traditional benchmarks.

%\todo{Q2 is Bruno's homework.}
\textbf{(Q2) Scalability.} The prediction performance on the test set, reported in Table \ref{tab:results}, reduces slowly with increasing dataset sizes. 
%In particular, $H@10$ decreases by an absolute difference of less than $1\%$ on the train set and less than $4\%$ on the test set between F\_v1 and the over ten times larger F\_v4 dataset. 
%Thus, these results provide a first empirical evidence for the scalability of \modelName's prediction performance. 
Regarding the time required for each instance, we recall that 
a problem modelled through CVXPY is first compiled and then sent to a solver of the user's choice, which in our case is MOSEK. 
Given a specific choice of hyperparameters $\lambda_1,\lambda_2,\lambda_3$, Table~\ref{tab:timeRes} displays the compilation and solving time required for obtaining an optimal solution to the problem in  Theorem~\ref{theo:convex}, for each of the datasets F\_v1-4. In our implementation, we tested 354 hyperparameter configurations for each dataset. While changing the hyperparameters requires solving the optimization problem again, it does   not require a recompilation. Overall, compilation does not take more than a couple of seconds and all solution times were less than $30$ seconds, which is quite reasonable considering that the final SOCP corresponding to $F\_v4$ has around $3.5$ million variables and $2.3$ million constraints, which has size comparable to some of the instances that appear in  Mittelmann's benchmark of large SOCPs  \citep{Mittelmann2026}. 
%\color{red}
 Even more, the compilation time in Table~\ref{tab:timeRes} grows linearly with the number of training axioms, while the solving time increases only sublinearly. %Extrapolating from these trends, a large real‑world KB such as the Gene Ontology \citep{Ashburner2000GeneOT} with roughly $100k$ axioms, i.e., containing about $35$ times more training axioms than $F\_v4$, should likely still be compiled and solved into a \modelName embedding within roughly $20$~minutes.\todo{B: To be completely honest, although it was fine for an answer to the referee, I think it might be better to not include speculation on the performance on other data sets. My suggestion would be to just keep the 1st sentence of the red block.}
% \todo[inline]{Ana: I agree to just keep the first sentence and remove "Extrapolating..."}
\color{black}
%\footnote{Mittelmann has, throughout the years, provided impartial benchmarks for solvers of different classes of optimization problems.}. 
%\todo{talk about training time}
% Here, we analyze how the compilation time in CVXPY and the solving time in MOSEK of \modelName are affected by changing dataset sizes.
%\todo{update fv4}
\begin{table}[h!]
\centering
\begin{tabular}{ccc}
\toprule
Dataset & Compilation Time & Solving Time \\
\midrule
F\_v1        &  0.34   &  11.91   \\
F\_v2        &  0.54   &  13.14   \\
F\_v3        &  1.05   &   16.78   \\
F\_v4        &  3.12   &   26.24  \\
%F\_v5        &  8h   &   815s  \\
\bottomrule
\end{tabular}
\caption{Time in seconds  split by dataset for \modelName.}
\label{tab:timeRes}
\end{table}

\textbf{(Q3) Comparison.} 
%\todo{explain that we are not giving the tbox to these models}
%\todo{in table 2 we present average results for 3 runs. in all cases, the standard dev was near to zero.   }
Comparing \modelName with other KBEs in a direct way is tricky since KBEs in the literature consider other languages and are mostly optimized using SGD.
%stochastic gradient descent (SGD). 
\modelName is the first KBE for
DL-Lite$^\Hmc$ ontologies and the first that solves link prediction via   convex optimization. In the discussion, we include an argument for why it is challenging to design convex optimization approaches for languages with conjunctions, which appear in other papers. 
%
%and the languages that KBEs can represent differ greatly. This  fair comparison hard, as \modelName is the first KBE designed to embed  DL-Lite$^\Hmc$ ontologies. 
To illustrate how our approach roughly compares with
%We include a comparison of  \modelName with the 
classical  embedding methods such as BoxE, RotatE, and ComplEx, based on SGD, we run those methods on {the ABox part of} F\_v1-4. We see that the results of 
\modelName are  better than RotatE and ComplEx, but still behind BoxE. One exception is 
F\_v2 where our method performs better. 
{
None of the SGD methods had any rule injections to reflect the DL-Lite$^\Hmc$ TBox axioms used in   BoxLitE. This is a disadvantage for the SGD methods. On the other hand, %BoxE outperformed 
BoxLitE has negative sampling  applied only to existential assertions, which is a disadvantage for our case.}
%in its current state for several reasons, including that negative sampling was applied only to existential assertions in our case.
%None of the SGD models had any rule injections to reflect the TBox axioms used in the BoxLitE. This is a disadvantage for the SGD models. Nonetheless, BoxE outperformed BoxLitE in its current state for several reasons, including that negative sampling was applied only to existential assertions.
% {\color{red} Our assessment is that an area in which our current implementation is lacking is the hyperparameter optimization part, which is currently done by brute force search over pre-defined range of parameters. We believe this is handled more efficiently in pykeen for the compared approaches. }
{
A possible main reason for the performance gap is \modelName's hyperparameter optimization (HPO). It currently relies on grid search, which ensures full control over the explored parameter space, vital for ablations. Yet, grid search does not adaptively guide the hyperparameter search. By contrast, the SGD approaches, implemented in PyKEEN \citep{ali2021pykeen}, make use of % the Optuna framework, which implements 
more efficient, model‑based optimization techniques. Incorporating such adaptive HPO methods in future work could lead to more effective exploration of \modelName's hyperparameter landscape and %potential  %it is already saying "could "before
performance gains.
% Old: % Our assessment is that an area in which our current implementation is lacking is the hyperparameter optimization part, which is currently done by grid search. We believe this is handled more efficiently in PyKEEN for the compared approaches.
%We believe that the hyperparameter optimization procedure of the SGD methods is more efficient than ours and improving our search would increase the performance of \modelName.  
}
\section{Discussion on Optimization in KBEs  }\label{sec:discussion}

In this section, 
we discuss  some aspects related to 
differentiability and %convexity in the literature. 
%sheds light on the 
primary causes for nonconvexity in previous KB embedding works. {We also recall how we address 
each challenge in our approach.}

%\todo{explain why nondiff ok}

\textbf{Nondifferentiability.}
In our approach nondifferentiability is not an issue because in view of Theorem~\ref{theo:socp} the underlying optimization problem can be cast as a second-order cone program. Informally, 
the nondifferentiable part of the problem 
gets embedded into the conic constraints and our solver of choice (MOSEK) \citep{mosek} can handle 
this kind of problem without theoretical 
issues. %\todo{Improve this part}

%\subsection{Nonconvexity in previous works}\label{sec:nonconvex}
\textbf{Negative sampling.} 
{Negative} sampling as described in \cite[Section~3.3]{RotatE} minimizes terms of the form:
\begin{equation*}%\label{eq:negsample}
-\log(\sigma(\gamma - d(x))) - \sum _{i=1}^n w_i(\log(\sigma(d(x_i)-\gamma))),
\end{equation*}
where $\sigma$ is a sigmoid function (e.g., $1/(1+e^{-x})$), $w_i$ are nonnegative weights, $\gamma$ is a margin parameter, the $x_i$'s are ``negative samples'' and $d$ is a distance-like function which may include a $p$-norm term. 
Generally speaking, a function of the form $-\log(\sigma(d(z)-\gamma )$ is neither convex (nor concave) nor differentiable everywhere as a function of $z$.
This can be seen by considering the 1-dimensional case, where $d$ is the absolute value function and $z$ is scalar, so that we obtain the function $-\ln(\sigma(|z|-\gamma)) = \ln(1+e^{\gamma-|z|})$, which, albeit continuous, is nonconvex and nondifferentiable at $z = 0$. 
%From an optimization point of view, { besides being a source of nonconvexity}, applying ADAM to an objective function containing terms of this form is theoretically unsound.
{ In contrast, \modelName adopts negative sampling in a convex way, using the negative concept regularization terms (Section~\ref{sec:Scoring}) that push individuals into complement boxes.}
%{\color{red} Even the first half of \eqref{eq:negsample} is troubling because the function $-\log(\sigma(\gamma-d(z) )$ is not typically differentiable at $d(z)=0$ which can be seen again by looking at the 1D case which corresponds to the  function $\ln(1+e^{|z|-\gamma})$. Since the goal is typically to drive the loss to zero (which is contained in the region of nondifferentiability),  a pure gradient method may struggle to accomplish that.}
%\todo[inline]{B: I put the part above in red because we may want to comment it out later.}

\textbf{Nonconvex loss terms.} 
%Even without negative sampling, 
The loss terms of previous works often contain operations that do not preserve convexity in general. We briefly take a look at this issue in two closely related works.
To be fair, none of these two works claim that their optimization problems are convex. 
%\begin{itemize}
 %   \item 
%    \textbf{BoxE.} 
    For BoxE, it is not clear whether the distance function considered in \cite[Section~4]{BoxE} is convex as a \emph{function of the parameters that need to be optimized during learning}, since it includes {division} and multiplication by a width term that represents box sizes. Division and multiplication are in general not operations that preserve convexity. 
    %\item 
    %\textbf{BoxEL.} 
    In BoxEL, the authors consider loss terms that are quotients of volumes of boxes (or approximations thereof), e.g., see \cite[Section~4.4]{BoxEL}. Again, division does not preserve convexity in general.
%\todo{Check reference to BoxEL}
{
In contrast, all of \modelName's loss terms together with the distance function in our approach are convex.
}
   % {\color{red}That said, the concept and role assertion loss terms in BoxEL (\cite[Section~4.3]{BoxEL}) are indeed convex, but once, say, a concept assertion is satisfied the loss value is always zero, thus points inside the box always have the same loss, no matter their position. In contrast,  assertion loss terms in \citep{BoxE} are not constant inside  boxes, which is desirable. However, BoxE makes use of a distance function which, as mentioned previously, is unclear whether it is convex. In our approach, because of our usage of the signed distance function, concept and role assertion  are {both} convex and nonconstant in the interior of the box. It also encourage points to be closer to the center of the box. }
 %  A typical difficulty is that for, say, a loss term that model a concept assertion, it is hard to devise a convex function that The design of an appropriate distance function is nontrivial. While it is simple to design a convex loss term that is zero if and only if, say, a concept is included in another
%    \todo{talk about sign distance in other works and how we solved in ours}
%\end{itemize}

\textbf{The logic fragment.} 
A final source of nonconvexity %for certain approaches 
seems to be the choice of {description logic fragments}. Approaches based on minimizing loss terms together with logical languages that include, say, conjunction on the left-hand side typically lead to nonconvexity. 
Suppose that the conjunction of concepts $C,D$ is interpreted as the set intersection of the corresponding boxes $\eta(C), \eta(D)$, {as observed in the works for the \EL ontology language}.
%{\color{red}For example,  consider an \EL axiom of the form ``$C \sqcap D \sqsubseteq E$'' which is allowed, e.g., in {Box2EL \citep{Box2EL}} and  TransBox:\EL$^{++}$\citep{YCS25}}. %all el not just these two
In this case, the axiom $C \sqcap D \sqsubseteq E$ translates to the constraints
 {$\mathbf{L}_E \leq \max(\mathbf{L}_C,\mathbf{L}_D)$} 
% $\max(\mathbf{L}_C[i],\mathbf{L}_D[i]) \leq \mathbf{L}_E[i]$
%\todo{Aleks: I beieve this should be $\mathbf{L}_E[i] \leq \max(\mathbf{L}_C[i],\mathbf{L}_D[i])$ instead.} 
and  $\min(\mathbf{U}_C,\mathbf{U}_D) \leq \mathbf{U}_E$, where $\mathbf{L}_X, \mathbf{U}_X$ indicate the lower and upper bounds of the box associated to a concept $X$.
Because the set $\Smc \coloneqq \{(a,b,c) \in \Re^3 \mid \min(a,b) \leq c \}$ is not convex\footnote{It suffices to observe that $(1,0,0), (0,1,0) \in \Smc  $, but $0.5(1,0,0) + 0.5(0,1,0) = (0.5,0.5,0)\not \in \Smc$.}, these constraints are not convex in general. 
Since the optimal sets of convex functions are convex, it is not possible to devise a convex $f:\Re^3\to \Re$ such that ``$(a,b,c) \in \Smc \Leftrightarrow (a,b,c)$ is optimal for $f$''.
In particular, absent extenuating circumstances, if the bounds of $C,D,E$ are parameters to be optimized during learning, it is impossible to construct a nonnegative \emph{convex} loss function that is zero if and only if $C \sqcap D \sqsubseteq E$ holds.
%We also comment briefly on 
Regarding \emph{cone semantics}, although convex optimization is mentioned as one motivation for using cones \citep{OLW20}, it is not explained how exactly convex optimization fits in the picture of their approach.
In another work, the authors show how to use axis-aligned cones and pairs of unions of convex cones to solve certain multi-label classification problems \citep{LOW22} via SVMs. 
The approach described in \citep{LOW22} is significantly different from the loss function minimization approach described in other papers. Moreover, the Propositional $\mathcal{ALC}$ language used in \citep{LOW22} is   different from DL-Lite$^{\Hmc}$, which we consider in this work, since DL-Lite$^{\Hmc}$ features roles and inverses. %, and role inclusions. 
\section{Conclusion and Future Work}
\label{sec:conclusion}

%\textbf{Concluion.} In this work, % commenting repetition
We propose \modelName, a KB embedding method that allows for convex optimization and ensures the satisfaction of TBox axioms.
%by translating them to convex constraints. %In addition, 
We prove that for any DL-Lite$^\Hmc$ KB, there is a   faithful embedding solution that is  a KB model.
%, and  for any DL-Lite$_{core}$ KB, an embedding solution exists, whose loss is equal to zero, ensuring it to be strongly TBox faithful. 
We %practically 
implement and evaluate \modelName's convex problem formulation for KB embeddings in CVXPY and MOSEK. 
The results reveal that MOSEK   finds a solution that is KB faithful and that satisfies the concept inclusions in the  TBox   for a prototypical ontology. %\todo{I think this part about strong faithfulness needs to be removed.} 
Within   a few seconds of solving time, MOSEK   finds embedding solutions on subsets (F\_v1 to F\_v4) of a real-world KB that achieve good link prediction results. %for KB representation and reasoning tasks.
{In the future, we would like to %investigate more \modelName's practical side and  
{consider more efficient methods for the HPO and evaluation.} 
Also, we want to study how to ensure other theoretical properties in convex KBE methods. 
As languages that allow for conjunctions on the left of inclusions often %are likely to 
lead to nonconvexity, 
%\todo{add this in the main paper}
another %interesting 
line lies in studying sound nonconvex KBE approaches that can ensure the construction of faithful models. 
Finally, we would like to investigate, using nonconvex tools, the effect of pushing negative samples outside the concept box into an `unknown' truth state, and how this affects prediction results.
%compared to 
%the effect observed in 
%our convex setting.

\section*{Acknowledgements}
This work was supported by the ``Strategic Research Projects'' grant from ROIS
(Research Organization of Information and Systems). Bruno F. Louren\c{c}o's work was partially supported by the JSPS Grant-in-Aid for Early-Career Scientists 23K16844. Ana Ozaki was supported by the Research Council of Norway, projects (316022, 322480) and Integreat - Norwegian Centre for knowledge-driven machine learning (332645).
Emanuel Sallinger's and Hesham Morgan's work on this paper was funded
by the Vienna
Science and Technology Fund (WWTF) [Grant ID: 10.47379/VRG18013, 10.47379/ICT25032, 10.47379/NXT22018, 10.47379/ ICT2201, 10.47379/DCDH001], and by the
Austrian Science Fund (FWF) 10.55776/COE12.
\section*{AI Declaration}
Gen AI tools were only used to help find grammatical mistakes and to aid in the  process of debugging  the code.
\bibliographystyle{kr}
\bibliography{kr}

\ifFullVersion
\appendix
%\newpage
%\section{Organization}
\section{{Supplemental Material}}
%\todo[inline]{B: I updated the  intro below, but we we should check whether there are things that I missed.}
This supplemental material contains additional information on the experimental setup,   theoretical and empirical results, and complete proofs for each corollary, proposition, and theorem. 
Specifically,~\cref{sec:appendixbasic} introduces the semantics of DL-Lite$^\Hmc$ for the convenience of the reader. Next,~\cref{app:basicDef} provides all proofs for the theoretical results in~\cref{sec:basicDef}. Afterwards,~\cref{app:faifthul}  provides proofs for the theoretical results of~\cref{sec:faithful}. Moreover,~\cref{app:opt} formulates in detail \modelName's convex optimization problem and provides the proofs for each theoretical result of  
\cref{sec:opt}. Then, ~\cref{app:additional_analyses} provides additional information on the sizes of the optimization problem in the experiments. %empirical analyses and results, for instance, measuring the compilation and solving time of \modelName on datasets of varying sizes. 
%Finally,
~\cref{app:expDetails} provides additional information on the experimental setup, including implementation details and a   discussion on the creation of the datasets (F\_v1-4), the training setup, hyperparameter optimization, evaluation protocol, and used metrics. 
%Next, ~\cref{sec:nonconvex} sheds light on the primary causes for nonconvexity in previous KB embedding works. 
%Finally,  
%\cref{app:protoDeductiveClosure} lists the deductive closure of the prototypical family KB (Figure \ref{fig:familyOnto} in  \cref{sec:Exp_Strong_Faithfulness}) that was used for evaluating whether the retrieved embedding solution of  \cref{sec:Exp_Strong_Faithfulness} is strongly TBox faithful.

\section{Basic Definitions: DL-Lite$^\Hmc$ Semantics}\label{sec:appendixbasic}
For the convenience of the reader, here we provide the definition
of the semantics for  DL-Lite$^\Hmc$,
which is standard and can be found in  references such as~\citep{dllite-jair09}. 

An \emph{interpretation} $\Imc$ is a pair $(\Delta^\Imc,\cdot^\Imc)$, where $\Delta^\Imc$
is a non-empty set, called the \emph{domain} of \Imc, and  {$\cdot^\Imc$} is a function that assigns a subset  $A^\Imc\subseteq \Delta^\Imc$ of the domain to each
$A\in\NC$, a binary relation  $R^\Imc\subseteq \Delta^\Imc\times\Delta^\Imc$ over the domain to each $R\in\NR$, and an element $a^\Imc \in \Delta^\Imc$  to each $a\in\NI$.
We extend $\cdot^\Imc$ to   role and concept expressions as follows:
\begin{align*} 
%	(\top)^\Imc := {} & \Delta^\Imc; \\ 
   % (C\sqcap D)^\Imc := {} &C^\Imc\cap D^\Imc;\\
%     (C\sqcup D)^\Imc = {} &C^\Imc\cup D^\Imc;\\
        (\neg B)^\Imc := {} &  \Delta^\Imc\setminus B^\Imc; \\ 
       %  (\neg S)^\Imc := {} &  (\Delta^\Imc\times \Delta^\Imc)\setminus S^\Imc; \\ 
	(R^-)^\Imc := {} & \{(e,d)\mid (d,e)\in R^\Imc\};\\	
	    (\exists S)^\Imc := {} & \{d\mid \exists e\in\Delta^\Imc
    \text{  such that  }(d,e)\in S^\Imc\}.
%  ({\sf ran} (R))^\Imc = {} & \{e\mid \exists d\in\Delta^\Imc    \text{ s.t. }(d,e)\in R^\Imc\};   \\ 
  \end{align*} 
 % \todo[inline]{Aleks: Why do we introduce the expression $\neg S$ above? I thought this was not a part of DL-Lite$^\Hmc$?} thanks yes we forgot to remove
We say that an interpretation \Imc \emph{satisfies}
\begin{itemize}
    \item a role inclusion $S\sqsubseteq T$ iff
    $S^\Imc\subseteq T^\Imc$;
    \item a concept inclusion $B\sqsubseteq C$ iff
    $B^\Imc\subseteq C^\Imc$; 
    \item a role assertion $R(a,b)$ iff $(a^\Imc,b^\Imc)\in R^\Imc$; and 
\item a concept assertion $D(a)$ iff $a^\Imc\in D^\Imc$.
\end{itemize}
%~\footnote{
In this work $\NC$, $\NR$, and $\NI$ are all \emph{finite sets}, considered to be the relevant symbols to express KBs. While in the   Description Logic literature these sets are often assumed to be countably infinite, in the KB embedding literature, they are usually assumed to be finite, e.g., \citep{BoxEL} and that is what we adopt here.
%}
{In what follows,  we denote the size of a finite set $V$  by $|V|$.}

\section{Proofs for Section~\ref{sec:basicDef}}
\label{app:basicDef}
Here, we provide proofs for the results in \cref{sec:basicDef}.
The canonical interpretations found in the literature,  e.g.,~\citep{DBLP:conf/kr/KontchakovLTWZ10} are usually designed for query answering and may not satisfy the (concept/role) inclusions that are entailed by the TBox or falsify those inclusions that are not entailed. Since satisfying the TBox is important in our work for   establishing    faithfulness results later, we provided our own definition of {the} canonical model (\cref{def:canonicalModel}) and now  provide the full proof of \cref{thm:canonicalModel}.

\smallskip

\theoremcanonical*
\begin{proof}
%The first point follows from the fact that
%the definition of the canonical model
%satisfies the ABox and the proof of the second point, which we focus now. \todo{Aleks: I am a little unsure what the previous sentence says.} this was old, the previous statement had diff points
In the following, assume
$A, B\in\NC$ and $R,S\in\NR$.
%\todo[inline]{B: Shouldn't we explicitly mention that   $A,B \in \NC$ (or is it $\NCE$?) and $R,S \in \NR$ (or is it $\NRm$)?}
\begin{claim}
 $\Imc_\Kmc\models A\sqsubseteq B$   iff $\Kmc\models A\sqsubseteq B$.
\end{claim}
\begin{proof}
    \textbf{Assume $\Kmc\models A\sqsubseteq B$.} We make a case distinction based 
on the elements in $\Delta^{\Imc_\Kmc}:=\NI \cup \Delta_\Kmc$. 
\begin{itemize}
\item $a\in \NI$: Assume $a\in A^{\Imc_\Kmc}$. By definition of $\Imc_\Kmc$, we have that
$a\in A^{\Imc_\Kmc}$ iff $\Kmc\models A(a)$.
By assumption, $\Kmc\models A\sqsubseteq B$, so $\Kmc\models B(a)$. Then, again by definition of $\Imc_\Kmc$,
$a\in B^{\Imc_\Kmc}$. Since $a$ was an arbitrary element of \NI this holds for all elements of this kind.
\item $c_D\in\Delta_\Kmc$: Assume $c_D\in A^{\Imc_\Kmc}$. By definition of $\Imc_\Kmc$, we have that
$c_D\in A^{\Imc_\Kmc}$ iff $\Kmc\models D\sqsubseteq A$.
By assumption, $\Kmc\models A\sqsubseteq B$, so $\Kmc\models D\sqsubseteq B$. Then, again by definition of $\Imc_\Kmc$,
$c_D\in B^{\Imc_\Kmc}$. Since $c_D$ was an arbitrary element of $\Delta_\Kmc$ this holds for all elements of this kind.
\end{itemize}
We have thus shown that $\Imc_\Kmc\models A\sqsubseteq B$.

\smallskip

\noindent
\textbf{Now, assume $\Kmc\not\models A\sqsubseteq B$.}
We show that $\canonmodel\not\models A\sqsubseteq B$. If $\Kmc\not\models A\sqsubseteq B$
then there is an interpretation \Imc that satisfies \Kmc with
$A^\Imc$ non-empty. This means that
$A$ is satisfiable w.r.t. \Kmc and thus
$c_{A}\in \Delta_\Kmc$.
By definition of $\canonmodel$, we have that
$c_{A}\in A^{\canonmodel}$  since $\Kmc\models A\sqsubseteq A$ holds trivially. 
We now argue that
$c_A\notin B^{\canonmodel}$. By definition of $\canonmodel$, 
an element of the form 
$c_D$ is in 
$B^{\canonmodel}$
iff
 $\Kmc\models D\sqsubseteq B$.
 By assumption $\Kmc\not\models A\sqsubseteq B$.
 So   $c_A$ is not in $B^{\canonmodel}$.
\end{proof}

\begin{claim}\label{clm:anegb}
 $\Imc_\Kmc\models A\sqsubseteq \neg B$   iff $\Kmc\models A\sqsubseteq \neg B$.
\end{claim}
\begin{proof}
    \textbf{Assume $\Kmc\models A\sqsubseteq \neg B$.} We make a case distinction based 
on the elements in $\Delta^{\Imc_\Kmc}:=\NI \cup \Delta_\Kmc$. 
\begin{itemize}
\item $a\in \NI$: Assume $a\in A^{\Imc_\Kmc}$. By definition of $\Imc_\Kmc$, we have that
$a\in A^{\Imc_\Kmc}$ iff $\Kmc\models A(a)$.
By assumption, $\Kmc\models A\sqsubseteq \neg B$.
Also, %so $\Kmc\models \neg B(a)$. 
by assumption \Kmc is satisfiable, meaning that $\Kmc\not\models  B(a)$. Then,   by definition of $\Imc_\Kmc$,
$a\notin B^{\Imc_\Kmc}$, that is, $a\in (\neg B)^{\Imc_\Kmc}$. As $a$ was an arbitrary element of \NI this holds for all elements of this kind.
\item $c_D\in\Delta_\Kmc$: Assume $c_D\in A^{\Imc_\Kmc}$. By definition of $\Imc_\Kmc$, we have that
$c_D\in A^{\Imc_\Kmc}$ iff $\Kmc\models D\sqsubseteq A$.
By assumption, $\Kmc\models A\sqsubseteq \neg B$, so $\Kmc\models D\sqsubseteq \neg B$. As $c_D\in\Delta_\Kmc$, by definition of $\Delta_\Kmc$, $D$ is satisfiable w.r.t \Kmc. So $\Kmc\not\models D\sqsubseteq  B$.
Then, again by definition of $\Imc_\Kmc$,
$c_D\in (\neg B)^{\Imc_\Kmc}$. Since $c_D$ was an arbitrary element of $\Delta_\Kmc$ this holds for all elements of this kind.
\end{itemize}
We have thus shown that $\Imc_\Kmc\models A\sqsubseteq \neg B$.

\smallskip

\noindent
\textbf{Now, assume $\Kmc\not\models A\sqsubseteq \neg B$.}
We show that $\canonmodel\not\models A\sqsubseteq \neg B$. If $\Kmc\not\models A\sqsubseteq \neg B$
then there is an interpretation \Imc that satisfies \Kmc with
$A^\Imc\cap B^\Imc$ non-empty. This means that
$A\sqcap B$ is satisfiable w.r.t. \Kmc and thus
$c_{A\sqcap B}\in \Delta_\Kmc$.
By definition of $\canonmodel$, we have that
$c_{A\sqcap B}\in A^{\canonmodel}$ and $c_{A\sqcap B}\in B^{\canonmodel}$  since $\Kmc\models A\sqcap B\sqsubseteq A$ and $\Kmc\models A\sqcap B\sqsubseteq B$. % holds trivially. 
%We now argue that
%$c_A\notin B^{\canonmodel}$. By definition of $\canonmodel$, 
%an element of the form 
%$c_D$ is in 
%$B^{\canonmodel}$
%iff
% $\Kmc\models D\sqsubseteq B$.
% By assumption $\Kmc\not\models A\sqsubseteq B$.
 So  $\canonmodel\not\models A\sqsubseteq \neg B$.
 %$c_A$ is not in $B^{\canonmodel}$.
\end{proof}
\begin{claim}
% Given $A,R\in{\sf sig}(\Kmc)$,
 $\Imc_{\Kmc}\models R\sqsubseteq S$   iff $\Kmc\models  R\sqsubseteq S$.
\end{claim}
\begin{proof}

\textbf{Assume $\Kmc\models R\sqsubseteq S$.} We make a case distinction based 
on the elements in $\Delta^{\canonmodel}$ and how they
can be related in the extension of a role name in the definition of $\canonmodel$.
\begin{itemize}
   
\item $(a,b)\in \NI\times\NI$:  
  Assume 
    $(a,b)\in R^{\canonmodel}$.
    We first argue that in this case
    $\Kmc\models R(a,b)$. 
    By definition of $\canonmodel$,
$     (a, b) \in  R^{\canonmodel} \text{ iff } \Kmc \models R(a,b)$.  
Since by assumption $\Kmc\models R\sqsubseteq S$
we have that $\Kmc \models S(a,b)$, so 
$(a,b)\in S^{\canonmodel}$. Since
$(a,b)$ was an arbitrary pair in $\NI\times\NI$, this holds for all such kinds of pairs. 
\item $(a,c_{\exists  {R'}})\in \NI\times \Delta_\Kmc$: Assume 
    $(a,c_{\exists  {R'}})\in R^{\canonmodel}$. 
    %We first argue that in this case
    %$\Kmc\models \exists R(a)$. 
    By definition of $\canonmodel$,
    we have that 
    %$(a,c_{\exists \overline{R'}})\in R^{\canonmodel}$ iff
    $\Kmc\models \exists \overline{R'}(a)$ and $\Kmc\models \overline{R'}\sqsubseteq R$.
    By assumption $\Kmc\models R\sqsubseteq S$.
    So $\Kmc\models \overline{R'}\sqsubseteq S$. Then, again by definition of $\canonmodel$,
    we have that $(a,c_{\exists {R'}})\in S^{\canonmodel}$. Since
$(a,c_{\exists {R'}})$ was an arbitrary pair of this format in $\NI\times \Delta_\Kmc$, this holds 
for all such kinds of pairs. 
\item $(c_{\exists R'},a)\in \Delta_\Kmc\times \NI$: 
Assume 
    $(c_{\exists R'},a)\in R^{\canonmodel}$. 
    %We first argue that in this case
    %$\Kmc\models \exists \overline{R}(a)$. 
    By definition of $\canonmodel$,
   % we have that $(c_{\exists R'},a)\in R^{\canonmodel}$ iff
    $\Kmc\models \exists \overline{R'}(a)$ and $\Kmc\models {R'}\sqsubseteq R$.
    By assumption $\Kmc\models R\sqsubseteq S$.
    So $\Kmc\models R'\sqsubseteq S$. Then, again by definition of $\canonmodel$,
    we have that $(c_{\exists R'},a)\in S^{\canonmodel}$. Since
$(c_{\exists {R'}},a)$ was an arbitrary pair of this format in $\NI\times \Delta_\Kmc$, this argument can be applied 
for all such kinds of pairs. 
%this case is similar to the above case, except for changes based on inverse roles.
%that here 
%we have $R'^-$ instead of $R'$.
\item $(c_{\exists R'},c_{\exists \overline{R'}})\in \Delta_\Kmc \times \Delta_\Kmc$: 
Assume 
    $(c_{\exists R'},c_{\exists \overline{R'}})\in R^{\canonmodel}$.
By definition of $\canonmodel$,
we have that $\Kmc\models R'\sqsubseteq R$. By assumption, $\Kmc\models R\sqsubseteq S$, so
$\Kmc\models R'\sqsubseteq S$.
Then, again by definition of $\canonmodel$,
 $(c_{\exists R'},c_{\exists \overline{R'}})\in S^{\canonmodel}$.
%\item $(c_{\exists R'^-},c_{\exists R'})\in \Delta_\Kmc \times \Delta_\Kmc$: this case is similar to the above case, except for changes based on inverse roles.
\item $(c_D,c_{\exists {R'}})\in \Delta_\Kmc \times \Delta_\Kmc$: Assume 
    $(c_{D},c_{\exists {R'}})\in R^{\canonmodel}$.
By definition of $\canonmodel$,
we have that $\Kmc\models D\sqsubseteq \exists \overline{R'}$ and $\Kmc\models \overline{R'}\sqsubseteq R$. By assumption, $\Kmc\models R\sqsubseteq S$, so
$\Kmc\models \overline{R'}\sqsubseteq S$.
Then, again by definition of $\canonmodel$,
 $(c_{D},c_{\exists {R'}})\in S^{\canonmodel}$. 
\item $(c_{\exists R'},c_{D})\in \Delta_\Kmc \times \Delta_\Kmc$: 
 Assume 
    $(c_{\exists {R'}},c_{D})\in R^{\canonmodel}$.
By definition of $\canonmodel$,
we have that $\Kmc\models D\sqsubseteq \exists \overline{R'}$ and $\Kmc\models R'\sqsubseteq R$. By assumption, $\Kmc\models R\sqsubseteq S$, so
$\Kmc\models R'\sqsubseteq S$.
Then, again by definition of $\canonmodel$,
 $(c_{\exists {R'}},c_{D})\in S^{\canonmodel}$. 
 \end{itemize}
We have thus shown that
$\canonmodel\models R\sqsubseteq S$.

\smallskip

\noindent
\textbf{Now, assume $\Kmc\not\models R\sqsubseteq S$.}
We show that $\canonmodel\not\models R\sqsubseteq S$.
By definition of $\canonmodel$, we have that
$\{(c_{\exists S},c_{\exists \overline{S}})\in \Delta_\Kmc \times \Delta_\Kmc\mid \Kmc\models S\sqsubseteq R\}\subseteq R^{\canonmodel}$.
By taking   $S=R$ (and since trivially $\Kmc\models R\sqsubseteq R$), 
we have in particular that
$(c_{\exists R},c_{\exists \overline{R}})\in R^{\canonmodel}$.
We now argue that
$(c_{\exists R},c_{\exists \overline{R}})\notin S^{\canonmodel}$. By definition of $\canonmodel$, 
a pair of the form 
$(c_{\exists S'},c_{\exists \overline{S'}})$ is in 
$S^{\canonmodel}$
iff
 $\Kmc\models S'\sqsubseteq S$.
 By assumption $\Kmc\not\models R\sqsubseteq S$.
 So   $(c_{\exists R},c_{\exists \overline{R}})\notin S^{\canonmodel}$.
\end{proof}

\begin{claim}\label{clm:exists}
 $\Imc_\Kmc\models \exists R\sqsubseteq A$   iff $\Kmc\models \exists R\sqsubseteq A$.
\end{claim}
\begin{proof}
    \textbf{Assume $\Kmc\models \exists R\sqsubseteq A$.} We make a case distinction based 
on the elements in $\Delta^{\Imc_\Kmc}:=\NI \cup \Delta_\Kmc$. 
\begin{itemize}
\item $a\in \NI$: Assume $a\in (\exists R)^{\Imc_\Kmc}$. In this case, by definition of $\Imc_\Kmc$, either (1)
there is $b\in \NI$ such that
$(a,b)\in R^{\Imc_\Kmc}$  
or (2) there is $c_{\exists \overline{R}}\in  \Delta_\Kmc$ such that
$(a,c_{\exists \overline{R}})\in R^{\canonmodel}$.
In case~(1), by definition of $\canonmodel$, $(a,b)\in R^{\canonmodel}$ implies that
$\Kmc\models R(a,b)$.  Together with the assumption that $\Kmc\models \exists R\sqsubseteq A$, this means that
$\Kmc\models A(a)$. Again by definition of $\canonmodel$,
we have that $a\in A^{\canonmodel}$.
In case~(2), by definition of $\canonmodel$, $(a,c_{\exists \overline{R}})\in R^{\canonmodel}$ implies that
%$D=\exists R^-$ and 
$\Kmc\models \exists R(a)$.
By assumption  $\Kmc\models \exists R\sqsubseteq A$, which means that
$\Kmc\models A(a)$. Again by definition of $\canonmodel$,
we have that $a\in A^{\canonmodel}$. Since $a$ was an arbitrary
element in $\NI$, this argument can be applied  for all  elements of this kind.
\item $c_D\in \Delta_\Kmc$: Assume $c_D\in (\exists R)^{\canonmodel}$. We first show that  $\Kmc\models D\sqsubseteq \exists R$. If $c_D\in (\exists R)^{\canonmodel}$ then, 
by definition of $\canonmodel$, either~(1) 
there is $a\in \NI$
 such that
$(c_D,a)\in R^{\canonmodel}$
  or (2)  
  there is $c_{D'}\in \Delta_\Kmc$
 such that
$(c_D,c_{D'})\in R^{\canonmodel}$.
 In case (1), 
by definition of $\canonmodel$, 
$(c_D,a)\in R^{\canonmodel}$ means that $D$ is of the form $\exists S$, $\Kmc\models S\sqsubseteq R$, and $\Kmc\models\exists \overline{S}(a)$. 
So $\Kmc \models D\sqsubseteq \exists R$. 
In case (2), by definition of $\canonmodel$, there are three possibilities: 
\begin{itemize}
    \item[(a)] $D$ is of the form
$\exists S$ and $D'$ is of the form
$\exists \overline{S}$;
%\item (b) $D$ is of the form
%$\exists \overline{S}$ and $D'$ is of the form
%$\exists S$
\item[(b)]   $D'$ is of the form
$\exists S$ and $\Kmc\models D\sqsubseteq \exists \overline{S}$;
\item[(c)] $D$ is of the form
$\exists S$ and $\Kmc\models D'\sqsubseteq \exists \overline{S}$.
\end{itemize}
In the sub-cases (a) and (c) we also have that
$\Kmc\models S\sqsubseteq R$.
So $\Kmc \models D\sqsubseteq \exists R$.
In the sub-case (b), we have that
$\Kmc\models \overline{S}\sqsubseteq R$.
Then again $\Kmc \models D\sqsubseteq \exists R$.
By assumption  $\Kmc\models \exists R\sqsubseteq A$, which means that
$\Kmc \models D\sqsubseteq A$.
Then, $c_D\in A^{\Imc_\Kmc}$, by definition of $\canonmodel$. 
%By assumption  $\Kmc\models \exists R\sqsubseteq A$, which means that
%if $\Kmc\models S\sqsubseteq R$ then $\Kmc\models \exists S\sqsubseteq A$. Then, in the sub-cases (a) and (c), by definition of $\canonmodel$, 
%$c_D\in A^{\Imc_\Kmc}$, since $D$ is of the form $\exists S$ in these sub-cases
%and $\Kmc\models \exists S\sqsubseteq A$. 
%In the sub-case (c), since $\Kmc\models D\sqsubseteq \exists S$ and (as just argued) $\Kmc\models \exists S\sqsubseteq A$,
%we have that $\Kmc\models D\sqsubseteq A$. Then, by definition of $\canonmodel$, 
%$c_D\in A^{\Imc_\Kmc}$.
%Finally, in the sub-case (b), we have that
%$\Kmc\models \overline{S}\sqsubseteq R$. By assumption,
%$\Kmc\models \exists R\sqsubseteq A$, which means that
%$\Kmc\models \exists \overline{S}\sqsubseteq A$. 
%Since in this sub-case  $\Kmc\models D\sqsubseteq \exists \overline{S}$,
%we have that $\Kmc\models D\sqsubseteq A$.
%So $c_D\in A^{\Imc_\Kmc}$, by definition of $\canonmodel$. 
%By assumption  $\Kmc\models \exists R\sqsubseteq A$, which means that
%if $\Kmc\models S\sqsubseteq R$ then $\Kmc\models \exists S\sqsubseteq A$. Then, by definition of $\canonmodel$,
%$c_D\in A^{\Imc_\Kmc}$ since $D$ is of the form $\exists S$
%and $\Kmc\models \exists S\sqsubseteq A$.
Since $c_D$ was an arbitrary
element in $\Delta_\Kmc$, this argument can be applied  for all  elements of this kind.
\end{itemize}
  We have thus shown that,
    for all elements $d$ in $\Delta^{\canonmodel}$,
    if $d\in (\exists R)^{\Imc_\Kmc}$ then $d\in A^{\Imc_\Kmc}$. So
$\Imc_\Kmc\models \exists R\sqsubseteq A$.

\smallskip

\noindent
\textbf{Now, assume $\Kmc\not\models \exists R\sqsubseteq A$.} If $\Kmc\not\models \exists R\sqsubseteq A$
then there is an interpretation \Imc that satisfies \Kmc with
$(\exists R)^\Imc$ non-empty. This means that
$\exists R$ is satisfiable w.r.t. \Kmc and thus
$c_{\exists R}\in \Delta_\Kmc$.
We show that $\canonmodel\not\models \exists R\sqsubseteq A$ by
showing that 
$c_{\exists R}\in 
(\exists R)^{\canonmodel}$
but $c_{\exists R}\not\in A^{\canonmodel}$.
By the definition of $\canonmodel$, $(c_{\exists S},c_{\exists \overline{S}})\in R^{\canonmodel}$ if $\Kmc \models S\sqsubseteq R$, which is trivially the case for $S=R$.
So $c_{\exists R}\in 
(\exists R)^{\canonmodel}$.
We now argue that $c_{\exists R}\not\in A^{\canonmodel}$.
By definition of $\canonmodel$, an element of the form
$c_D$ is in $A^{\canonmodel}$ iff
  $\mathcal{K} \models D\sqsubseteq A $. By assumption $\Kmc\not\models \exists R\sqsubseteq A$. So  
  $c_{\exists R}$ is not in $ A^{\canonmodel}$.
\end{proof}
\begin{claim}\label{clm:aexistsr}
 $\Imc_\Kmc\models A\sqsubseteq \exists R$   iff $\Kmc\models A\sqsubseteq \exists R$.
\end{claim}
\begin{proof}
\textbf{Assume $\Kmc\models A\sqsubseteq \exists R$}.
We make a case distinction based 
on the elements in $\Delta^{\Imc_\Kmc}:=\NI \cup \Delta_\Kmc$. 
\begin{itemize}
\item $a\in \NI$: Assume $a\in A^{\Imc_\Kmc}$. In this case, by definition of $\Imc_\Kmc$, we have that
$\Kmc\models A(a)$. By assumption, $\Kmc\models A\sqsubseteq \exists R$. So
$\Kmc\models \exists R(a)$. Then, by definition
of $\Imc_\Kmc$, $(a,c_{\exists \overline{R}})\in R^{\Imc_\Kmc}$.
Thus, $a\in (\exists R)^{\Imc_\Kmc}$, as required. 
 Since $a$ was an arbitrary
element in $\NI$, this argument can be applied  for all  elements of this kind.
\item $c_D\in \Delta_\Kmc$:   
Assume $c_D\in A^{\canonmodel}$. In this case, by definition of $\canonmodel$, 
we have that $\Kmc\models D\sqsubseteq A$. By assumption,
$\Kmc\models A\sqsubseteq \exists R$, so $\Kmc\models D\sqsubseteq \exists R$. Then, by definition of $\canonmodel$, 
we have that $(c_D,c_{\exists \overline{S}})\in R^{\Imc_\Kmc}$ for $S\in\NRm$ such that
$\Kmc\models S\sqsubseteq R$. Taking $S=R$ this trivially holds, so $(c_D,c_{\exists \overline{R}})\in R^{\Imc_\Kmc}$.
This means that $c_D\in (\exists R)^{\Imc_\Kmc}$, as required. 
 Since $c_D$ was an arbitrary
element in $\Delta_\Kmc$, this argument can be applied  for all  elements of this kind.
\end{itemize}
  We have thus shown that,
    for all elements $d$ in $\Delta^{\canonmodel}$,
    if $d\in A^{\Imc_\Kmc}$ then $d\in (\exists R)^{\Imc_\Kmc}$. So
$\Imc_\Kmc\models A\sqsubseteq \exists R$.

\smallskip

\noindent
\textbf{Now, assume $\Kmc\not\models A\sqsubseteq \exists R$.} If $\Kmc\not\models A\sqsubseteq \exists R$
then there is an interpretation \Imc that satisfies \Kmc with
$A^\Imc$ non-empty. This means that
$A$ is satisfiable w.r.t. \Kmc and thus
$c_{A}\in \Delta_\Kmc$.
We show that $\canonmodel\not\models A\sqsubseteq \exists R$ by
showing that 
$c_{A}\in 
A^{\canonmodel}$
but $c_{A}\not\in (\exists R)^{\canonmodel}$.
By definition of $\canonmodel$, an element of the form
$c_D$ is in $A^{\canonmodel}$ iff
  $\mathcal{K} \models D\sqsubseteq A $, which is
  trivially the case for $D=A$.
So $c_{A}\in 
A^{\canonmodel}$.
We now argue that $c_{A}\not\in (\exists R)^{\canonmodel}$.
By definition of $\canonmodel$,  
$c_A\in (\exists R)^{\canonmodel}$ iff
  $\mathcal{K} \models A\sqsubseteq \exists S $ and $\Kmc\models S\sqsubseteq R$. By assumption $\Kmc\not\models A\sqsubseteq \exists R$.
  So either $\mathcal{K} \not\models A\sqsubseteq \exists S $ or $\Kmc\not\models S\sqsubseteq R$. Thus,  
  $c_{A}\notin (\exists R)^{\canonmodel}$.
\end{proof}
%Note that in our definition of the canonical model, in particular, the part about roles, we always define analogous cases to cover inverse roles. 
\begin{claim}\label{clm:conceptdisjointness}
 $\Imc_\Kmc\models A\sqsubseteq \neg\exists R$   iff $\Kmc\models A\sqsubseteq \neg \exists R$.
\end{claim}
\begin{proof}
\textbf{Assume $\Kmc\models A\sqsubseteq \neg \exists R$}.
We make a case distinction based 
on the elements in $\Delta^{\Imc_\Kmc}:=\NI \cup \Delta_\Kmc$. 
\begin{itemize}
\item $a\in \NI$: Assume $a\in A^{\Imc_\Kmc}$. In this case, by definition of $\Imc_\Kmc$, we have that
$\Kmc\models A(a)$. By assumption, $\Kmc\models A\sqsubseteq \neg\exists R$. 
%%on going
As \Kmc is satisfiable,
$\Kmc\not\models \exists R(a)$. 
By definition of $\Imc_\Kmc$, to show that
$\Imc_\Kmc\not\models \exists R(a)$,
we need to rule out the following two cases, where  $R'\in\NR$.
\begin{itemize}
\item $(a,b)\in R'^{\Imc_\Kmc}$ and $R'=R$. Since this would imply 
$\Kmc\models  R'(a,b)$ and
$\Kmc\not\models \exists R(a)$, this cannot happen. 
    \item $(a,c_{\exists S})\in R'^{\Imc_\Kmc}$ and $R'=R$. Since $\Kmc\not\models \exists R(a)$, there is no role $S$ such that  $(a,c_{\exists S})\in R'^{\Imc_\Kmc}$, $\Kmc\models\exists \overline{S}(a)$, $\Kmc\models \overline{S}\sqsubseteq R$. So this case cannot happen
    %indeed $(a,c_{\exists S})\notin R'^{\Imc_\Kmc}$ 
    because otherwise we would have that $\Kmc\models \exists R(a)$. 
    \item $(c_{\exists S},a)\in R'^{\Imc_\Kmc}$ and $R'=R^-$. Since $\Kmc\not\models \exists R(a)$, there is no role $S$ such that $(c_{\exists S},a)\in R'^{\Imc_\Kmc}$, $\Kmc\models\exists \overline{S}(a)$, and $\Kmc\models {S}\sqsubseteq R$. So this case cannot happen 
    %indeed 
    %$(c_{\exists S},a)\notin R'^{\Imc_\Kmc}$  
    because otherwise   $\Kmc\models \exists R(a)$ would hold. 
\end{itemize}
Then, by definition
of $\Imc_\Kmc$, we have that $\Imc_\Kmc\not\models \exists R(a)$. 
 Since $a$ was an arbitrary
element in $\NI$, this argument can be applied  for all  elements of this kind.
\item $c_D\in \Delta_\Kmc$:   
Assume $c_D\in A^{\canonmodel}$. In this case, by definition of $\canonmodel$, 
we have that $\Kmc\models D\sqsubseteq A$. By assumption,
$\Kmc\models A\sqsubseteq \neg\exists R$, so $\Kmc\models D\sqsubseteq \neg\exists R$. In the second item of the argument of Claim~\ref{clm:exists}, we have shown that
if $c_D\in (\exists R)^{\Imc_\Kmc}$ then
$\Kmc\models D\sqsubseteq \exists R$. 
%(one can check that we considered all cases for when $c_D\in (\exists R)^{\Imc_\Kmc}$ and in all of them $\Kmc\models D\sqsubseteq \exists R$). 
As 
$c_D\in \Delta_\Kmc$, we have that $D$ is satisfiable. So $\Kmc\not\models D\sqsubseteq \exists R$. 
Then, by contrapositive, we have that
$c_D\notin (\exists R)^{\Imc_\Kmc}$.
%\todo{Ana:WIP}
%By definition of $\canonmodel$, to show that
%$c_D\notin (\exists R)^{\Imc_\Kmc}$,
%we need to rule out the following   cases, where  $R'\in\NR$.
%\begin{itemize}
  %  \item $(c_{\exists S},c_{\exists \overline{S}})\in R'^{\Imc_\Kmc}$, $R'=R$, and $c_D=c_{\exists S}$.
 %   By definition of $\Imc_\Kmc$, $\Kmc\models S\sqsubseteq R'$.
    %$\{(c_{\exists S},c_{\exists \overline{S}})\in \Delta_\Kmc \times \Delta_\Kmc \mid \Kmc\models S\sqsubseteq R\} $ \\  
%\{(c_{\exists S^-},c_{\exists S})\in \Delta_\Kmc \times \Delta_\Kmc \mid \Kmc\models S^-\sqsubseteq R\} \cup \\
%\item $\{(c_D,c_{\exists S})\in \Delta_\Kmc \times \Delta_\Kmc {\mid} \Kmc\models D\sqsubseteq \exists \overline{S},\ \Kmc\models \overline{S}\sqsubseteq R\}$ \\ 
%\item $\{(c_{\exists S},c_{D})\in \Delta_\Kmc \times \Delta_\Kmc \mid \Kmc\models D\sqsubseteq \exists \overline{S},\ \Kmc\models S\sqsubseteq R\}$
%\end{itemize}
%we have that $(c_D,c_{\exists \overline{S}})\in R^{\Imc_\Kmc}$ for $S\in\NRm$ such that
%$\Kmc\models S\sqsubseteq R$. Taking $S=R$ this trivially holds, so $(c_D,c_{\exists \overline{R}})\in R^{\Imc_\Kmc}$.
%This means that $c_D\in (\exists R)^{\Imc_\Kmc}$, as required. 
% Since $c_D$ was an arbitrary
%element in $\Delta_\Kmc$, this argument can be applied  for all  elements of this kind.
\end{itemize}
  We have thus shown that,
    for all elements $d$ in $\Delta^{\canonmodel}$,
    if $d\in A^{\Imc_\Kmc}$ then $d\in (\neg\exists R)^{\Imc_\Kmc}$. So
$\Imc_\Kmc\models A\sqsubseteq \neg\exists R$.

\smallskip

\noindent
\textbf{Now, assume $\Kmc\not\models A\sqsubseteq \neg \exists R$.} If $\Kmc\not\models A\sqsubseteq \neg \exists R$
then there is an interpretation \Imc that satisfies \Kmc with
$A^\Imc\cap (\exists R)^\Imc$ non-empty. This means that
$A\sqcap \exists R$ is satisfiable w.r.t. \Kmc and thus
$c_{A\sqcap \exists R}\in \Delta_\Kmc$.
By definition of $\canonmodel$, we have that
$c_{A\sqcap \exists R}\in A^{\canonmodel}$ and $c_{A\sqcap \exists R}\in (\exists R)^{\canonmodel}$  since $\Kmc\models A\sqcap \exists R\sqsubseteq A$ and $\Kmc\models A\sqcap \exists R\sqsubseteq \exists R$ (in more details, we have that $(c_{A\sqcap\exists R},c_{\exists R^-})\in R^{\Imc_\Kmc}$). % 
 So  $\canonmodel\not\models A\sqsubseteq \neg \exists R$. 
\end{proof}

\begin{claim}
 $\Imc_\Kmc\models \exists R\sqsubseteq \neg A$   iff $\Kmc\models \exists R\sqsubseteq \neg A$.
\end{claim}
\begin{proof}
Since $\Kmc\models \exists R\sqsubseteq \neg A$ iff $\Kmc\models A\sqsubseteq \neg \exists R$ and $\Imc_\Kmc\models \exists R\sqsubseteq \neg A$ iff $\Imc_\Kmc\models A\sqsubseteq \neg \exists R$, this claim follows from Claim~\ref{clm:conceptdisjointness}. 
\end{proof}

Regarding the  assertions, we have that $\Imc_\Kmc\models R(a,b)$   iff $\Kmc\models R(a,b)$ directly follows from the definition of $\Imc_\Kmc$ (Definition~\ref{def:canonicalModel}). The same holds for $\Imc_\Kmc\models A(a)$   iff $\Kmc\models A(a)$.
It remains to argue about assertions of the form $\exists R'(a)$ (where $R'$ is a role name or its inverse).
\begin{claim}
 $\Imc_\Kmc\models \exists R'(a)$   iff $\Kmc\models \exists R'(a)$.
\end{claim}
\begin{proof}
We first show that if $\Kmc\models \exists R'(a)$ then $\Imc_\Kmc\models \exists R'(a)$.
We make a case distinction.
\begin{itemize}
    \item $R'=R$.
If $\Kmc\models \exists R(a)$ and $R\in\NR$ then 
$(a,c_{\exists \overline{R}})\in R^{\Imc_\Kmc}$ (take $\overline{S}=R$ in Definition~\ref{def:canonicalModel}). So 
 $\Imc_\Kmc\models \exists R(a)$.
 \item $R'=R^-$. If $\Kmc\models \exists R^-(a)$ and $R\in\NR$ then 
$(c_{\exists {R}},a)\in R^{\Imc_\Kmc}$ (take   $\overline{S}=R^-$ in Definition~\ref{def:canonicalModel}). So 
 $\Imc_\Kmc\models \exists R^-(a)$. 
\end{itemize} 
Conversely, assume $\Imc_\Kmc\models \exists R'(a)$. There are two   cases. 
  \begin{itemize}
      \item $R'=R$ and $(a,c_{\exists S})\in R^{\Imc_\Kmc}$ for some $S$ such that $\Kmc\models \exists \overline{S}(a)$ and $\Kmc\models \overline{S}\sqsubseteq R$. 
      In this case, $\Kmc\models \exists R(a)$. Then
      %as $\Kmc\models \exists \overline{S}(a)$ and $\Kmc\models \overline{S}\sqsubseteq R$ 
       $\Kmc\models \exists R'(a)$ since here $R'=R$.
      \item $R'=R^-$ and $(c_{\exists S},a)\in R^{\Imc_\Kmc}$ for some $S$ such that $\Kmc\models \exists \overline{S}(a)$ and $\Kmc\models {S}\sqsubseteq R$. In this case, note that $(c_{\exists S},a)\in R^{\Imc_\Kmc}$ means $\Imc_\Kmc\models \exists R^-(a)$.
      As $\Kmc\models \exists \overline{S}(a)$ and $\Kmc\models {S}\sqsubseteq R$ we have that $\Kmc\models \exists R^-(a)$ so $\Kmc\models \exists R'(a)$ since here $R'=R^-$.
  \end{itemize}
\end{proof}
We have completed now our proof with all DL-Lite$^\Hmc$ axioms.
%except for role disjointness. We note that there would be not much technical difficulty to change our definition of canonical model and proof in order to cover role disjointness in this theorem. However, as we mention in the paper there is a technical difficulty regarding our box interpretations. So we left this axiom out of this construction.
\end{proof}

\begin{corollary}\label{cor:ModelExistence}
Let \Kmc be a satisfiable DL-Lite$^\Hmc$ KB  and let $\Imc'_{\Kmc}$ be the result of modifying the canonical model  $\Imc_{\Kmc}$  of  \Kmc by setting
$\Delta_\Kmc$ as $\{c_{\exists R},c_{\exists R^-}\mid R\in\NR\}$.
    Then,  for all DL-Lite$^\Hmc$ axioms $\alpha$, % over   ${\sf sig}(\Kmc)$, 
    we have that if $\Kmc \models \alpha$ then $\Imc_{\Kmc} \models \alpha$.
\end{corollary}
\begin{proof}
    The proof is the same as in \cref{thm:canonicalModel}, except that since here we only want one direction of the theorem (showing that the interpretation satisfies \Kmc) we do not require all the elements of $\Delta_\Kmc$ used in \cref{thm:canonicalModel}.
\end{proof}
\thmprop*
\begin{proof}
%We show that for any \geometric{} interpretation $\cAssignSymb$:
%
%\begin{enumerate}
    %\item For all $A \in \NC$: there is a complement $\sComp{\cAssign{A}} \in \B$
    %\label{lat:comp1}
    %\item For all $A \in \NC$: if $\sComp{\sComp{\cAssign{A}}} = \cAssign{A}$.
    %\label{lat:comp2}
    %\item For all $A, B \in \NC$: if $\cAssign{A} \subseteq \cAssign{B}$ then %$\sComp{\cAssign{B}} \subseteq \sComp{\cAssign{A}}$.
 %   \label{lat:comp3}
%    \item For all $A, B \in \NC$: if $\cAssign{A} \subseteq \cAssign{B}$ and $\cAssign{A} \subseteq \cAssign{\neg B}$ then $\cAssign{A} = \emptyset$. 
%    \label{lat:comp4}
%\end{enumerate} 
(Condition \ref{lat:comp1}) What is to be shown is that for all concept embeddings $\cAssign{C} \in \B$ with $C \in \NCE$, there is a complement $\sComp{\cAssign{C}} \in \B$. 
By the definition of the complement of boxes, we know that $\sComp{\cAssign{C}}$ is:
\[
   \{\mathbf{x} \mid (-\worldSize - \mathbf{L_C} + \epsilonVec) \leqd \mathbf{x} \leqd (\worldSize - \mathbf{U_C} - \epsilonVec), \mathbf{x} \in \Rd \} . \]

Given \cref{def:box} and \cref{def:boxinterpretation}, 
%(1) set of boxes $\B$ and (2) the complement of a box, 
the bounds of the complement box $\sComp{\cAssign{C}}$ are $\mathbf{L_{\neg C}} = -\worldSize - \mathbf{L_C}$ and $\mathbf{U_{\neg C}} = \worldSize - \mathbf{U_C}$.
The negation of a box $\sComp{\cAssign{C}}$ is again a box, since: (i) we can compute the width of $\sComp{\cAssign{C}}$ as $ \mathbf{U_{\neg C}} - \mathbf{L_{\neg C}} = (\worldSize - \mathbf{U_C}) - (-\worldSize -\mathbf{L_C}) = 2\worldSize - (\mathbf{U_C} - \mathbf{L_C})  $ and (ii) by the definition of the set $\B$ of boxes  we know that %\todo{Ana: now the def of the set of boxes is different, without $\worldSize$ so this needs to change (?)}
$ \mathbf{0} \leqd (\mathbf{U_C} - \mathbf{L_C}) \leqd 2\worldSize$. From (i) and (ii) we have that $\mathbf{0} \leqd 2\worldSize - (\mathbf{U_C} - \mathbf{L_C}) \leqd  2\worldSize$, proving that $\sComp{\cAssign{C}} \in \B$. Thus, we have shown that for all concept embeddings $\cAssign{C} \in \B$ with $C \in \NCE$, it holds that $\sComp{\cAssign{C}} \in \B$.

(Condition~\ref{lat:comp2}) We show that $\sComp{\sComp{\cAssign{C}}} = \cAssign{C}$  for all $\cAssign{C} \in \B$ with $C \in \NCE$. We have that $\sComp{\sComp{\cAssign{C}}}$ is %equal to
    \begin{align*}
       %  & \\
        &=  \sComp{\sComp{\{\mathbf{x} \mid \mathbf{L_C} + \epsilonVec \leqd \mathbf{x} \leqd \mathbf{U_C} - \epsilonVec, \mathbf{x} \in \Rd \}}} \\
        &=  \sComp{\{\mathbf{x} \mid (-\worldSize - \mathbf{L_C} + \epsilonVec) \leqd \mathbf{x} \leqd (\worldSize - \mathbf{U_C} - \epsilonVec), \mathbf{x} \in \Rd \}} \\
        &=  \{\mathbf{x} \mid (-\worldSize -(-\worldSize - \mathbf{L_C}) + \epsilonVec) \leqd \mathbf{x} \leqd \\ &\qquad (\worldSize - (\worldSize - \mathbf{U_C}) - \epsilonVec), \mathbf{x} \in \Rd \} \\
        &=  \{\mathbf{x} \mid \mathbf{L_C} + \epsilonVec \leqd \mathbf{x} \leqd \mathbf{U_C} - \epsilonVec, \mathbf{x} \in \Rd \} \\
        &= \cAssign{C}.
    \end{align*}
(Condition \ref{lat:comp3}) We show that if $\cAssign{C} \subseteq \cAssign{D}$ then $\sComp{\cAssign{D}} \subseteq \sComp{\cAssign{C}}$, for all $\cAssign{C}, \cAssign{D} \in \B$  with $C,D \in \NCE$:  
    \begin{align*}
         &  \cAssign{C} \subseteq \cAssign{D}\\
        &\Leftrightarrow (\mathbf{L_D} \leqd \mathbf{L_C}) \land (\mathbf{U_C} \leqd \mathbf{U_D}) \\
        &\Leftrightarrow ((\mathbf{\worldSize} + \mathbf{L_D}) \leqd (\mathbf{\worldSize} + \mathbf{L_C})) \land ((- \mathbf{\worldSize} + \mathbf{U_C}) \leqd\\
        & \qquad (- \mathbf{\worldSize} + \mathbf{U_D})) \\
        &\Leftrightarrow ((- \mathbf{\worldSize} - \mathbf{L_C}) \leqd (- \mathbf{\worldSize} - \mathbf{L_D})) \land ((\mathbf{\worldSize} - \mathbf{U_D}) \leqd \\&\qquad (\mathbf{\worldSize} - \mathbf{U_C})) \\
        &\Leftrightarrow (\mathbf{L_\sComp{\cAssign{C}}} \leqd\mathbf{L_\sComp{\cAssign{D}}}) \land (\mathbf{U_\sComp{\cAssign{D}}} \leqd \mathbf{U_\sComp{\cAssign{C}}}) \\
        &\Leftrightarrow \sComp{\cAssign{D}} \subseteq \sComp{\cAssign{C}}.
    \end{align*}
% Finally, Conditions \ref{lat:comp4}: This directly follows from $\cAssignSymb$ being consistent.
\end{proof}

\section{Proofs for Section~\ref{sec:faithful}}
\label{app:faifthul}
We start by proving \cref{lem:tboxfaith}
and \cref{lem:kbfaith} stated in the main text of Section~\ref{sec:faithful}.

\lemweakfaithemptyabox*
\begin{proof}
    Let $\eta$ be a box interpretation for a DL-Lite$^\Hmc$ KB \Kmc with empty ABox and assume $\eta \models \Kmc$. 
    \begin{claim}\label{clm:satisfiable}
    If \Kmc is a DL-Lite$^\Hmc$ KB with an empty ABox then \Kmc is  satisfiable.
    \end{claim}
     \begin{proof}
         Let \Imc be an interpretation such that for all $A\in\NC$ and all $R\in\NR$ we have that
          $A^\Imc=R^\Imc=\emptyset$.
          Then, trivially,
          $ C^\Imc\subseteq D^\Imc$
          and $R^\Imc\subseteq S^\Imc$
          where $C,D$ are arbitrary 
          DL-Lite$^\Hmc$ concepts
          and $R,S$ are arbitrary roles 
          with symbols in $\NC$ and $\NR$.
          This means that, $\Imc\models\alpha$ for all
          (concept/role) inclusions
          $\alpha$ with symbols in $\NC\cup\NR$.
          If \Kmc has an empty ABox, all
          axioms in \Kmc are inclusions and thus all are satisfied by \Imc. In other words, $\Imc\models\Kmc$, so \Kmc is satisfiable.
     \end{proof}
\begin{claim}\label{clm:unsatisfiable}
    Assume \Kmc is a satisfiable DL-Lite$^\Hmc$ KB with an empty ABox.  If $\Kmc\cup \{\alpha\}$ is unsatisfiable then
    $\alpha$ is an assertion.
    %
    %of the form $C(a)$ with
    %$C$ unsatisfiable w.r.t. \Kmc.  
    \end{claim}
     \begin{proof}
         Suppose $\Kmc\cup \{\alpha\}$ is unsatisfiable and \Kmc is a satisfiable DL-Lite$^\Hmc$ KB with an empty ABox. By 
         %Claim~\ref{clm:satisfiable},
         the proof of Claim~\ref{clm:satisfiable},
         any interpretation \Imc with 
          $A^\Imc=R^\Imc=\emptyset$, for all $A\in\NC$ and all $R\in\NR$, satisfies not only \Kmc but also any extension of \Kmc with an inclusion.
          Thus, if $\Kmc\cup \{\alpha\}$ 
          is unsatisfiable
          $\alpha$ cannot be an inclusion.
          That is, $\alpha$ must be an assertion.
     \end{proof}
  By Claim~\ref{clm:satisfiable},
     \Kmc is  satisfiable.
    By Claim~\ref{clm:unsatisfiable},
    $\Kmc\cup \{\alpha\}$ is unsatisfiable only if
    $\alpha$ is an assertion.
    %
    %of the form $C(a)$ with
    %$C$ unsatisfiable w.r.t. \Kmc. 
    Then the lemma follows since we only require TBox faithfulness, so $\alpha$ in Definition~\ref{def:faithfulness} ranges only over inclusions, not assertions.
\end{proof}
\lemweakfaith*
\begin{proof}
%We first show the following technical claim.
%\begin{claim}
 %   \label{clm:satisfiablekb}
  %  Let \Kmc be a DL-Lite$^\Hmc$ KB. 
   % If $\eta$ is a consistent box  interpretation  and $\eta \models \Kmc$ then
    % \Kmc is satisfiable.
    %\end{claim}
    %\begin{proof}
  %\todo{Ana:WIP}
  %  \end{proof}
    Let $\eta$ be a consistent box interpretation for a DL-Lite$^\Hmc$ KB \Kmc  and assume $\eta \models \Kmc$.
Let $\Kmc'$ be $\{\alpha\mid \eta\models \alpha\}$. As $\NI,\NC,\NR$ are finite,
there are   finitely many axioms, so $\Kmc'$ is finite.
As $\eta$ is box consistent, there is an interpretation \Imc that satisfies
%by Claim~\ref{clm:satisfiablekb},
$\Kmc'$. Indeed,
$\Imc$ can be constructed from $\eta$ as follows.
\begin{itemize}
    \item Let $\Delta^\Imc:=\NI\cup \Delta_\Kmc$ (where $\Delta_\Kmc$ is as in \cref{def:canonicalModel}).
    \item For all $a\in\NI$, we have that $a^\Imc:=a$. Moreover,
    \item $A^\Imc:=\{a\in\NI\mid \eta\models A(a)\}\cup \{c_D\in\Delta_\Kmc\mid
    \eta\models D\sqsubseteq A\}$, for all $A\in\NC$, and
    \item   $R^{\Imc}:= \{(a,b)\in \NI\times\NI\mid  \eta\models R(a,b)\}\cup  \\ 
\{(a,c_{\exists S})\in \NI\times \Delta_\Kmc \mid  \eta\models \exists \overline{S}(a),\ \eta\models \overline{S}\sqsubseteq R\}\cup  \\
\{(c_{\exists S},a)\in \Delta_\Kmc\times \NI \mid  \eta\models \exists \overline{S}(a),\ \eta\models S\sqsubseteq R\}\cup  \\
\{(c_{\exists S},c_{\exists \overline{S}})\in \Delta_\Kmc \times \Delta_\Kmc \mid \eta\models S\sqsubseteq R\} \cup \\  
%\{(c_{\exists S^-},c_{\exists S})\in \Delta_\Kmc \times \Delta_\Kmc \mid \Kmc\models S^-\sqsubseteq R\} \cup \\
\{(c_D,c_{\exists S})\in \Delta_\Kmc \times \Delta_\Kmc {\mid} \eta\models D\sqsubseteq \exists \overline{S},\ \eta\models \overline{S}\sqsubseteq R\}\cup \\ \{(c_{\exists S},c_{D})\in \Delta_\Kmc \times \Delta_\Kmc \mid \eta\models D\sqsubseteq \exists \overline{S},\ \eta\models S\sqsubseteq R\}$, for all $R\in\NR$.
\end{itemize}
We can see that \Imc is defined using $\eta$ in the same way $\Imc_\Kmc$ is defined using \Kmc in \cref{def:canonicalModel} (box consistency of $\eta$ ensures that \Imc is well-defined). One can then employ the same argument as in \cref{thm:canonicalModel} to show that
$\eta\models\alpha$ iff $\Imc\models\alpha$. This means that $\Imc\models\Kmc'$.
%is satisfiable.
%Let $\Imc_{\Kmc'}$ be the canonical model
%of $\Kmc'$ (\cref{def:canonicalModel}).

If $\eta\models\alpha$ then 
$\alpha\in\Kmc'$.
%By \cref{thm:canonicalModel},
As $\Imc\models \Kmc'$, we have that
$\Imc\models \alpha$.
By assumption $\eta \models \Kmc$.
Then, by definition, $\Kmc'$ contains all axioms in \Kmc. As $\Imc\models \Kmc'$, we have that
$\Imc\models \Kmc$.
%by \cref{thm:canonicalModel}, $\Imc_{\Kmc'}\models \Kmc$.
%As $\Imc_{\Kmc'}\models \alpha$ and $\Imc_{\Kmc'}\models \Kmc$, we have that
This means that $\Imc\models \Kmc\cup\{\alpha\}$. So $\Kmc\cup\{\alpha\}$ is satisfiable.
\end{proof}

%Let $\Imc$ be an interpretation
%with domain $\NI$ union
%$\{c_D\mid D\text{ is satisfiable w.r.t. }\Kmc\}$ and
%\begin{itemize}
 %   \item $A^\Imc:=\{a\in \NI\mid \eta\models A(a) \}\cup 
  %  \{c_A\text{ if } c_A\in \Delta^\Imc\}$
%\end{itemize}

\begin{table}[t]
    \centering
    \begin{tabular}{ccc}
        \toprule
         Box Name & $\mathbf{L}[i_C]$ & $\mathbf{U}[i_C]$ \\
         \midrule
%         $\Sempty$ & 0 & 0  \\ 
         $\Seq$ & -4 & -0.5  \\ 
         % $\Sneq$ & 0 & 4.5 \\
         $\Ssub$ & -2 & -0.5  \\
         % $\Snsub$ & -2 & 4.5  \\
         $\Ssup$ & -4 & 2  \\
         $\Snsup$ & 0 & 2  \\
         $\Sinter$ & -2 & 2  \\
         % $\Sninter$ & -2 & 2  \\
         %$\BB$ & $-\epsilon$ & $\epsilon$  \\
         {$\BB$} & $0$ & $0$  \\
         $\World$ & $-4$ & $4$ \\
         \bottomrule
    \end{tabular}
    \caption{
    %$\cAssignSymb_{\Imc_{\ontooC}}C$'s 
    Parameters for boxes $\Seq$, % $\Sneq$, 
    $\Ssub$, %$\Snsub$, 
    $\Ssup$, $\Snsup$, $\Sinter$, %\Sninter$, 
    and $\World$ in dimension $i_C$.}
    \label{tab:ConceptBoxes}
\end{table}

\begin{table}[t]
    \centering
    \begin{tabular}{ccc}
        \toprule
         Embedding Name & Value in Dimension $i_C$ \\
         \midrule
         $\PC$ & $\PCvalue$   \\
         $\PnC$ & $\PnCvalue$ \\
         $\BC$ & $\BCvalue$ \\
         \bottomrule
    \end{tabular}
    \caption{Parameters for   $\PC$ and $\PnC$ in dimension $i_C$.}
    \label{tab:ConceptIndividual}
\end{table}

\begin{table}[t]
    \centering
    \begin{tabular}{ccc}
        \toprule
         Box Name & $\mathbf{L}[i_{R,a}]$ & $\mathbf{U}[i_{R,a}]$ \\
         \midrule
         $\SC$ & $-2$ & $2$ \\ 
         $\Qasub$ & $-1$ & $1$ \\
         $\Qnasub$ & $-2$ & $1$ \\
         %$\BBRa$ & $-1 - \epsilon$ & $1 + \epsilon$ \\         
         %$\BBRna$ & $-1 - \epsilon$ & $0 + \epsilon$ \\   
          {$\BBRa$} & $-1  $ & $1  $ \\        {$\BBRna$} & $-1  $ & $0  $ \\   
         $\World$ & $-4$ & $4$ \\
         \bottomrule
    \end{tabular}
    \caption{Parameters for boxes $\SC$, $\Qasub$, $\Qnasub$, $\BBRa$, $\BBRna$, and $\World$ in dimension $i_{R,a}$.}
    \label{tab:RoleBoxes}
\end{table}
 
\begin{table}[t]
    \centering
    \begin{tabular}{ccc}
        \toprule
         Embedding Name & Value in Dimension $i_{R,a}$ \\
         \midrule
         $\PRa$ & $\PRavalue$ \\ %$-1 + \epsilon$ \\         
         $\PRna$ &  $\PRnavalue$ \\ %$1 - \epsilon$ \\                  
         $\BRa$ & $\BRavalue$ \\ % $0$ \\         
         $\BRna$ & $\BRnavalue$ \\ % $-1$ \\
         \bottomrule
    \end{tabular}
    \caption{Parameters for point embeddings $\PRa$, $\PRna$, $\BRa$, and $\BRna$ in dimension $i_{R,a}$.}
    \label{tab:RoleIndividual}
\end{table}

\begin{figure}[t]
    \centering
    \resizebox{\columnwidth}{!}{%
    \begin{tikzpicture}[dot/.style={circle, minimum size=2mm, inner sep=1pt, fill, name=#1 }]
    \large
    \definecolor{pastelgreen}{rgb}{0.01, 0.75, 0.24}
    \definecolor{pastelpurple}{rgb}{0.59, 0.44, 0.84}
    \definecolor{pastelorange}{rgb}{1.0, 0.7, 0.28}
    \definecolor{darkpastelblue}{rgb}{0.47, 0.62, 0.8}
    \definecolor{pastelred}{rgb}{1.0, 0.41, 0.38}
    \definecolor{pastelgray}{rgb}{0.25, 0.25, 0.28}

%    \usetikzlibrary {arrows.meta}
%    \draw [line width=0.5mm, color=pastelgray, arrows = {-Latex[width=10pt, length=10pt]}] (-5.5,0) -- (5.5,0);
    \draw [->, line width=0.5mm, color=pastelgray] (-5.5,0) -- (5.5,0);
    \node [color=pastelgray, label={[text=pastelgray]:$i_{R,a}$}] at (5.5,0) {};
        
    \node [color=pastelpurple, dot=neg_sw, label={[text=pastelpurple]:$-\worldSize$}] at (-4,0) {};
    \node [color=pastelpurple, dot=pos_sw, label={[text=pastelpurple]:$\worldSize$}] at (4,0) {};
    \draw[line width=0.5mm, color=pastelpurple] (neg_sw) -- (pos_sw);
%    \node [dot=zero, label=$\mathbf{0}$] at (0,0) {};

    \node [color=pastelgreen, dot=l_s_c] at (-2,\intervalSkip + \subSkip/2) {};
    \node [color=pastelgreen, dot=u_s_c] at (2,\intervalSkip + \subSkip/2) {};
    \draw[line width=0.5mm, color=pastelgreen] (l_s_c) -- node[label={$\SC$
    % = \SnC
    }] {} ++ (u_s_c);    

    \node [color=pastelorange, dot=l_q_asub] at (-1,2*\intervalSkip) {};
    \node [color=pastelorange, dot=u_q_asub] at (1,2*\intervalSkip) {};
    \draw[line width=0.5mm, pastelorange] (l_q_asub) -- node[label=$\Qasub$] {} ++ (u_q_asub);
    \draw[->, line width=0.85mm, color=pastelred] (-2,2*\intervalSkip) -- (l_q_asub);
    \draw[->, line width=0.85mm, color=pastelred] (2,2*\intervalSkip) -- (u_q_asub);

    \node [color=darkpastelblue, dot=l_q_nasub] at (-2,3*\intervalSkip) {};
    \node [color=darkpastelblue, dot=u_q_nasub] at (1,3*\intervalSkip) {};
    \draw[line width=0.5mm, darkpastelblue] (l_q_nasub) -- node[label=$\Qnasub$] {} ++ (u_q_nasub);
    \draw[->, line width=0.85mm, color=pastelred] (2,3*\intervalSkip) -- (u_q_nasub);

    \node [color=pastelgray, dot=pna, label={[text=pastelgray]:$\PRa$}] at (-1,0) {};
    \node [color=pastelgray, dot=pa, label={[text=pastelgray]:$\PRna$}] at (1,0) {};

    \draw[->, line width=0.85mm, color=pastelred] (pna) -- (-2,0);
    \draw[->, line width=0.85mm, color=pastelred] (pa) -- (0,0);

    \node [color=pastelred, label={[text=pastelred, below]:$\BRna$}] at (-2,-0.25) {};
    \node [color=pastelred, label={[text=pastelred, below]:$\BRna$}] at (0,-0.25) {};
    \end{tikzpicture}
    }
\caption{Visualization of the parameters of $\SC$, $\Qasub$, $\Qnasub$, $\PRa$, $\PRna$, $\BRa$, $\BRna$, and $\World$ in dimension $i_{R,a}$. }
    \label{fig:RoleBoxes}
\end{figure}
\begin{figure}[t]
    \centering
    \resizebox{\columnwidth}{!}{%
    \begin{tikzpicture}[dot/.style={circle, minimum size=2mm, inner sep=1pt, fill, name=#1 }]
    \large
    \definecolor{pastelgreen}{rgb}{0.01, 0.75, 0.24}
    \definecolor{pastelpurple}{rgb}{0.59, 0.44, 0.84}
    \definecolor{pastelorange}{rgb}{1.0, 0.7, 0.28}
    \definecolor{darkpastelblue}{rgb}{0.47, 0.62, 0.8}
    \definecolor{pastelred}{rgb}{1.0, 0.41, 0.38}
    \definecolor{pastelgray}{rgb}{0.25, 0.25, 0.28}
    
%    \usetikzlibrary {arrows.meta}
%    \draw [line width=0.5mm, color=pastelgray, arrows = {-Latex[width=10pt, length=10pt]}] (-5.5,0) -- (5.5,0);
    \draw [->, line width=0.5mm, color=pastelgray] (-5.5,0) -- (5.5,0);
    \node [color=pastelgray, label={[text=pastelgray]:$i_C$}] at (5.5,0) {};
    
    \node [color=pastelpurple, dot=neg_sw, label={[text=pastelpurple]:$-\worldSize$}] at (-4,0) {};
    \node [color=pastelpurple, dot=pos_sw, label={[text=pastelpurple]:$\worldSize$}] at (4,0) {};
    \draw[line width=0.5mm, color=pastelpurple] (neg_sw) -- (pos_sw);
%    \node [dot=zero, label=$\mathbf{0}$] at (0,0) {};
 
    \node [color=pastelgreen, dot=l_s_eq] at (-4,3*\intervalSkip) {};
    \node [color=pastelgreen, dot=u_s_eq] at (-0.5,3*\intervalSkip) {};
    \draw[line width=0.5mm, color=pastelgreen] (l_s_eq) -- node[label=$\Seq$] {} ++ (u_s_eq);    
    % \node [color=pastelgreen, dot=l_s_neq] at (-4+4,3*\intervalSkip) {};
    % \node [color=pastelgreen, dot=u_s_neq] at (4+0.5,3*\intervalSkip) {};
    % \draw[line width=0.5mm, pastelgreen] (l_s_neq) -- node[label=$\Sneq$] {} ++ (u_s_neq);

    \node [color=pastelorange, dot=l_s_sub] at (-2,1*\intervalSkip+\subSkip/2) {};
    \node [color=pastelorange, dot=u_s_sub] at (-0.5,1*\intervalSkip+\subSkip/2) {};
    \draw[line width=0.5mm, pastelorange] (l_s_sub) -- node[label=$\Ssub$] {} ++ (u_s_sub);
    % \node [color=pastelorange, dot=l_s_nsub] at (-4+2,5*\intervalSkip-\subSkip/2) {};
    % \node [color=pastelorange, dot=u_s_nsub] at (4+0.5,5*\intervalSkip-\subSkip/2) {};
    % \draw[line width=0.5mm, pastelorange] (l_s_nsub) -- node[label=$\Snsub$] {} ++ (u_s_nsub);

    \node [color=darkpastelblue, dot=l_s_nsup] at (0,1*\intervalSkip+\subSkip/2) {};
    \node [color=darkpastelblue, dot=u_s_nsup] at (2,1*\intervalSkip+\subSkip/2) {};
    \draw[line width=0.5mm, darkpastelblue] (l_s_nsup) -- node[label=$\Snsup$] {} ++ (u_s_nsup);
    \node [color=darkpastelblue, dot=l_s_sup] at (-4+0,4*\intervalSkip-\subSkip/2) {};
    \node [color=darkpastelblue, dot=u_s_sup] at (4-2,4*\intervalSkip-\subSkip/2) {};
    \draw[line width=0.5mm, darkpastelblue] (l_s_sup) -- node[label=$\Ssup$] {} ++ (u_s_sup);

    \node [color=pastelred, dot=l_s_inter] at (-2,2*\intervalSkip+\subSkip/2) {};
    \node [color=pastelred, dot=u_s_inter] at (2,2*\intervalSkip+\subSkip/2) {};
    \draw[line width=0.5mm, pastelred] (l_s_inter) -- node[label={$\Sinter$ 
    %= \Sninter$
    }] {} ++ (u_s_inter);

    \node [color=pastelgray, dot=pd, label={[text=pastelgray]:$\PC$}] at (-1,0) {};
    \node [color=pastelgray, dot=pnd, label={[text=pastelgray]:$\PnC$}] at (1,0) {};
    \end{tikzpicture}
    }
\caption{Visualization of the parameters of $\Seq$, %$\Sneq$, 
$\Ssub$, %$\Snsub$, 
$\Ssup$, $\Snsup$, $\Sinter$, %$\Sninter$, 
$\PC$, $\PnC$ and $\World$ in dimension $i_C$. }
    \label{fig:ConceptBoxes}
\end{figure}

As already hinted in the main text, the proof strategy of \cref{thm:faith_bool_alc} consists of first creating a mapping between
finite interpretations and \geometric interpretations and then using the canonical model for a KB to establish KB faithfulness. 
\cref{def:mapping} describes this mapping. The values in Tables~\ref{tab:ConceptBoxes}, \ref{tab:ConceptIndividual}, \ref{tab:RoleBoxes}, and \ref{tab:RoleIndividual} satisfy the constraints in \cref{fig:RoleBoxes} and \cref{fig:ConceptBoxes}. Any choice of values that satisfy \cref{fig:RoleBoxes} and \cref{fig:ConceptBoxes} could be used.

%Then, given a DL-Lite$^\Hmc$ KB \Kmc, we construct the canonical model $\Imc_\Kmc$ in \cref{def:canonicalModel} and use the mapping to create 
%a \geometric interpretation $\eta_{\Imc_\Kmc}$ with  $s_\Omega=4$ which is a strongly KB faithful model of \Kmc. We provide a detailed proof in the appendix.
\begin{definition}\label{def:mapping}
   Given    an interpretation \Imc with finite domain,
   %over  finite sets of concept, role, and individual names, $\NC,\NR,\NI$, respectively. 
we define a box interpretation $\eta_\Imc$ in a $d$-dimensional %2\cdot |\NR|+
Euclidean space, where $d=|\NCE|
+ 
|\NR|\cdot |\Delta^\Imc|$, as follows.
To each concept $C\in\NCE$, we
associate a dimension in the vector space and denote its index by $i_C$.  
Similarly, to each pair $(R,\element)$ in $\NR\times \Delta^\Imc$, we
associate a dimension $i_{R,\element}$.
%in the vector space.
%
In our construction, we  use constants 
%$\BC$, $\BRa$, $\Ssub$, etc., that are 
defined in Tables~\ref{tab:ConceptBoxes}, \ref{tab:ConceptIndividual}, \ref{tab:RoleBoxes}, and \ref{tab:RoleIndividual}. 
For each $a\in \NI$, let 
$\eta_\Imc(a):=\{\posAssignI{\Imc}{a},\bumpAssignI{\Imc}{a}\}$ where:
\begin{itemize}
    \item for every $C\in \NCE$, 
    $\posAssignI{\Imc}{a}[i_C]=\PC$ if $a^\Imc\in C^\Imc$, and $\posAssignI{\Imc}{a}[i_C]=\PnC$ if $a^\Imc\notin C^\Imc$; also $\bumpAssignI{\Imc}{a}[i_C]=\BC$ (if $a^\Imc\in C^\Imc$ or not);
    \item for every pair $(R,\element)$ with $R\in \NR$ and $\element\in \Delta^\Imc$,
$\bumpAssignI{\Imc}{a}[i_{R,\element}]=\BRa$ if 
$(c,a^\Imc)\in R^\Imc$, otherwise $\bumpAssignI{\Imc}{a}[i_{R,\element}]=\BRna$, 
    $\posAssignI{\Imc}{a}[i_{R,\element}]=\PRa$ if $a^\Imc=c$, otherwise (that is, $a^\Imc\neq c$) $\posAssignI{\Imc}{a}[i_{R,\element}]=\PRna$.
%
%
     %if $x = a$ then $\posAssign{a}[i_{R, a}] \coloneqq \PRa$,
      %      else $\posAssign{a}[i_{R, a}] \coloneqq \PRna$.
      %   
       %  Furthermore, for any individual $x \in \Delta^{\Imc_\ontoo}$, we set:
     %    
        %     if $(a, x) \in R^\Imc$ then $\bumpAssign{a}[i_{R, a}] \coloneqq \BRa$,
         %    else $\bumpAssign{a}[i_{R, a}] \coloneqq \BRna$.
    %
\end{itemize}
For each $D\in \NCE$  { with $\emptyset = D^\Imc$}, 
let $\eta_\Imc(D)$ be the box with
all dimensions set to $\Sempty$, and,
for each $D\in \NCEneg$  { with $\emptyset \neq D^\Imc$}, %Ana: I thought about changing this to NC as planned but then it was complicated to define the head and tail of iC
let $\eta_\Imc(D)$ be the box where:
%
%with lower and upper bounds $\lowerboundI{\Imc}{C},\upperboundI{\Imc}{C}$, respectively,
%defined as:
\begin{itemize}
    \item for every $C\in \NCE$,  
    \begin{itemize}
    %\item $\eta_\Imc(D)[i_C]:=\Sempty$ iff $\emptyset = D^\Imc$;  
        \item 
%    $\eta_\Imc(D)[i_C]:=(-4, -0.5)$ 
    $\eta_\Imc(D)[i_C]:=\Seq$
    iff
    $\emptyset\neq D^\Imc = C^\Imc$,
\item
%     $\eta_\Imc(D)[i_C]:=(-2, -0.5)$ 
      $\eta_\Imc(D)[i_C]:=\Ssub$    
     iff $\emptyset\neq D^\Imc \subset C^\Imc$,
    \item $\eta_\Imc(D)[i_C]:=\Ssup$
    iff $ D^\Imc \supset C^\Imc \neq \emptyset$,
        \item $\eta_\Imc(D)[i_C]:=\Snsup$
    iff $\emptyset\neq C^\Imc$, $\emptyset\neq   D^\Imc$, and $\emptyset = C^\Imc\cap   D^\Imc$,
   % \item $\eta_\Imc(D)[i_C]:=\Snsup$ iff $\emptyset\neq C^\Imc$, $\emptyset\neq   D^\Imc$, and $C^\Imc \subset \Delta^\Imc \setminus D^\Imc$,
   % \item $\eta_\Imc(D)[i_C]:=\Sneq$ iff $\emptyset\neq C^\Imc$, $\emptyset\neq   D^\Imc$, and $C^\Imc = \Delta^\Imc \setminus D^\Imc$,
    %
    \item  else $\eta_\Imc(D)[i_C]:=\Sinter$; 
    %if 
    %$C^\Imc\not\subseteq D^\Imc$,     $D^\Imc\not\subseteq C^\Imc$,  $\emptyset \neq C^\Imc\cap   D^\Imc$
        \end{itemize}
    \item for every pair $(R,\element)$ with $R\in \NR$ and $\element\in \Delta^\Imc$, $\eta_\Imc(D)[i_{R,\element}]:=\SC$. %\todo{Ana: we have a problem here because the concepts of the form $\exists R$ are getting values that do not match with the role head, tail, bump box def}
\end{itemize}
For each $S\in \NR$, let $\eta_\Imc(S)$ be the     boxes \headAssignI{S}{\Imc}, \tailAssignI{S}{\Imc}, \bumpBoxAssignI{S}{\Imc}:
 \begin{itemize}
    \item for every $C\in \NCE$,  
    %\begin{itemize}
     %   \item 
        $\headAssignI{S}{\Imc}[i_C]:=\eta_\Imc(\exists S)[i_C]$,
    %    \item 
        $\tailAssignI{S}{\Imc}[i_C]:=\eta_\Imc(\exists S^-)[i_C]$,
   %     \item 
        $\bumpBoxAssignI{S}{\Imc}[i_C]:=\BB$;
  %  \end{itemize}
    \item for every pair $(R,\element)$ with $R\in \NR$ and $\element\in \Delta^\Imc$, 
    \begin{itemize}
        \item $\headAssignI{S}{\Imc}[i_{R,\element}]:=\Qasub$ and $\bumpBoxAssignI{S}{\Imc}[i_{R,\element}]:=\BBRa$
        %(-1-\epsilon,1+\epsilon)$ 
        if for all $ \anotherelement \in \Delta^{\Imc}$ we have that $(c,\anotherelement) \in S^{\Imc}$ implies $(c,\anotherelement) \in R^{\Imc}$, otherwise $\headAssignI{S}{\Imc}[i_{R,\element}]:=\Qnasub$ and $\bumpBoxAssignI{S}{\Imc}[i_{R,\element}]:=\BBRna$,
        %(-1-\epsilon,\epsilon)$,
        \item $\tailAssignI{S}{\Imc}[i_{R,\element}]:=\Qasub$ %changed here because technically above we associate a head box to each role name
        if for all $ \anotherelement \in \Delta^{\Imc}$ we have that $(\anotherelement,c) \in S^{\Imc}$ implies %$(\anotherelement,c) \in R^{\Imc}$
          {$(c,\anotherelement) \in R^{\Imc}$}, otherwise $\tailAssignI{S}{\Imc}[i_{R,\element}]:=\Qnasub$.
    \end{itemize}
 \end{itemize}
%
%\todo{Ana: WIP}
    %mapping WIP
\end{definition}

%\todo[inline]{B: In view of the definitions in Section~\ref{sec:basicDef} should we use term ``Knowledge Base'' instead?}

%\begin{lemma}\label{lem:mappingaxioms}
 %   Given an interpretation \Imc with finite domain, let $\eta_\Imc$ be the box interpretation as in Definition~\ref{def:mapping}.
  %  Then, %the following holds:
    %\begin{itemize}
     %   \item for all $\alpha$
      %  \todo{add all cases except role disjointness}
    %\end{itemize}
%    for all DL-Lite$^\Hmc$ axioms $\alpha$, we have that
 %   $\Imc\models\alpha$ iff $\eta_\Imc\models\alpha$.
%\end{lemma}
%\todo{if we add negative assertions, change here}

%\cref{f}

We first argue that $\eta_\Imc$ is well-defined. That is,
$\eta_\Imc$ is indeed a box consistent interpretation. We also show that it satisfies additional properties that are used in other proofs.
\begin{restatable}{theorem}{thmpropfaithfulness}\label{thm:propembedding} %\todo{Ana: either not mention ontology or use the canonical model?}
%Let $\ontoo$ be a DL-Lite$^\Hmc$ KB, $\Imc$ be a model of $\ontoo$ 
{Let \Imc be an interpretation 
with finite domain}, and let $\epsilon$ be an arbitrary value with $0 < \epsilon \leq \epsilonMax$. %(with $\epsilonMax = 0.5$ see Section~\ref{sec:boxSemantics}). 
Then, $\cAssignSymb_{\Imc}$ constructed from $\Imc$, as  in \cref{def:mapping},
is a \geometric{} interpretation that satisfies the  following additional  property.
%the proof of \cref{thm:faith_bool_alc}.
%by Definition~\ref{def:mapping}}:
   \begin{itemize}
       \item For any concept $C \in \NCE$, it holds that:
    \begin{align*}
     \frac{\mathbf{L_C}[i_C] + \mathbf{U_C}[i_C]}{2} \leq - \frac{\worldSizeScalar}{2},
     \label{eq:boxConsistencyInequality}
     \end{align*}
    %\item 
    which implies that $\cAssignSymb_{\Imc}$ is box consistent.
   \end{itemize}
\end{restatable}
%\thmpropfaithfulness*
\begin{proof}
% Let $\epsilonMax = 0.5$ 
%Let $\ontoo$ be a DL-Lite$^\Hmc$ KB, $\Imc$ be a model of $\ontoo$, 
{Let \Imc be an interpretation with finite domain and let }
$\epsilon$ be an arbitrary value with $0 < \epsilon \leq \epsilonMax$, and recall that in Section~\ref{sec:boxSemantics} we assume  that $\epsilonMax$ is bounded by $0.5$ throughout this paper. 
Also, $\World[i_C]=(-4,4)$ (\cref{tab:RoleBoxes}) implies
    $s_\Omega=4$.
We first prove that $\eta_\Imc$ as described in Definition~\ref{def:mapping} is well-defined, i.e., that it is a box interpretation (c.f. \cref{def:boxinterpretation}).
We start by proving the conditions related to individual names in \cref{def:boxinterpretation}, namely
\begin{itemize}
    \item 
each individual name $a\in\NI$ is mapped to two vectors $\eta(a) = (\posAssign{a}, \bumpAssign{a})$, namely, a position $\posAssign{e} \in \World$ and a bump $\bumpAssign{e} \in \World$.
\end{itemize}
Indeed \cref{def:mapping} maps each $a\in\NI$ to two vectors:
$\posAssignI{\Imc}{a}$ and $ \bumpAssignI{\Imc}{a}$. To complete this item, we need to show that
$\posAssignI{\Imc}{a},\bumpAssignI{\Imc}{a}\in \World$, which we do in \cref{clam:saga}.
\begin{claim}\label{clam:saga}
For any individual $a \in \NI$: 
    \begin{align*}
     -\worldSize + \epsilonVec \leqd \posAssignI{\Imc}{a} \leqd \worldSize - \epsilonVec
     %\label{eq:indPosConstrain}
     \\
     -\worldSize + \epsilonVec \leqd \bumpAssignI{\Imc}{a} \leqd \worldSize - \epsilonVec
     %\label{eq:indBumpConstrain}
    \end{align*}
    \end{claim}
 \begin{proof}
     For any   $a \in \NI$ and for any dimension $i$ with $0 \leq i \leq d$, Definition \ref{def:mapping} assigns (i) $\posAssignI{\Imc}{a}[i]$ either to $\PC, \PnC, \PRa,$ or $\PRna$. Furthermore, by Tables~\ref{tab:ConceptBoxes}, \ref{tab:ConceptIndividual}, and \ref{tab:RoleIndividual} it holds that (ii) $-\worldSizeScalar + \epsilonMax \leq \PC, \PnC, \PRa, \PRna \leq \worldSizeScalar - \epsilonMax$. By  (i) and (ii) for all $a\in \NI$ it holds that $-\worldSize + \epsilonVec \leqd \posAssignI{\Imc}{a} \leqd \worldSize - \epsilonVec$.
     %, proving (\ref{eq:indPosConstrain}).
    
     For any   $a \in \NI$ and for any dimension $i$ with $0 \leq i \leq d$, Definition \ref{def:mapping} assigns (i) $\bumpAssignI{\Imc}{a}[i]$ either to $\BRa$ or $\BRna$. 
    Also, by Tables~\ref{tab:ConceptIndividual}, \ref{tab:RoleBoxes}, and \ref{tab:RoleIndividual} it holds that (ii) $-\worldSizeScalar + \epsilonMax \leq \BRa, \BRna \leq \worldSizeScalar - \epsilonMax$. By (i)  and (ii), for all $a\in \NI$,  $-\worldSize + \epsilonVec \leqd \bumpAssignI{\Imc}{a} \leqd \worldSize - \epsilonVec$.
    %, proving (\ref{eq:indBumpConstrain}).
 \end{proof}   
We now proceed with the second item of \cref{def:boxinterpretation}.
 \begin{itemize}
\item Each concept name $A\in\NC$ is mapped to $\cAssign{A} \in \B$.
\end{itemize}
To show that $\cAssign{A} \in \B$ we need  \cref{clm:sagacont} to hold.  
\begin{claim}\label{clm:sagacont}
    For any concept $C \in \NC$, it holds that:
    \begin{align*}
     \mathbf{0} \leqd \mathbf{U_C} - \mathbf{L_C} \leqd 2\worldSize   
    % \label{eq:conceptConstrain}
    \end{align*}
    (in fact  this holds for all $C \in \NCE$).
\end{claim}
\begin{proof}
    For any concept  $C \in \NCE$ and for any dimension $i$ with $0 \leq i \leq d$,   \cref{def:mapping} assigns   $\cAssignI{\Imc}{C}[i]$ either to $\Sempty, \Seq, \Ssub, \Ssup, \Snsup, \Sinter, \text{ or }\SC$. 
    Then, by Tables~\ref{tab:ConceptBoxes} and \ref{tab:RoleBoxes}, 
    for $(\mathbf{L_C}[i],\mathbf{U_C}[i]):=\cAssignI{\Imc}{C}[i]$ it holds that 
    %for any box $X \in \{\Sempty, \Seq, \Ssub, \Ssup, \Snsup, \Sinter, \SC %, \SnC
    %\}$ if $\mathbf{L_X}$ and $\mathbf{U_X}$ denote their respective bounds then (2) 
    $0 \leq \mathbf{U_C}[i]-\mathbf{L_C}[i]  \leq 2\worldSizeScalar$. Thus, $\mathbf{0} \leqd \mathbf{U_C} - \mathbf{L_C} \leqd 2\worldSize$. 
\end{proof}
%Based on these assumptions, we prove that each of the properties in   \cref{cor:ConsistencyCorollary} are satisfied for  
%{any \geometric{} interpretation 
%$\cAssignSymb_{\Imc}$ constructed from $\Imc$, as described in Definition~\ref{def:mapping}}.
The third item of \cref{def:boxinterpretation} is:
\begin{itemize}
\item each role name  $R\in\NR$ is mapped to three boxes  $\cAssign{R} = (\headAssign{R}, \tailAssign{R}, \bumpBoxAssign{R}$), which we call $R$'s {head} $\headAssign{R} $, {tail} $\tailAssign{R} $ and {bump box} $\bumpBoxAssign{R} $. %\todo{Question to Ana: Is there an equivalent to top and bottom for roles?}
\end{itemize} 
We show this in \cref{clm:whatasaga}.
\begin{claim}\label{clm:whatasaga}
    For any role $R \in \NR$, it holds that: 
    \begin{align*}
        \mathbf{0} \leqd \mathbf{U}^X_\mathbf{R} - \mathbf{L}^X_\mathbf{R} \leqd 2\worldSize \; {\sf with} \; X \in \{\mathbf{H}, \mathbf{T}, \mathbf{B}\}.
    % \label{eq:roleConstrain}
    \end{align*}
\end{claim}
\begin{proof}
For any   $R \in \NR$ and for any dimension $i$ with $0 \leq i \leq d$, Definition \ref{def:mapping} assigns   $\headAssignI{R}{\Imc}[i]$ and $\tailAssignI{R}{\Imc}[i]$ either to $\Qasub, \Qnasub, {\Sempty,} \Seq, \Ssub, \Ssup, \Snsup,$ or $\Sinter$. Then, by Tables~\ref{tab:ConceptBoxes} and \ref{tab:RoleBoxes}, 
    for $(\mathbf{L^H_R}[i],\mathbf{U^H_R}[i]):=\headAssignI{R}{\Imc}[i]$ it holds that 
    %for any box $X \in \{\Sempty, \Seq, \Ssub, \Ssup, \Snsup, \Sinter, \SC %, \SnC
    %\}$ if $\mathbf{L_X}$ and $\mathbf{U_X}$ denote their respective bounds then (2) 
    $0 \leq \mathbf{U^H_R}[i]-\mathbf{L^H_R}[i]  \leq 2\worldSizeScalar$, and
    for $(\mathbf{L^T_R}[i],\mathbf{U^T_R}[i]):=\tailAssignI{R}{\Imc}[i]$ it holds that 
    %for any box $X \in \{\Sempty, \Seq, \Ssub, \Ssup, \Snsup, \Sinter, \SC %, \SnC
    %\}$ if $\mathbf{L_X}$ and $\mathbf{U_X}$ denote their respective bounds then (2) 
    $0 \leq \mathbf{U^T_R}[i]-\mathbf{L^T_R}[i]  \leq 2\worldSizeScalar$.
    %, proving (\ref{eq:roleConstrain}) for the head and tail boxes.
    %it holds that (2) for any box $\beta \in \{\Qasub, \Qnasub, \Seq, \Ssub, \Ssup, \Snsup, \Sinter\}$ if $L_\beta$ and $U_\beta$ denote their respective bounds then (2) $0 \leq U_\beta - L_\beta \leq 2\worldSizeScalar$. Thus, by (1) and (2) for any    $R \in \NR$ it holds that $\mathbf{0} \leqd \mathbf{U^x_R} - \mathbf{L^x_R} \leqd 2\worldSize$ with $x \in \{H, T\}$, proving (\ref{eq:roleConstrain}) with $x \in \{H, T\}$. 
    
      For any   $R \in \NR$ and for any dimension $i$ with $0 \leq i \leq d$, Definition \ref{def:mapping} assigns   $\bumpAssignI{\Imc}{R}[i]$ either to $\BB, \BBRa$, or $\BBRna$. Then, by Tables~\ref{tab:ConceptBoxes} and \ref{tab:RoleBoxes},
    for $(\mathbf{L^B_R}[i],\mathbf{U^B_R}[i]):=\bumpAssignI{\Imc}{R}[i]$ it holds that 
    $0 \leq \mathbf{U^B_R}[i]-\mathbf{L^B_R}[i]  \leq 2\worldSizeScalar$.
    %, proving (\ref{eq:roleConstrain}) for the bump box.     
\end{proof}  
 The extension of $\eta_\Imc$ to arbitrary DL-Lite$^\Hmc$ concept and role expressions  
is as  in \cref{def:boxinterpretation}, though, for presentation purposes we did not explicitly define
the head and tail boxes for the dimension
$i_C$ (this was defined in terms of $\eta_\Imc(\exists S)$ and $\eta_\Imc(\exists S^-)$). So here we need to argue that $\eta_\Imc(\exists S)$ is equal to
 \[\{\mathbf{x} \in \Rd \mid \mathbf{L^H_S} - \mathbf{U^B_S} + \epsilonVec \leqd \mathbf{x} \leqd \mathbf{U^H_S}  - \mathbf{L^B_S} - \epsilonVec \}\]
 (recall that $\mathbf{L_{\exists S}}=\mathbf{L^H_S} - \mathbf{U^B_S}$
 and $\mathbf{U_{\exists S}}=\mathbf{U^H_S}  - \mathbf{L^B_S}$
  see \cref{def:boxinterpretation}).
Indeed,  by the values in
Table~\ref{tab:RoleBoxes}, for the dimensions of the form $i_{R,c}$ we have that
$\mathbf{L^H_S}[i_{R,c}] - \mathbf{U^B_S}[i_{R,c}]=-2$ and
$\mathbf{U^H_S}[i_{R,c}]  - \mathbf{L^B_S}[i_{R,c}]=2$, which correspond to $\SC$, the value of $\eta_\Imc(\exists S)[i_{R,c}]$. For the dimensions of the form $i_{C}$, $\mathbf{L^H_S}[i_{C}] - \mathbf{U^B_S}[i_{C}]=\mathbf{L^H_S}[i_{C}]$ and
$\mathbf{U^H_S}[i_{C}]  - \mathbf{L^B_S}[i_{C}]=\mathbf{U^H_S}[i_{C}] $ since in this case $\bumpBoxAssignI{S}{\Imc}[i_C]:=\BB$, which is $(0,0)$, by the values in
Table~\ref{tab:ConceptBoxes}. This is as required as in this case we have $\headAssignI{S}{\Imc}[i_C]=\eta_\Imc(\exists S)[i_C]$.
The argument that
 $\eta_\Imc(\exists S^-)$ 
is equal to
\[\{\mathbf{x} \in \Rd \mid \mathbf{L^T_S} - \mathbf{U^B_S} + \epsilonVec \leqd \mathbf{x} \leqd \mathbf{U^T_S}  - \mathbf{L^B_S} - \epsilonVec \}\]
is similar. 
It remains to argue the following.
\begin{claim}    
For any concept $C \in \NCE$, it holds that:
    \begin{align}
     \frac{\mathbf{L_C}[i_C] + \mathbf{U_C}[i_C]}{2} \leq - \frac{\worldSizeScalar}{2}
     \label{eq:boxConsistencyInequality}
     \end{align}
    and that this implies $\cAssignSymb_{\Imc}$ is box consistent.
   \end{claim}

\begin{proof}
     {For any concept $C \in \NCE$ and dimension $i_C$, (i) Definition \ref{def:mapping} assigns $\cAssignI{\Imc}{C}[i_C] = \Seq$. Then, by Table~\ref{tab:ConceptBoxes} it holds that (ii) the lower bound of $\Seq$ is set to $-4$ and the upper bound of $\Seq$ is set to $-0.5$. By (i) and (ii), we have that (iii) $\frac{\mathbf{L_{C}}[i_C] + \mathbf{U_{C}}[i_C]}{2} = -2.25 \leq -\worldSizeScalar/2 = -2$. From (i), (ii), and (iii),   for any concept $C \in \NCE$, it holds that $\frac{\mathbf{L_C}[i_C] + \mathbf{U_C}[i_C]}{2} \leq - \frac{\worldSizeScalar}{2}$, proving (Equation~\eqref{eq:boxConsistencyInequality}).
     }
     
      Finally, we briefly check that (\ref{eq:boxConsistencyInequality}) guarantees box consistency. 
    % In particular, for any concept $C \in \NCE$ it holds --- by (\ref{eq:boxConsistencyInequality}) --- that $\frac{\mathbf{L_C}[i_C] + \mathbf{U_C}[i_C]}{2} \leq - \frac{\worldSizeScalar}{2}$. Now
    %Next, we check briefly that \eqref{eq:concept_cons} ensures box consistency. 
    Let $C$ be an arbitrary concept $C \in \NCE$ and assume that  (\ref{eq:boxConsistencyInequality}) holds. Next, recall that $\mathbf{L_{\neg C}} = - \worldSize - \mathbf{L_{C}}$ holds (see Section~\ref{sec:ProblemFormulation}). Now we have that:
\begin{align}
    \mathbf{U_{C}}[i_C] &\leq-\worldSizeScalar -\mathbf{L_{C}}[i_C] = 
    \mathbf{L_{\neg C}}[i_C]
    \label{eq:cons2}  
\end{align}

In particular, 
\[
\mathbf{U_{C}}[i_C] - \epsilon < \mathbf{L_{\neg C}}[i_C]+\epsilon,
\]
which means that the boxes defined by $C$ and $\neg C$ do not intersect in dimension $i_C$. This implies that $\cAssignI{\Imc}{C} \cap \cAssignI{\Imc}{\neg C} \neq \emptyset$, proving that $\cAssignSymb_\Imc$ is box consistent.
\end{proof}
This finishes the proof of this theorem.
\end{proof}

%\begin{lemma}
 %   Let \Imc be an interpretation with finite domain and let
  %  $\eta_\Imc$ be as in \cref{def:canonicalModel}. Then
   % $\eta_\Imc$ is  a box consistent interpretation  (\cref{sec:boxSemantics}).
%\end{lemma}
%\begin{proof}
 
%\end{proof}

\cref{fig:ConceptBoxes} and \cref{fig:RoleBoxes}
help to visualize how the constants 
%$\BC$, $\BRa$, $\Ssub$, etc., that are 
defined in Tables~\ref{tab:ConceptBoxes}, \ref{tab:ConceptIndividual}, \ref{tab:RoleBoxes}, and \ref{tab:RoleIndividual} relate to each other. In the next lemmas, we argue that $\Imc\models \alpha$ iff $\eta_\Imc\models\alpha$
when $\alpha$ is a concept assertion, role assertion, concept inclusion, or role inclusion.

\begin{lemma}\label{lem:faithfulnessconceptassertions}Given an interpretation \Imc with finite domain, let 
$\eta_\Imc$ be as in Definition~\ref{def:mapping}.
For all $a\in \NI$ and all  concepts $D\in \NCE$, $\Imc\models D(a)$ iff $\eta_\Imc\models D(a)$.
\end{lemma}
\begin{proof}
We need to show that $a^\Imc\in D^\Imc$ iff $\posAssignI{\Imc}{a}\in \eta_\Imc(D)$. 
($\Rightarrow$) Suppose $a^\Imc\in D^\Imc$.
    To show that $\posAssignI{\Imc}{a}\in \eta_\Imc(D)$, we argue that $\posAssignI{\Imc}{a}[j]\in \eta_\Imc(D)[j]$,  for every dimension $1\leq j\leq d$.
    We make a case distinction with $C\in\NCE$.
    \begin{itemize}
        \item \textbf{Dimension $i_C$ with $a^\Imc \in C^{\Imc}$.} 
        In this case, (i)  $\posAssignI{\Imc}{a}[i_C] = \PC$. By the parameters of the boxes of dimension $i_C$ (see Tables~\ref{tab:ConceptBoxes} and \ref{tab:ConceptIndividual}, and Figure~\ref{fig:ConceptBoxes}), it holds that (ii) $\PC \in \Seq \cap \Ssup \cap \Ssub \cap \Sinter$. By Definition~\ref{def:mapping},  $\cAssignSymb_{\Imc}(D)[i_C]$ is either equal to $\Seq$, $\Ssup$, $\Ssub$, $\Sinter$, or $\Snsup$  (since by the ($\Rightarrow$) assumption $D^\Imc\neq\emptyset$). By the ($\Rightarrow$) assumption,     $a^\Imc \in D^{\Imc}$ and by the assumption of this case    $a^\Imc \in C^{\Imc}$. By Definition~\ref{def:mapping}, we have that $\eta_\Imc(D)[i_C]$ is not  $\Snsup$   because 
        %they require $C^\Imc\cap D^\Imc=\emptyset$ and 
        $a^\Imc\in C^\Imc\cap D^\Imc$; so it holds that (iii) $\eta_\Imc(D)[i_C] \neq \Snsup$. By (i)-(iii) it holds that $\posAssignI{\Imc}{a}[i_C]\in \eta_\Imc(D)[i_C]$.
        \item \textbf{Dimension $i_C$ with $a^\Imc \not\in C^{\Imc}$.}  In this case, (i) $\posAssignI{\Imc}{a}[i_C] = \PnC$. By the parameters of the   boxes of dimension $i_C$ (see Tables~\ref{tab:ConceptBoxes} and \ref{tab:ConceptIndividual}, and Figure~\ref{fig:ConceptBoxes}), it holds that (ii) $\PnC \in \Snsup \cap \Ssup \cap \Sinter$. By Definition~\ref{def:mapping},  $\cAssignSymb_{\Imc}(D)[i_C]$ is either equal to   $\Seq$, $\Ssup$, $\Ssub$, $\Sinter$, or $\Snsup$,  (since by the ($\Rightarrow$) assumption $D^\Imc\neq\emptyset$).  By the ($\Rightarrow$) assumption,     $a^\Imc \in D^{\Imc}$ and by the assumption of this case    $a^\Imc \not\in C^{\Imc}$. Also,  by Definition~\ref{def:mapping} 
        we have that $\eta_\Imc(D)[i_C]$ is neither   $\Seq$ nor $\Ssub$ because the former requires $C^\Imc= D^\Imc$, the latter requires $D^\Imc\subset C^\Imc $, but in the case we consider here we have  $a^\Imc\in D^\Imc\setminus C^\Imc$;
        so   (iii) $\eta_\Imc(D)[i_C] \neq \Seq$ and   (iv) $\eta_\Imc(D)[i_C] \neq \Ssub$.
         By (i)-(iv) it holds that $\posAssignI{\Imc}{a}[i_C]\in \eta_\Imc(D)[i_C]$.
        %\item \textbf{Dimension $i_{R,c}$.}
    \end{itemize} 
    Finally, consider the dimensions of the form $i_{R,c}$, where $R\in\NR$ and $c\in \Delta^\Imc$. In this case (i) $\eta_\Imc(D)[i_{R,\element}] = \SC$ for any $R\in\NR$ and $c\in \Delta^\Imc$. Furthermore, for any individual name in \NI, in particular $a$, we have that  (ii) $\posAssignI{\Imc}{a}[i_{R,c}] = \PRa$ or $\posAssignI{\Imc}{a}[i_{R,c}] = \PRna$. By the parameters of $\SC$, $\PRa$, and $\PRna$ (see Tables~\ref{tab:RoleBoxes} and \ref{tab:RoleIndividual}, and Figure~\ref{fig:RoleBoxes}), it holds that (iii) $\PRa \in \SC$ and $\PRna \in \SC$. By (i)-(iii) it holds that $\posAssignI{\Imc}{a}[i_{R,c}]\in \eta_\Imc(D)[i_{R,c}]$  { provided that $\epsilon\leq \epsilonMax$ (see Table~\ref{tab:RoleIndividual} and Section~\ref{sec:boxSemantics})}.
We have   shown  that $\posAssignI{\Imc}{a}[j]\in \eta_\Imc(D)[j]$ for any dimension $1\leq j\leq d$. Thus, we have   shown that $\posAssignI{\Imc}{a}\in \eta_\Imc(D)$ if $a^\Imc \in D^{\Imc}$.
    
($\Leftarrow$) Now suppose $\posAssignI{\Imc}{a}\in \eta_\Imc(D)$. 
  %Then, by Definition \ref{def:BoxSatisfaction}, it holds that $\posAssign{a} \in \cAssignSymb_{\Imc_{\ontooC}}(D)$. 
  This means, for dimension $i_D$ in particular, that  (i)  $\posAssignI{\Imc}{a}[i_D]\in \eta_\Imc(D)[i_D]$. By the construction of dimension $i_D$ of $\eta_\Imc$ it holds that (ii) $\eta_\Imc(D)[i_D] = \Seq$. From (i) and (ii),   $\posAssignI{\Imc}{a}[i_D] = \PC$.  By the construction of dimension $i_D$    $\posAssignI{\Imc}{a}[i_D] = \PC$ can only be if $a^\Imc \in D^{\Imc}$. 
  %From (2) and (3) follows $a \in D^{\Imc}$ which is what we wanted to show. 
%    \begin{itemize}
 %       \item $C=A$ with $A\in\NC$:
  %      suppose $a^\Imc\in C^\Imc$. 
   %     \item $C=\exists R$ with $R\in\NR$: 
    %    \item $C=\exists R^-$ with $R\in\NR$:
    %\end{itemize}
\end{proof}

%Ana WIP trying assertions first
\begin{lemma}\label{lem:faithfulnessroleassertions}Given an interpretation \Imc with finite domain, let 
$\eta_\Imc$ be as in Definition~\ref{def:mapping}.
For all $a,b\in \NI$ and all   $R\in \NR$, $\Imc\models R(a,b)$ iff $\eta_\Imc\models R(a,b)$.
\end{lemma}
\begin{proof}($\Rightarrow$) Assume $\Imc \models R(a,b)$. 
    We start by showing that $\posAssignI{\Imc}{a} + \bumpAssignI{\Imc}{b}  \in \headAssignI{R}{\Imc}$
    by showing that this holds
    in every dimension.
    For this we make a case distinction, where $C\in\NCE$, $S\in\NR$, and $c\in\Delta^\Imc$.
    \begin{itemize}
        \item \textbf{Dimension $i_C$.}
        We want to show that $\posAssignI{\Imc}{a}[i_C] + \bumpAssignI{\Imc}{b}[i_C]  \in \headAssignI{R}{\Imc}[i_C]$. 
            By Definition~\ref{def:mapping}, $\headAssignI{R}{\Imc}[i_C]=\eta_\Imc(\exists R)[i_C]$ can be either equal to 
            $\Seq$, $\Ssup$, $\Ssub$, $\Sinter$, or $\Snsup$  (since by the ($\Rightarrow$) assumption $\Imc\models R(a,b)$,
            so $(\exists R)^\Imc\neq\emptyset$).
            By Definition~\ref{def:mapping}, $\posAssignI{\Imc}{a}[i_C]=\PC$ if $a^\Imc\in C^\Imc$, and $\posAssignI{\Imc}{a}[i_C]=\PnC$ if $a^\Imc\notin C^\Imc$; also $\bumpAssignI{\Imc}{a}[i_C]=\BC$ (if $a^\Imc\in C^\Imc$ or not).
         By the values in 
        Tables~\ref{tab:ConceptBoxes} and \ref{tab:ConceptIndividual},
we have that all possible values for $\posAssignI{\Imc}{a}[i_C] + \bumpAssignI{\Imc}{b}[i_C]$ are in $ \Ssup\cap \Sinter$. So we need to consider three cases, namely,
(i) when $\headAssignI{R}{\Imc}[i_C]= \Seq$, (ii) when $\headAssignI{R}{\Imc}[i_C]= \Ssub$, and (iii)  when $\headAssignI{R}{\Imc}[i_C]= \Snsup$. By the ($\Rightarrow$) assumption,
$\Imc\models R(a,b)$, so $a^\Imc\in(\exists R)^\Imc$.
In case (i) and in case (ii), 
by Definition~\ref{def:mapping}, we have that $(\exists R)^\Imc=C^\Imc$ and $(\exists R)^\Imc\subset C^\Imc$, respectively.
So, in both cases $a^\Imc\in C^\Imc$ and, by Definition~\ref{def:mapping}, $\posAssignI{\Imc}{a}[i_C]=\PC$.
By the values in 
        Tables~\ref{tab:ConceptBoxes} and \ref{tab:ConceptIndividual}, $\posAssignI{\Imc}{a}[i_C] + \bumpAssignI{\Imc}{b}[i_C]\in \Seq\cap\Ssub$.
In case (iii), $(\exists R)^\Imc\cap C^\Imc=\emptyset$. So $a^\Imc\notin C^\Imc$ and, by Definition~\ref{def:mapping}, $\posAssignI{\Imc}{a}[i_C]=\PnC$. By the values in 
        Tables~\ref{tab:ConceptBoxes} and \ref{tab:ConceptIndividual}, $\posAssignI{\Imc}{a}[i_C] + \bumpAssignI{\Imc}{b}[i_C]\in \Snsup$. %, and we are done.
        \item \textbf{Dimension $i_{S,c}$ with   $a^\Imc=c$ and for all $ \anotherelement \in \Delta^{\Imc}$, $(c,\anotherelement) \in R^{\Imc}$ implies $(c,\anotherelement) \in S^{\Imc}$.} 
        %We want to show that $\posAssignI{\Imc}{a}[i_{S,c}] + \bumpAssignI{\Imc}{b}[i_{S,c}]  \in \headAssignI{R}{\Imc}[i_{S,c}]$. 
        By Definition~\ref{def:mapping}, when $a^\Imc=c$, we have that
        $\posAssignI{\Imc}{a}[i_{S,c}]=\PRa$ and $\bumpAssignI{\Imc}{b}[i_{S,c}]=\BRa$. Also, in this case 
        $\headAssignI{R}{\Imc}[i_{S,c}]$ is $\Qasub$.
        %Assume for all $ \anotherelement \in \Delta^{\Imc}$, $(c,\anotherelement) \in R^{\Imc}$ implies $(c,\anotherelement) \in S^{\Imc}$
        By the values in Tables~\ref{tab:RoleBoxes} and~\ref{tab:RoleIndividual},
        $\posAssignI{\Imc}{a}[i_{S,c}] + \bumpAssignI{\Imc}{b}[i_{S,c}]  \in \headAssignI{R}{\Imc}[i_{S,c}]$.
        %, as required.
        \item \textbf{Dimension $i_{S,c}$ with   $a^\Imc=c$ and there is $ \anotherelement \in \Delta^{\Imc}$, with $(c,\anotherelement) \in R^{\Imc}$ but $(c,\anotherelement) \notin S^{\Imc}$.} 
        %We want to show that $\posAssignI{\Imc}{a}[i_{S,c}] + \bumpAssignI{\Imc}{b}[i_{S,c}]  \in \headAssignI{R}{\Imc}[i_{S,c}]$. 
        By Definition~\ref{def:mapping}, when $a^\Imc=c$, we have that
        $\posAssignI{\Imc}{a}[i_{S,c}]=\PRa$ and 
        $\headAssignI{R}{\Imc}[i_{S,c}]=\Qnasub$.
         Also, $\bumpAssignI{\Imc}{b}[i_{S,c}]=\BRa$ or $\bumpAssignI{\Imc}{b}[i_{S,c}]=\BRna$.
        %Assume for all $ \anotherelement \in \Delta^{\Imc}$, $(c,\anotherelement) \in R^{\Imc}$ implies $(c,\anotherelement) \in S^{\Imc}$
        In both cases, by the values in Tables~\ref{tab:RoleBoxes} and~\ref{tab:RoleIndividual},
        $\posAssignI{\Imc}{a}[i_{S,c}] + \bumpAssignI{\Imc}{b}[i_{S,c}]  \in \headAssignI{R}{\Imc}[i_{S,c}]$.
        \item \textbf{Dimension $i_{S,c}$ with   $a^\Imc\neq c$.} In this case, by Definition~\ref{def:mapping},
        $\posAssignI{\Imc}{a}[i_{S,c}]=\PRna$. It can be that $\headAssignI{R}{\Imc}[i_{S,c}]=\Qasub$ or $\headAssignI{R}{\Imc}[i_{S,c}]=\Qnasub$.
        Also, it can be that
        $\bumpAssignI{\Imc}{b}[i_{S,c}]=\BRna$ or $\bumpAssignI{\Imc}{b}[i_{S,c}]=\BRa$ (depending
        on \Imc) but in all cases, by the values in Tables~\ref{tab:RoleBoxes} and~\ref{tab:RoleIndividual},
        $\posAssignI{\Imc}{a}[i_{S,c}] + \bumpAssignI{\Imc}{b}[i_{S,c}]  \in \headAssignI{R}{\Imc}[i_{S,c}]$, as required.
        %\item \textbf{Dimension $i_{S,c}$ with   $a^\Imc=c$ and, there is $ \anotherelement \in \Delta^{\Imc}$, with $(c,\anotherelement) \in R^{\Imc}$ but $(c,\anotherelement) \notin S^{\Imc}$.} In this case, by Definition~\ref{def:mapping},
        %$\posAssignI{\Imc}{a}[i_{S,c}]=\PRa$, $\bumpAssignI{\Imc}{b}[i_{S,c}]=\BRa$ and
        %$\headAssignI{R}{\Imc}[i_{S,c}]=\Qnasub$. Then, by the values in Tables~\ref{tab:RoleBoxes} and~\ref{tab:RoleIndividual},
        %$\posAssignI{\Imc}{a}[i_{S,c}] + \bumpAssignI{\Imc}{b}[i_{S,c}]  \in \headAssignI{R}{\Imc}[i_{S,c}]$, as required.
        %\todo{Ana: WIP}
        %\item \textbf{Dimension $i_{S,c}$ with none of the above.}
    \end{itemize}
We now show that
$\posAssignI{\Imc}{b}+ \bumpAssignI{\Imc}{a} \in \tailAssignI{R}{\Imc}$
    by showing that this holds
    in every dimension.
    We   make a  similar case distinction as in the argument above. %where $C\in\NCE$, $S\in\NR$, and $c\in\Delta^\Imc$.
    \begin{itemize}
        \item \textbf{Dimension $i_C$.}
        We want to show that $\posAssignI{\Imc}{b}[i_C] + \bumpAssignI{\Imc}{a}[i_C]  \in \tailAssignI{R}{\Imc}[i_C]$. 
            By Definition~\ref{def:mapping}, $\tailAssignI{R}{\Imc}[i_C]=\eta_\Imc(\exists R^-)[i_C]$ can be either equal to 
            $\Seq$, $\Ssup$, $\Ssub$, $\Sinter$, or $\Snsup$  (since by the ($\Rightarrow$) assumption $\Imc\models R(a,b)$,
            so $(\exists R^-)^\Imc\neq\emptyset$).
         By Definition~\ref{def:mapping}, $\posAssignI{\Imc}{b}[i_C]=\PC$ if $b^\Imc\in C^\Imc$, and $\posAssignI{\Imc}{b}[i_C]=\PnC$ if $b^\Imc\notin C^\Imc$; also $\bumpAssignI{\Imc}{b}[i_C]=\BC$ (if $b^\Imc\in C^\Imc$ or not).
         By the values in 
        Tables~\ref{tab:ConceptBoxes} and \ref{tab:ConceptIndividual},
we have that all possible values for $\posAssignI{\Imc}{b}[i_C] + \bumpAssignI{\Imc}{a}[i_C]$ are in $ \Ssup\cap \Sinter$. So we need to consider three cases, namely,
(i) when $\tailAssignI{R}{\Imc}[i_C]= \Seq$, (ii) when $\tailAssignI{R}{\Imc}[i_C]= \Ssub$, and (iii)  when $\tailAssignI{R}{\Imc}[i_C]= \Snsup$. By the ($\Rightarrow$) assumption,
$\Imc\models R(a,b)$, so $b^\Imc\in(\exists R^-)^\Imc$.
In case (i) and in case (ii), 
by Definition~\ref{def:mapping}, we have that $(\exists R^-)^\Imc=C^\Imc$ and $(\exists R^-)^\Imc\subset C^\Imc$, respectively.
So, in both cases $b^\Imc\in C^\Imc$ and, by Definition~\ref{def:mapping}, $\posAssignI{\Imc}{b}[i_C]=\PC$.
By the values in 
        Tables~\ref{tab:ConceptBoxes} and \ref{tab:ConceptIndividual}, $\posAssignI{\Imc}{b}[i_C] + \bumpAssignI{\Imc}{a}[i_C]\in \Seq\cap\Ssub$.
In case (iii), $(\exists R^-)^\Imc\cap C^\Imc=\emptyset$. So $b^\Imc\notin C^\Imc$ and, by Definition~\ref{def:mapping}, $\posAssignI{\Imc}{b}[i_C]=\PnC$. By the values in 
        Tables~\ref{tab:ConceptBoxes} and \ref{tab:ConceptIndividual}, $\posAssignI{\Imc}{b}[i_C] + \bumpAssignI{\Imc}{a}[i_C]\in \Snsup$. %, and we are done.     
        \item \textbf{Dimension $i_{S,c}$ with   $b^\Imc=c$  and for all $ \anotherelement \in \Delta^{\Imc}$, $(\anotherelement,c) \in R^{\Imc}$ implies $(c,\anotherelement) \in S^{\Imc}$.}
        By Definition~\ref{def:mapping}, when $b^\Imc=c$, we have that
        $\posAssignI{\Imc}{b}[i_{S,c}]=\PRa$ and
        %or $\bumpAssignI{\Imc}{a}[i_{S,c}]=\BRna$.
        $\tailAssignI{R}{\Imc}[i_{S,c}]=\Qasub$.
        We also have that $\bumpAssignI{\Imc}{a}[i_{S,c}]=\BRa$ because
        $\Imc\models R(a,b)$ and in our particular case this also implies
        $\Imc\models R(b,a)$.
        %and $\Qnasub$. 
        %Assume for all $ \anotherelement \in \Delta^{\Imc}$, $(c,\anotherelement) \in R^{\Imc}$ implies $(c,\anotherelement) \in S^{\Imc}$
        By the values in Tables~\ref{tab:RoleBoxes} and~\ref{tab:RoleIndividual},
        $\posAssignI{\Imc}{b}[i_{S,c}] + \bumpAssignI{\Imc}{a}[i_{S,c}]  \in \tailAssignI{R}{\Imc}[i_{S,c}]$. %, as required.
        %and, for all $ \anotherelement \in \Delta^{\Imc}$, $(c,\anotherelement) \in R^{\Imc}$ implies $(c,\anotherelement) \in S^{\Imc}$.}
        \item \textbf{Dimension $i_{S,c}$ with   $b^\Imc=c$  and there is $ \anotherelement \in \Delta^{\Imc}$, with$(\anotherelement,c) \in R^{\Imc}$ but $(c,\anotherelement) \notin S^{\Imc}$.}
        By Definition~\ref{def:mapping}, when $b^\Imc=c$, we have that
        $\posAssignI{\Imc}{b}[i_{S,c}]=\PRa$ and
        %or $\bumpAssignI{\Imc}{a}[i_{S,c}]=\BRna$.
        $\tailAssignI{R}{\Imc}[i_{S,c}]=\Qnasub$.  Also, $\bumpAssignI{\Imc}{a}[i_{S,c}]=\BRa$ or $\bumpAssignI{\Imc}{a}[i_{S,c}]=\BRna$.
        %Assume for all $ \anotherelement \in \Delta^{\Imc}$, $(c,\anotherelement) \in R^{\Imc}$ implies $(c,\anotherelement) \in S^{\Imc}$
        In both cases, by the values in Tables~\ref{tab:RoleBoxes} and~\ref{tab:RoleIndividual},
        $\posAssignI{\Imc}{b}[i_{S,c}] + \bumpAssignI{\Imc}{a}[i_{S,c}]  \in \tailAssignI{R}{\Imc}[i_{S,c}]$.
        \item \textbf{Dimension $i_{S,c}$ with   $b^\Imc\neq c$.} In this case, by Definition~\ref{def:mapping},
        $\posAssignI{\Imc}{b}[i_{S,c}]=\PRna$. It can be that $\tailAssignI{R}{\Imc}[i_{S,c}]=\Qasub$ or $\tailAssignI{R}{\Imc}[i_{S,c}]=\Qnasub$.
        Also, it can be that
        $\bumpAssignI{\Imc}{a}[i_{S,c}]=\BRna$ or $\bumpAssignI{\Imc}{a}[i_{S,c}]=\BRa$ (depending
        on \Imc) but in all cases, by the values in Tables~\ref{tab:RoleBoxes} and~\ref{tab:RoleIndividual},
        $\posAssignI{\Imc}{b}[i_{S,c}] + \bumpAssignI{\Imc}{a}[i_{S,c}]  \in \tailAssignI{R}{\Imc}[i_{S,c}]$. %, as required.
    \end{itemize}
It remains to show that
$\bumpAssignI{\Imc}{a} \in \bumpBoxAssignI{R}{\Imc}
 $ and $
        \bumpAssignI{\Imc}{b} \in \bumpBoxAssignI{R}{\Imc}$. 
         We  again  make a   case distinction. 
\begin{itemize}
        \item \textbf{Dimension $i_C$.} By Definition~\ref{def:mapping}, 
        we have that $\bumpBoxAssignI{R}{\Imc}[i_C]=\BB$ and $\bumpAssignI{\Imc}{a}[i_C]=\bumpAssignI{\Imc}{b}[i_C]=\BC$. By the values in Tables~\ref{tab:ConceptBoxes} and \ref{tab:ConceptIndividual},
        $\bumpAssignI{\Imc}{a}[i_C]=\bumpAssignI{\Imc}{b}[i_C]  \in \bumpBoxAssignI{R}{\Imc}[i_C]$.
        \item \textbf{Dimension $i_{S,c}$.} By Definition~\ref{def:mapping}, 
        we have that $\bumpAssignI{\Imc}{a}[i_{S,c}]=\BRa$ or $\bumpAssignI{\Imc}{a}[i_{S,c}]=\BRna$ and the same for $\bumpAssignI{\Imc}{b}[i_{S,c}]$. Also,
        $\bumpBoxAssignI{R}{\Imc}[i_{S,c}]=\BBRa$ or $\bumpBoxAssignI{R}{\Imc}[i_{S,c}]=\BBRna$. In all cases, by the values in Tables~\ref{tab:RoleBoxes} and~\ref{tab:RoleIndividual}, $\bumpAssignI{\Imc}{a}[i_{S,c}]  \in \bumpBoxAssignI{R}{\Imc}[i_{S,c}]$ and 
        $\bumpAssignI{\Imc}{b}[i_{S,c}]  \in \bumpBoxAssignI{R}{\Imc}[i_{S,c}]$.

        \end{itemize}

We have thus shown that
\begin{align*}
       \posAssignI{\Imc}{a} + \bumpAssignI{\Imc}{b}  &\in \headAssignI{R}{\Imc} 
          \\
        \posAssignI{\Imc}{b}+ \bumpAssignI{\Imc}{a} &\in \tailAssignI{R}{\Imc} 
          \\
        \bumpAssignI{\Imc}{a} &\in \bumpBoxAssignI{R}{\Imc}
          \\
        \bumpAssignI{\Imc}{b} &\in \bumpBoxAssignI{R}{\Imc}.     
    \end{align*}
Then, by Definition~\ref{def:BoxSatisfaction}, we have  that $\eta_\Imc\models R(a,b)$. 

($\Leftarrow$) Assume $\eta_\Imc\models R(a,b)$. Let $c\in\Delta^\Imc$ be the element such that $a^\Imc=c$. By Definition~\ref{def:mapping},  $\posAssignI{\Imc}{a}[i_{R,\element}]=\PRa$ and  $\headAssignI{R}{\Imc}[i_{R,\element}]:=\Qasub$.
Also, either  $\bumpAssignI{\Imc}{b}[i_{R,\element}]=\BRa$  or  $\bumpAssignI{\Imc}{b}[i_{R,\element}]=\BRna$. 
By Definition~\ref{def:BoxSatisfaction},
$\posAssignI{\Imc}{a} + \bumpAssignI{\Imc}{b}  \in \headAssignI{R}{\Imc}$.
%\begin{align*}
 %       \posAssignI{\Imc}{a} + \bumpAssignI{\Imc}{b}  &\in \headAssignI{R}{\Imc} 
  %        \\
   %     \posAssignI{\Imc}{b}+ \bumpAssignI{\Imc}{a} &\in \tailAssignI{R}{\Imc} 
    %      \\
     %   \bumpAssignI{\Imc}{a} &\in \bumpBoxAssignI{R}{\Imc}
      %    \\
       % \bumpAssignI{\Imc}{b} &\in \bumpBoxAssignI{R}{\Imc}.     
   % \end{align*}
This means, in particular, that for the dimension
$i_{R,c}$ 
\[\posAssignI{\Imc}{a}[i_{R,c}] + \bumpAssignI{\Imc}{b}[i_{R,c}]  \in \headAssignI{R}{\Imc}[i_{R,c}].\]
%\begin{align*}
 %       \posAssignI{\Imc}{a}[i_{R,c}] + \bumpAssignI{\Imc}{b}[i_{R,c}]  &\in \headAssignI{R}{\Imc}[i_{R,c}]
  %        \\
   %     \posAssignI{\Imc}{b}[i_{R,c}]+ \bumpAssignI{\Imc}{a}[i_{R,c}] &\in \tailAssignI{R}{\Imc}[i_{R,c}] 
    %      \\
     %   \bumpAssignI{\Imc}{a}[i_{R,c}] &\in \bumpBoxAssignI{R}{\Imc}[i_{R,c}]
      %    \\
      %  \bumpAssignI{\Imc}{b}[i_{R,c}] &\in \bumpBoxAssignI{R}{\Imc}[i_{R,c}].     
    %\end{align*}
 %and $\bumpBoxAssignI{R}{\Imc}[i_{R,\element}]:=\BBRa$. 
The only possible value of $\bumpAssignI{\Imc}{b}[i_{R,\element}]$ that satisfies the conditions above is $\BRa$, which is 
assigned to
    $\bumpAssignI{\Imc}{b}[i_{R,\element}]$ iff 
$(c,b^\Imc)\in R^\Imc$. Since $a^\Imc=c$, we have that $(a^\Imc,b^\Imc)\in R^\Imc$.
In other words, $\Imc \models R(a,b)$.
%, otherwise $\bumpAssignI{\Imc}{a}[i_{R,\element}]=\BRna$, 
 %   $\posAssignI{\Imc}{a}[i_{R,\element}]=\PRa$ if $a=c$, otherwise (that is, $a\neq c$) $\posAssignI{\Imc}{a}[i_{R,\element}]=\PRna$.  
\end{proof}

%\begin{lemma}Given an interpretation \Imc with finite domain, let 
%$\eta_\Imc$ be as in Definition~\ref{def:mapping}.
%For all $d\in \Delta^\Imc$ and all DL-Lite$^\Hmc$ concepts $C$, $d\in C^\Imc$ iff $\mu(d)\subseteq \eta_\Imc(C)$. 
%\end{lemma}

\begin{lemma}\label{lem:canonicalci}Given an interpretation \Imc with finite domain, let 
$\eta_\Imc$ be as in Definition~\ref{def:mapping}.
For all DL-Lite$^\Hmc$ CIs $C\sqsubseteq D$, $\Imc\models C\sqsubseteq D$ iff $ \eta_\Imc\models  C \sqsubseteq D$. 
\end{lemma}
\begin{proof}
We need to show that $C^\Imc\subseteq D^\Imc$ iff $ \eta_\Imc(C) \subseteq \eta_\Imc(D)$. 
($\Rightarrow$) Suppose $C^\Imc\subseteq D^\Imc$.
   To show that $ \eta_\Imc(C) \subseteq \eta_\Imc(D)$, we argue that $ \eta_\Imc(C)[j] \subseteq \eta_\Imc(D)[j]$,  for every dimension $1\leq j\leq d$.
   First assume $\emptyset = D^\Imc$.
   Then, by the ($\Rightarrow$) assumption, $\emptyset = C^\Imc$.
   By the box definition (see Section~\ref{sec:boxSemantics}), we have that 
$\eta_\Imc(C)=\eta_\Imc(D)=\emptyset$
   (when $\epsilon>0$, which is assumed to be the case in this work). So $ \eta_\Imc(C) \subseteq \eta_\Imc(D)$.
    We now make a case distinction with $\emptyset \neq D^\Imc$ and $E\in\NCE$. 
       \begin{itemize}
       %\item   
       % \textbf{Dimension $i_{E}$ with $\emptyset= D^\Imc $.} By the ($\Rightarrow$) assumption, $C^\Imc\subseteq D^\Imc$. Then, by the assumption in this case 
        %$\emptyset= C^\Imc $. By Definition~\ref{def:mapping}
        %$ \eta_\Imc(C)[i_E] = \eta_\Imc(D)[i_E]=\Sempty$, so $ \eta_\Imc(C)[i_E] \subseteq \eta_\Imc(D)[i_E]$.
        \item   
        \textbf{Dimension $i_{E}$ with $\emptyset\neq D^\Imc = E^\Imc$.}
By Definition~\ref{def:mapping},
$\eta_\Imc(E)[i_E]=\Seq$.   
      By the assumption in this case (i) $D^\Imc=E^\Imc$, so  (ii)
      $\eta_\Imc(D)[i_E]=\Seq$ by Definition~\ref{def:mapping}.
      By the ($\Rightarrow$) assumption, $C^\Imc\subseteq D^\Imc$. We have two cases.
      \begin{itemize}
          \item  If $C^\Imc= D^\Imc$ then, by (i),
      $C^\Imc= E^\Imc$. So, by Definition~\ref{def:mapping}, we have that $\eta_\Imc(C)[i_E]=\Seq$, which implies $ \eta_\Imc(C)[i_E] = \eta_\Imc(D)[i_E]$ and then $ \eta_\Imc(C)[i_E] \subseteq \eta_\Imc(D)[i_E]$, as required.
      \item       Otherwise, $C^\Imc\subset  D^\Imc$. Then, by (i),
      $C^\Imc\subset E^\Imc$. If $C^\Imc\neq\emptyset$ then by Definition~\ref{def:mapping},  $\eta_\Imc(C)[i_E]=\Ssub$.  By (ii),
      $\eta_\Imc(D)[i_E]=\Seq$.
      Since  $\Ssub \subseteq \Seq$ (see Table~\ref{tab:ConceptBoxes} and Figure~\ref{fig:ConceptBoxes}),  
      $ \eta_\Imc(C)[i_E] \subseteq \eta_\Imc(D)[i_E]$.  Otherwise (when 
 $C^\Imc= \emptyset$) then $\eta_\Imc(C)=\emptyset$ and we are done. %\todo{Ana: problem here, it was ok when it was the empty set but now with 0,0 it is not ok anymore}
 %so $ \eta_\Imc(C)[i_E] \subseteq \eta_\Imc(D)[i_E]$.
      \end{itemize}
%\textbf{Dimension $i_{E}$ with $\emptyset\neq \eta_\Imc(D) = \eta_\Imc(E)$.}
%By Definition~\ref{def:mapping},
%$\eta_\Imc(E)[i_E]=\Seq$.   
 %     Then, by the assumption in this case,
  %    $\eta_\Imc(D)[i_E]=\Seq$, which only happens when (1) $D^\Imc=E^\Imc$.      
   %   By the ($\Rightarrow$) assumption, $C^\Imc\subseteq D^\Imc$. Then we have two cases.
    %  \begin{itemize}
     %     \item  If $C^\Imc= D^\Imc$ then, by (1),
     % $C^\Imc= E^\Imc$. So, by Definition~\ref{def:mapping}, $\eta_\Imc(C)[i_E]=\Seq$, which implies $ \eta_\Imc(C)[i_E] = \eta_\Imc(D)[i_E]$, so $ \eta_\Imc(C)[i_E] \subseteq \eta_\Imc(D)[i_E]$.
     % \item       Otherwise, $C^\Imc\subset  D^\Imc$. Then, by (1),
     % $C^\Imc\subset E^\Imc$. By Definition~\ref{def:mapping}, (2) $\eta_\Imc(C)[i_E]=\Ssub$.  By the 
     %   assumption in this case,
     % $\eta_\Imc(D)[i_E]=\Seq$.
     % Since  $\Ssub \subseteq \Seq$ (see Table~\ref{tab:ConceptBoxes} and Figure~\ref{fig:ConceptBoxes}),  
     % $ \eta_\Imc(C)[i_E] \subseteq \eta_\Imc(D)%[i_E]$.
      %\end{itemize}
\item
\textbf{Dimension $i_{E}$ with $\emptyset\neq  D^\Imc \subset E^\Imc$.}
 By Definition~\ref{def:mapping} and the assumption in this case  (i) $\cAssignSymb_{\Imc}(D)[i_E] = \Ssub$. By the ($\Rightarrow$) assumption $C^\Imc\subseteq D^\Imc$ and
 by the assumption in this case
 $D^\Imc \subset E^\Imc$, so
 $C^\Imc \subset E^\Imc$. If $C^\Imc\neq \emptyset$ then
 by Definition~\ref{def:mapping}, (ii)
  $\cAssignSymb_{\Imc}(C)[i_E] = \Ssub$. Then, by (i) and (ii) it holds that 
 $ \eta_\Imc(C)[i_E] = \eta_\Imc(D)[i_E]$ so $ \eta_\Imc(C)[i_E] \subseteq \eta_\Imc(D)[i_E]$.
 Otherwise (when 
 $C^\Imc= \emptyset$),    $\eta_\Imc(C)=\emptyset$ and we are done.
 %so $ \eta_\Imc(C)[i_E] \subseteq \eta_\Imc(D)[i_E]$ as argued above. 
 %       Otherwise (when 
 %$C^\Imc= \emptyset$) then $\eta_\Imc(C)[i_E]=\emptyset$ so $ \eta_\Imc(C)[i_E] \subseteq \eta_\Imc(D)[i_E]$.
%\textbf{Dimension $i_{E}$ with $\emptyset\neq \eta_\Imc(D) \subset \eta_\Imc(E)$.}
 %     $\eta_\Imc(D)[i_E]=\Ssub$    
    \item 
    \textbf{Dimension $i_{E}$ with $ D^\Imc \supset E^\Imc \neq \emptyset$.} 
    By Definition~\ref{def:mapping} and the assumption in this case  (i) $\cAssignSymb_{\Imc}(D)[i_E] = \Ssup$.   
  By Definition~\ref{def:mapping}
  $\cAssignSymb_{\Imc}(C)[i_E]$ is either
equal to (ii) $\emptyset$, $\Seq$, $\Ssup$, $\Ssub$, $\Sinter$,  or $\Snsup$. %, or $\Sneq$. 
By the chosen parameters of $\Seq$, $\Ssup$, $\Ssub$, $\Sinter$, and $\Snsup$ %, and $\Sneq$ 
(see Table~\ref{tab:ConceptBoxes} and Figure~\ref{fig:ConceptBoxes}) it holds that (iii)
  $\Seq \cup \Ssup \cup \Ssub \cup \Sinter \cup \Snsup \subseteq \Ssup$.
  %(also $\Sempty\subseteq\Ssup$). \todo{change here}
  Then, by (i)-(iii), any possible value of $\cAssignSymb_{\Imc}(C)[i_E]$ leads to $\cAssignSymb_{\Imc}(C)[i_E] \subseteq \cAssignSymb_{\Imc}(D)[i_E]$, as required.
         \item 
        \textbf{Dimension $i_{E}$ with $\emptyset\neq E^\Imc$, $\emptyset\neq   D^\Imc$, and $\emptyset = E^\Imc\cap   D^\Imc$.}
        By Definition~\ref{def:mapping} and the assumption in this case  (i) $\eta_\Imc(D)[i_E]=\Snsup$. By the ($\Rightarrow$) assumption $C^\Imc\subseteq D^\Imc$, so
        $\emptyset = E^\Imc\cap   C^\Imc$.
     If $\emptyset\neq C^\Imc$ then
     by Definition~\ref{def:mapping}   (ii) $\eta_\Imc(C)[i_E]=\Snsup$. By (i) and (ii)
     $ \eta_\Imc(C)[i_E] = \eta_\Imc(D)[i_E]$, so $ \eta_\Imc(C)[i_E] \subseteq \eta_\Imc(D)[i_E]$. The case where 
 $C^\Imc= \emptyset$ is as argued above.
     
    \item  \textbf{Dimension $i_{E}$ with none of the above.} By Definition~\ref{def:mapping} and the assumption in this case  (i) $\eta_\Imc(D)[i_E]=\Sinter$. Again by Definition~\ref{def:mapping},
  $\cAssignSymb_{\Imc}(C)[i_E]$ is either
equal to (ii) $\emptyset$, $\Seq$, $\Ssup$, $\Ssub$, $\Sinter$,  or $\Snsup$. %, or $\Sneq$. 
By the chosen parameters of $\Seq$, $\Ssup$, $\Ssub$, $\Sinter$, and $\Snsup$ %, and $\Sneq$ 
(see Table~\ref{tab:ConceptBoxes} and Figure~\ref{fig:ConceptBoxes}) it holds that (iii)
  $  \Ssub \cup \Sinter \cup \Snsup \subseteq \Sinter$ (also $\emptyset\subseteq\Sinter$).
  If $\cAssignSymb_{\Imc}(C)[i_E]$ is $\emptyset$, 
  $\Ssub$, $\Sinter$,  or $\Snsup$ then by (i)-(iii) $ \eta_\Imc(C)[i_E] \subseteq \eta_\Imc(D)[i_E]$. 
  Otherwise $\cAssignSymb_{\Imc}(C)[i_E]$ is either $\Seq$ or  $\Ssup$. The former case happens iff $C^\Imc=E^\Imc\neq \emptyset$.
  However, by the ($\Rightarrow$) assumption, we have that $C^\Imc\subseteq D^\Imc$. So
  $C^\Imc=E^\Imc\neq \emptyset$ implies 
  $\emptyset \neq E^\Imc\subseteq D^\Imc$, which cannot happen by the ``none of the above'' assumption in this case. The latter case happens iff $C^\Imc\supset E^\Imc\neq \emptyset$. Again by the ($\Rightarrow$) assumption, we have that $C^\Imc\subseteq D^\Imc$. So
  $C^\Imc\supset E^\Imc\neq \emptyset$ implies 
  $\emptyset \neq E^\Imc\subseteq D^\Imc$, which cannot happen by the ``none of the above'' assumption.
    %if 
    %$C^\Imc\not\subseteq D^\Imc$,     $D^\Imc\not\subseteq C^\Imc$,  $\emptyset \neq C^\Imc\cap   D^\Imc$
        \end{itemize}
        It remains to argue for the dimensions of the form $i_{R,c}$, where $R\in\NR$ and $c\in \Delta^\Imc$. By Definition~\ref{def:mapping}, for every pair $(R,\element)$ with $R\in \NR$ and $\element\in \Delta^\Imc$, $\eta_\Imc(C)[i_{R,\element}]=\eta_\Imc(D)[i_{R,\element}]=\SC$. So, for every  dimension of this form,
    $\eta_\Imc(C)[i_{R,\element}]\subseteq\eta_\Imc(D)[i_{R,\element}]$.
    
        ($\Leftarrow$) Now suppose $ \eta_\Imc(C) \subseteq \eta_\Imc(D)$. 
        If $\eta_\Imc(C)=\emptyset$
        then, by Definition~\ref{def:mapping}, this happens iff 
        $C^\Imc=\emptyset$ and
        we are done since this trivially implies
        $C^\Imc\subseteq D^\Imc$.
        If $\eta_\Imc(D)=\emptyset$ 
        then, by the ($\Leftarrow$) assumption, $\eta_\Imc(C)=\emptyset$. 
        Since, again by Definition~\ref{def:mapping}, this happens iff 
        $C^\Imc=\emptyset$
        we are done since this trivially implies
        $C^\Imc\subseteq D^\Imc$.
        We can then assume that
        both $\eta_\Imc(C)\neq\emptyset$
        and $\eta_\Imc(D)\neq\emptyset$ hold, which means by Definition~\ref{def:mapping} that $C^\Imc\neq\emptyset$ and $D^\Imc\neq\emptyset$.
        %This means in particular that $ \eta_\Imc(C)[i_D] \subseteq \eta_\Imc(D)[i_D]$.  We want to show that $C^\Imc\subseteq D^\Imc$.
        %Assume $D^\Imc\neq\emptyset$ (otherwise $ \eta_\Imc(C)[i_D] \subseteq \eta_\Imc(D)[i_D]$ implies $C^\Imc=\emptyset$ and we are done).
        By Definition~\ref{def:mapping},
        $\eta_\Imc(D)[i_D]=\Seq$. By the ($\Leftarrow$) assumption, (i) $ \eta_\Imc(C)[i_D] \subseteq \Seq$. 
  By Definition~\ref{def:mapping}
  $\cAssignSymb_{\Imc}(C)[i_D]$ is either
equal to   $\Seq$, $\Ssup$, $\Ssub$, $\Sinter$,  or $\Snsup$. By the chosen parameters of $\Seq$, $\Ssup$, $\Ssub$, $\Sinter$, and $\Snsup$ 
(see Table~\ref{tab:ConceptBoxes} and Figure~\ref{fig:ConceptBoxes}) the only options that are subsets of $\Seq$ are   $\Seq$ itself and
$\Ssub$. 
%If $\eta_\Imc(C)[i_D]=\Sempty$ then
%this means $C^\Imc=\emptyset$ and we are done.
%Otherwise, 
By (i), either  $\eta_\Imc(C)[i_D]=\Seq$ or  $\eta_\Imc(C)[i_D]=\Ssub$. 
The former case can only happen if $C^\Imc=D^\Imc$, which means that $C^\Imc\subseteq D^\Imc$, as required.
The latter case can only happen if $C^\Imc\subset D^\Imc$, which means that $C^\Imc\subseteq D^\Imc$, as required.
%\begin{itemize}
 %   \item[(a)] $\eta_\Imc(C)[i_D]=\Seq$:
  %  this case can only happen if $C^\Imc=D^\Imc$, which means that $C^\Imc\subseteq D^\Imc$, as required.
   % \item[(b)] $\eta_\Imc(C)[i_D]=\Ssub$: this case can only happen if $C^\Imc\subset D^\Imc$, which means that $C^\Imc\subseteq D^\Imc$, as required.
%\end{itemize}
%it holds that 
 % $\Seq \cup \Ssup \cup \Ssub \cup \Sinter \cup \Snsup \subseteq \Ssup$. 
\end{proof}

\begin{lemma}\label{lem:faithfulnessroleinclusions}Given an interpretation \Imc with finite domain, let 
$\eta_\Imc$ be as in Definition~\ref{def:mapping}.
For all  % {positive} 
DL-Lite$^\Hmc$ RIs $R\sqsubseteq S$, 
% (that is, DL-Lite$^\Hmc$ without role disjointness) 
%\todo{I removed ``positive'' and ``(that is, DL-Lite$^\Hmc$ without role disjointness)'', as this is equal to ``DL-Lite$^\Hmc$'' right?} yes thanks
%todo: think how we say this
  $\Imc\models R \sqsubseteq S$ iff 
$ \eta_\Imc\models R \sqsubseteq S$. 
\end{lemma}
\begin{proof}
($\Rightarrow$) Assume $R^\Imc\subseteq S^\Imc$. To show that $\eta_\Imc(R) \subseteq \eta_\Imc(S) $, we show   that
  \begin{align*}
            \headAssignI{R}{\Imc} &\subseteq \headAssignI{S}{\Imc},
        \\
           \tailAssignI{R}{\Imc} &\subseteq \tailAssignI{S}{\Imc},
        \\
            \bumpBoxAssignI{R}{\Imc} &\subseteq \bumpBoxAssignI{S}{\Imc}.
    \end{align*}
    We start with arguing that $\headAssignI{R}{\Imc} \subseteq \headAssignI{S}{\Imc}$. 
    We make the following case distinction.
    \begin{itemize}
    \item \textbf{Dimension $i_{C}$.} By assumption,
    $R^\Imc\subseteq S^\Imc$. This implies $(\exists R)^\Imc\subseteq (\exists S)^\Imc$.
    By Definition~\ref{def:mapping},  $\headAssignI{R}{\Imc}[i_{C}]=\eta_\Imc(\exists R)[i_{C}]$ and $\headAssignI{S}{\Imc}[i_{C}]=\eta_\Imc(\exists S)[i_{C}]$. By Lemma~\ref{lem:canonicalci}, $\eta_\Imc(\exists R)[i_{C}]\subseteq \eta_\Imc(\exists S)[i_{C}]$ and we are done.
        \item \textbf{Dimension $i_{T,c}$ and for all $ \anotherelement \in \Delta^{\Imc}$ we have that $(c,\anotherelement) \in S^{\Imc}$ implies $(c,\anotherelement) \in T^{\Imc}$.} By Definition~\ref{def:mapping}, $\headAssignI{S}{\Imc}[i_{T,\element}]=\Qasub$. By assumption,
    $R^\Imc\subseteq S^\Imc$. Then, for all $ \anotherelement \in \Delta^{\Imc}$ we have that $(c,\anotherelement) \in R^{\Imc}$ implies $(c,\anotherelement) \in T^{\Imc}$. By Definition~\ref{def:mapping}, 
    $\headAssignI{R}{\Imc}[i_{T,\element}]=\Qasub$. So
    $\headAssignI{R}{\Imc}[i_{T,\element}]=\headAssignI{S}{\Imc}[i_{T,\element}]$ and thus $\headAssignI{R}{\Imc}[i_{T,\element}]\subseteq\headAssignI{S}{\Imc}[i_{T,\element}]$.
         \item \textbf{Dimension $i_{T,c}$ and there is $ \anotherelement \in \Delta^{\Imc}$ with $(c,\anotherelement) \in S^{\Imc}$ but $(c,\anotherelement) \notin T^{\Imc}$.} By Definition~\ref{def:mapping}, $\headAssignI{S}{\Imc}[i_{T,\element}]=\Qnasub$.
         Also, $\headAssignI{R}{\Imc}[i_{T,\element}]=\Qasub$ or $\headAssignI{R}{\Imc}[i_{T,\element}]=\Qnasub$. In both cases, by the values in Table~\ref{tab:RoleBoxes}, we have that
         $\headAssignI{R}{\Imc}[i_{T,\element}]\subseteq\headAssignI{S}{\Imc}[i_{T,\element}]$.
    \end{itemize}
     We now  argue that $\tailAssignI{R}{\Imc} \subseteq \tailAssignI{S}{\Imc}$.  
     \begin{itemize}
     \item \textbf{Dimension $i_{C}$.} By assumption,
    $R^\Imc\subseteq S^\Imc$. This implies $(R^-)^\Imc\subseteq (S^-)^\Imc$, so $(\exists R^-)^\Imc\subseteq (\exists S^-)^\Imc$.
    By Definition~\ref{def:mapping}, $\tailAssignI{R}{\Imc}[i_{C}]=\eta_\Imc(\exists R^-)[i_{C}]$ and $\tailAssignI{S}{\Imc}[i_{C}]=\eta_\Imc(\exists S^-)[i_{C}]$. Then, by Lemma~\ref{lem:canonicalci}, $\eta_\Imc(\exists R^-)[i_{C}]\subseteq \eta_\Imc(\exists S^-)[i_{C}]$ and we are done.
     \item \textbf{Dimension $i_{T,c}$ and for all $ \anotherelement \in \Delta^{\Imc}$ we have that $(\anotherelement,c) \in S^{\Imc}$ implies $(c,\anotherelement) \in T^{\Imc}$.} In this case, $\tailAssignI{S}{\Imc}[i_{T,\element}]=\Qasub$. By assumption,
    $R^\Imc\subseteq S^\Imc$. Then, for all $ \anotherelement \in \Delta^{\Imc}$ we have that $(\anotherelement,c) \in R^{\Imc}$ implies $(c,\anotherelement) \in T^{\Imc}$. By Definition~\ref{def:mapping}, 
    $\tailAssignI{R}{\Imc}[i_{T,\element}]=\Qasub$. So
    $\tailAssignI{R}{\Imc}[i_{T,\element}]=\tailAssignI{S}{\Imc}[i_{T,\element}]$ and thus $\tailAssignI{R}{\Imc}[i_{T,\element}]\subseteq\tailAssignI{S}{\Imc}[i_{T,\element}]$.
         \item \textbf{Dimension $i_{T,c}$ and there is $ \anotherelement \in \Delta^{\Imc}$ with $(\anotherelement,c) \in S^{\Imc}$ but $(c,\anotherelement) \notin T^{\Imc}$.} By Definition~\ref{def:mapping}, $\tailAssignI{S}{\Imc}[i_{T,\element}]=\Qnasub$.
         Also, $\tailAssignI{R}{\Imc}[i_{T,\element}]=\Qasub$ or $\tailAssignI{R}{\Imc}[i_{T,\element}]=\Qnasub$. In both cases, by the values in Table~\ref{tab:RoleBoxes}, we have that
         $\tailAssignI{R}{\Imc}[i_{T,\element}]\subseteq\tailAssignI{S}{\Imc}[i_{T,\element}]$.
     \end{itemize}
     Finally, we  argue that $\bumpBoxAssignI{R}{\Imc} \subseteq \bumpBoxAssignI{S}{\Imc}$. 
     \begin{itemize}
         \item \textbf{Dimension $i_C$.} In this case, by Definition~\ref{def:mapping}, 
           $\bumpBoxAssignI{R}{\Imc}[i_C]=\bumpBoxAssignI{S}{\Imc}[i_C]=\BB$. Then, trivially, $\bumpBoxAssignI{R}{\Imc}[i_C]\subseteq\bumpBoxAssignI{S}{\Imc}[i_C]$.
        \item \textbf{Dimension $i_{S,c}$ and  for all $ \anotherelement \in \Delta^{\Imc}$ we have that $(c,\anotherelement) \in R^{\Imc}$ implies $(c,\anotherelement) \in T^{\Imc}$.} By Definition~\ref{def:mapping}, $\bumpBoxAssignI{S}{\Imc}[i_{T,\element}]=\BBRa$.
         Also, $\bumpBoxAssignI{R}{\Imc}[i_{T,\element}]=\BBRa$ or $\bumpBoxAssignI{R}{\Imc}[i_{T,\element}]=\BBRna$. In both cases, by the values in Table~\ref{tab:RoleBoxes}, we have that
         $\bumpBoxAssignI{R}{\Imc}[i_{T,\element}]\subseteq\bumpBoxAssignI{S}{\Imc}[i_{T,\element}]$.
         \item \textbf{Dimension $i_{T,c}$ and there is $ \anotherelement \in \Delta^{\Imc}$ with $(c,\anotherelement) \in R^{\Imc}$ but $(c,\anotherelement) \notin T^{\Imc}$.} By Definition~\ref{def:mapping}, $\bumpBoxAssignI{R}{\Imc}[i_{T,\element}]=\BBRna$. By assumption, we have that 
    $R^\Imc\subseteq S^\Imc$. This implies that  there is $ \anotherelement \in \Delta^{\Imc}$ with $(c,\anotherelement) \in S^{\Imc}$ but $(c,\anotherelement) \notin T^{\Imc}$. By Definition~\ref{def:mapping}, 
    $\bumpBoxAssignI{S}{\Imc}[i_{T,\element}]=\BBRna$. So
    $\bumpBoxAssignI{R}{\Imc}[i_{T,\element}]=\bumpBoxAssignI{S}{\Imc}[i_{T,\element}]$ and thus $\bumpBoxAssignI{R}{\Imc}[i_{T,\element}]\subseteq\bumpBoxAssignI{S}{\Imc}[i_{T,\element}]$.
     \end{itemize}

    ($\Leftarrow$) Assume $ \eta_\Imc(R) \subseteq \eta_\Imc(S)$. 
    Let $c$ be an arbitrary element of $\Delta^\Imc$. By Definition~\ref{def:mapping},
    $\headAssignI{S}{\Imc}[i_{S,c}]=\Qasub$.
    By Definition~\ref{def:BoxSatisfaction} and the ($\Leftarrow$) assumption, 
    $\headAssignI{R}{\Imc} \subseteq \headAssignI{S}{\Imc}$, so (i)
   $\headAssignI{R}{\Imc}[i_{S,c}]\subseteq\Qasub$. By Definition~\ref{def:mapping} again, $\headAssignI{R}{\Imc}$ can
   be either $\Qasub$ or $\Qnasub$. 
   Since $\Qnasub\not\subseteq\Qasub$ and (i) holds, we actually have that (ii) $\headAssignI{R}{\Imc}[i_{S,c}]=\Qasub$ and, by Definition~\ref{def:mapping}, %$\headAssignI{R}{\Imc}[i_{S,c}]=\Qasub$ iff,
   for all $e\in\Delta^\Imc$,
   $(c,e)\in R^\Imc$ implies $(c,e)\in S^\Imc$.
   Since $c$ was an arbitrary element of $\Delta^\Imc$, (iii) this holds for all such elements. 
   Thus, by (ii)-(iii), $R^\Imc\subseteq S^\Imc$. 
\end{proof}

\thmfaithfulness*
\begin{proof}
    Given a satisfiable DL-Lite$^\Hmc$ KB 
    \Kmc, 
    let $\Imc_\Kmc$ be the canonical model for \Kmc, as in Definition~\ref{def:canonicalModel}.
    By Theorem~\ref{thm:canonicalModel},
    for all DL-Lite$^\Hmc$ axioms $\alpha$, % over   ${\sf sig}(\Kmc)$, 
    we have that $\Kmc \models \alpha$ iff $\Imc_{\Kmc} \models \alpha$. By Lemmas~\ref{lem:faithfulnessconceptassertions}, \ref{lem:faithfulnessroleassertions},
    \ref{lem:canonicalci},
    and \ref{lem:faithfulnessroleinclusions},
    $\Imc_{\Kmc} \models \alpha$ iff 
    $\eta_{\Imc_{\Kmc}} \models \alpha$.
    Then, $\Kmc \models \alpha$ iff $\eta_{\Imc_{\Kmc}} \models \alpha$, which corresponds to the notion of strong KB faithfulness~\citep{OLW20,DBLP:conf/kr/Bourgaux0KLO24}. Note that $\World[i_C]=(-4,4)$ (\cref{tab:RoleBoxes}) implies
    $s_\Omega=4$ in our construction.
\end{proof}
%\todo{bring statement and check for (A) in the proof}

%\subsection{Proofs of Corollaries~\ref{cor:ConsistencyDim} and \ref{corollary:FaithfulnessDim}}

\corFaithfulSolelyABox*

\begin{proof}
By Claim~\ref{clm:satisfiable}, any DL-Lite$^\Hmc$ KB with an empty ABox is satisfiable. 
%(indeed, simply take
%$\Imc$ with $A^\Imc=\emptyset$ and $R^\Imc=\emptyset$ for all $A\in\NC$ and all $R\in\NR$). 
Then, Corollary \ref{cor:ConsistencyDim} follows from the proof of Theorem \ref{thm:faith_bool_alc}.
Specifically, if $\ontoo$'s ABox is empty, then the translation of a model $\Imc$ of $\ontoo$ to $\cAssignSymb_{\Imc}$ (see Definition \ref{def:mapping}) reduces to (1) one dimension $i_C$ for any concept $C \in \NCE$ and (2) one dimension $i_{R,c}$ for any pair of roles $R \in \NR$ and elements $c \in \Delta^\Imc$. 
Thus, the constructed embedding has $|\NCE| + |\NR| |\Delta^\Imc|$ dimensions. This dimensionality bound depends on $|\Delta^\Imc|$. Thus, translating different models of $\ontoo$ leads to \geometric interpretations with different dimensionalities. This means  that the selected model of $\ontoo$ influences the dimensionality bounds for weak   faithfulness, as we show next.

%\textbf{Weak Faithfulness.} In the case of weak faithfulness, we can choose an arbitrary model of $\ontoo$, as we do not need the unnamed individuals of $\Delta_\ontoo$ to falsify unwanted inclusions. More formally, 
For any $\Imc \models \ontoo$, $(i)$ by Theorem~\ref{thm:propembedding} it holds that the constructed $\cAssignSymb_{\Imc}$ is box consistent and $(ii)$ by Lemmas~\ref{lem:faithfulnessconceptassertions}-\ref{lem:faithfulnessroleinclusions}, it holds that $\cAssignSymb_{\Imc}$ is KB-entailed for \Kmc.
%\models \ontoo$.
By Points $(i)$ and $(ii)$ and Proposition~\ref{lem:tboxfaith} it holds that any $\cAssignSymb_{\Imc}$ with $\Imc \models \ontoo$ is a  (weakly) TBox faithful model of $\ontoo$. In particular, we can choose a model $\Imc$ with the smallest domain $|\Delta^\Imc| = 1$, which results in $\cAssignSymb_{\Imc}$ having the following number of dimensions: 
\begin{align*}
|\NCE| + |\NR| |\Delta^\Imc| &=\\
|\NC| + 2|\NR| + |\NR| &=\\
|\NC| + 3|\NR|
\end{align*}
%
%\textbf{Strong Faithfulness.} In the case of strong faithfulness, we have shown in Theorem \ref{thm:faith_bool_alc} that we can translate the canonical model $\Imc_\ontoo$ of $\ontoo$ to a \geometric interpretation $\cAssignSymb_{\Imc_\ontoo}$ that is a strongly TBox faithful model of $\ontoo$. Choosing the canonical model for the translation, leads to an embedding with the following number of dimensions: 
% \begin{align*}
% |\NCE| + |\NR| |\Delta_\ontoo| &=\\
% |\NC| + 2|\NR| + |\NR| 2(|\NC| + 2|\NR|) &= \\
% (2|\NR| + 1) (|\NC| + 2|\NR|) &
% \end{align*}
%
%\begin{align*}
%|\NCE| + |\NR| |\Delta^{\Imc_\ontoo}| &=\\
%|\NC| + 2|\NR| + |\NR| (|\NC| + 2|\NR|)^2 &=\\
%|\NC| + |\NR| (2 + (|\NC| + 2|\NR|)^2)
%\end{align*}

%Observe that 
One can extend $\cAssignSymb_{\Imc_{\ontoo}}$ to a box interpretation $\cAssignSymb_{\Imc_{\ontoo}}'$ with more dimensions
%by adding an arbitrary number of additional dimensions 
while keeping $\cAssignSymb_{\Imc_{\ontoo}}'$ a weakly TBox faithful model of $\ontoo$. This can be done, for instance, 
%it can be easily seen that 
by copying an arbitrary dimension of $\cAssignSymb_{\Imc_{\ontoo}}$ multiple times. This keeps the extended $\cAssignSymb_{\Imc_{\ontoo}}'$ a weakly TBox faithful model of $\ontoo$, while increasing its dimensionality arbitrarily. 
%It is easy to see that 
%In %exactly 
%the same way, one can use the copying strategy to extend $\cAssignSymb_{\Imc}$ to $\cAssignSymb_{\Imc}'$ with an arbitrary number of additional dimensions while keeping it a weakly TBox faithful model of $\ontoo$. 
Thus, we have shown that %$(i)$ 
for any $d \geq |\NC| + 3|\NR|$ there is a $\cAssignSymb$, s.t., $\cAssignSymb$ is a weakly TBox faithful model of $\ontoo$.
%; and $(ii)$ for any $d \geq |\NC| + |\NR| (2 + (|\NC| + 2|\NR|)^2)$ there is a $\cAssignSymb$, s.t., $\cAssignSymb$ is a strongly TBox faithful model of $\ontoo$.
\end{proof}

\corFaithfulEnitreOntoo*
%\noindent
%\textbf{Corollary 2.}
%\textit{
%Let $\ontoo$ be a satisfiable DL-Lite$^\Hmc$ KB. %\todo{Ana: do we need satisfiable here?}
    %\(\ontoo\). 
 %   Then for any $d \geq d_{\min}$ there is a \geometric interpretation $\cAssignSymb$ with dimensionality $d$ such that:
%\begin{enumerate}[$(i)$]
 %   \item $d_{\min} = |\NC| + |\NR|(2 +|\NI|+2|\NR|)$  and $\cAssignSymb$ is a weakly  {KB} faithful model of $\ontoo$.
  %  \item $d_{\min} = |\NC| + |\NR| (2 + |\NI| + (|\NC| + 2|\NR|)^2)$  and $\cAssignSymb$ is a strongly  {KB}
  %  faithful model of $\ontoo$.
%\end{enumerate}
%}
% Let $\ontoo$ be a DL-Lite$^\Hmc$-ontology
%     \(\ontoo\).
% Then for any $d \geq {\color{red}|\NC| + 2|\NR| + |\NR| (|\NC| + 2|\NR| + |\NI|)}
% $ there is a \geometric interpretation $\cAssignSymb$ with dimensionality $d$ such that  $\cAssignSymb$ is a strongly faithful model of $\ontoo$.
% \end{restatable}

\begin{proof} { Analogous  to the proof of Corollary \ref{cor:ConsistencyDim}}, Corollary \ref{corollary:FaithfulnessDim} essentially follows from the proof of Theorem \ref{thm:faith_bool_alc}. Specifically, the translation of a model $\Imc$ of $\ontoo$ to $\cAssignSymb_{\Imc}$ (see Definition \ref{def:mapping}) reduces to (1) one dimension $i_C$ for any concept $C \in \NCE$ and (2) one dimension $i_{R,c}$ for any pair of roles $R \in \NR$ and elements $c \in \Delta^\Imc$. Thus, the constructed embedding has $|\NCE| + |\NR| |\Delta^\Imc|$ dimensions. This dimensionality bound depends again on $|\Delta^\Imc|$. Thus, translating different models of $\ontoo$ leads to \geometric interpretations with different dimensionalities. This means, in particular, that the selected model of $\ontoo$ influences the dimensionality bounds for faithfulness, as we show next.

%\textbf{Weak Faithfulness.} In the case of weak faithfulness, we can choose an arbitrary model of $\ontoo$, as we do not need the unnamed individuals of $\Delta_\ontoo$ to falsify unwanted inclusions. More formally, 
For any $\Imc \models \ontoo$, $(i)$ by Theorem~\ref{thm:propembedding} it holds that the constructed $\cAssignSymb_{\Imc}$ is box consistent and $(ii)$ by Lemmas~\ref{lem:faithfulnessconceptassertions}-\ref{lem:faithfulnessroleinclusions}, it holds that $\cAssignSymb_{\Imc} \models \ontoo$. By Points $(i)$ and $(ii)$ and Proposition~\ref{lem:kbfaith} it holds that any $\cAssignSymb_{\Imc}$ with $\Imc \models \ontoo$ is weakly KB faithful. In particular, we know by Corollary \ref{cor:ModelExistence} that $\ontoo$ has a model $\Imc$ with $|\Delta^\Imc| = |\NI|+2|\NR|$, 
%\todo{AO: this needs some fix}
which results in $\cAssignSymb_{\Imc}$ having the following number of dimensions: 
\begin{align*}
|\NCE| + |\NR| |\Delta^\Imc| &=\\
|\NC| + 2|\NR| + |\NR| (|\NI|+2|\NR|) &=\\
|\NC| + |\NR|(2 +|\NI|+2|\NR|)
\end{align*}

As argued in \cref{cor:ConsistencyDim},   one can modify $\cAssignSymb_{\Imc_{\ontoo}}$ 
so as to add more dimensions while keeping 
%$\cAssignSymb_{\Imc_{\ontoo}}'$ 
%it a strongly KB faithful model of $\ontoo$. 
%In the same way, one can modify $\cAssignSymb_{\Imc}$ 
%so as to add more dimensions while keeping 
%$\cAssignSymb_{\Imc_{\ontoo}}'$ 
it a weakly KB faithful model of $\ontoo$.  
Thus, we have shown that %$(i)$ 
for any 
%\todo{check the number that follows this note}
$d \geq |\NC| + |\NR|(2 + |\NI|+2|\NR|)$
%|\NI|)$ 
there is a $\cAssignSymb$, s.t., $\cAssignSymb$ is a weakly KB faithful model of $\ontoo$.
%; and %$(ii)$ for any $d \geq |\NC| + |\NR| (2 + |\NI| + (|\NC| + 2|\NR|)^2)$ there is a $\cAssignSymb$, s.t., $\cAssignSymb$ is a strongly KB faithful model of $\ontoo$.
\end{proof}

\section{Convex Optimization Formulation for \modelName and Proofs for Section~\ref{sec:opt}}
\label{app:opt}
Here we provide the details for the signed distance function and the convex optimization problem in Equation~\eqref{eq:opt_ont}.
%As described in Section~\ref{sec:opt}, given a \geometric interpretation $\eta$, we concatenate all the parameters of $\eta$ into a single vector $z$ of dimension $n := (2|\NI|+2|\NC|+6|\NR|)d$. Conversely, given such an embedding solution $z \in \Re^n$, we can reconstruct $\eta$.
%The problem in Equation~\eqref{eq:opt_ont}
%is stated in terms of $z$ but, for simplicity, in the next subsections we will describe the constraints and the objective functions in terms of the corresponding $\eta$.

\subsection{Signed distance function}\label{app:signed}
The notation here is as described in Section~\ref{sec:Scoring}.
\thmsign*
\begin{proof}
This proposition also follows from Example~5.1 in \citep{LWL18}, where the expression for $\sdist(y, \Re^n_+)$ is described. For the sake of self-containment, we present a complete proof here.

We consider two cases. 
First, suppose $y \not \in \Re^n_-$, so,
by definition,
$\sdist(y,\Re^n_-) =\dist(y, \Re^n_-)$.
By direct computation,
\begin{align*}
\dist(y, \Re^n_-) & =    \inf_{x \in \Re^n_-}\left[ \sum_{i=1}^n (y_i-x_i)^2 \right]^{1/2}\\
& =    \left[ \sum_{i=1}^n \inf_{x_i \in \Re_-}(y_i-x_i)^2 \right]^{1/2}\\
& = \left[ \sum_{i=1}^n \max(y_i,0)^2 \right]^{1/2} \\
& = \norm{y^+}_2.
\end{align*}
This takes care of the first case. 

Next, suppose that $y  \in \Re^n_- $.
Let $\partial \Re^n_-$ denote the topological boundary of $\Re^n_-$, which corresponds 
to the elements of $\Re^n_-$ that have at least one zero component.
As $y \in \Re^n_-$ by assumption, the distance between $y$ and $(\Re^n_-)^{c}$ is the same as the distance between $y$ and the boundary $\partial \Re^n_-$, that is,  $\dist(y, (\Re^n_-)^{c}) = \dist(y, \partial \Re^n_-)$. 

Let $\mathcal{P}$ denote the set of all  subsets of $\{1,\ldots, n\}$ except for $\emptyset$.
Then, $\partial \Re^n_-$ is the union of sets of the form $\Re^n_{\mathcal{Q}} \coloneqq \{ x \in \Re^n_- \mid x_i = 0,  \forall i \in \mathcal{Q}\}$, where $\mathcal{Q} \in \mathcal{P}$.
With that, $\dist(y, (\Re^n_-)^{c})$ is the minimum of the distances between $y$ and all $\Re^n_{\mathcal{Q}}$ for $\mathcal{Q} \in \mathcal{P}$. 
%Therefore, if $y$ has a zero component (say, $y_i = 0$) % (i.e., $\exists i \in \{1,\ldots, n\}$ such that $y_i \neq 0$) 
%then $\dist(y,(\Re^n_-)^c)$ must be zero since $\dist(y, \Re^n_{\{i\}}) = 0$.

Given $\mathcal{Q} \in \mathcal{P}$, we have that 
$\dist(y,  \Re^n_{\mathcal{Q}})$ is equal to
\begin{align*}
 &     \inf_{x \in \Re^n_{\mathcal{Q}}}\left[ \sum_{i \in \mathcal{Q} } (y_i-x_i)^2 +
\sum_{i \in \{1,\ldots,n\}\setminus \mathcal{Q} } (y_i-x_i)^2
\right]^{1/2}\\
& =  \left[ \sum _{i \in Q} y_i^2\right]^{1/2},
\end{align*}
where the equality follows because $x_i = 0$ for $i \in Q$ and the infimum value for the second summation is zero since $x_i$ can be taken to be equal to $y_i$ when $i$ does not belong to $Q$.

Therefore, in order to minimize $\dist(y,  \Re^n_{\mathcal{Q}})$ over all $\mathcal{Q} \in \mathcal{P}$, it is enough to consider a singleton subset $\mathcal{Q}$ %. In order for this selection to be optimal, it is enough to consider a singleton set 
corresponding to a component of $y$ that has the smallest absolute value. %needs to be as close to zero as possible. 
%find the ``subvector''\footnote{A subvector of a vector  $\mathbf{x}$ is a vector that differs from $\mathbf{x}$ by having $0$ instead of the value of one or more components of $\mathbf{x}$.} of $y$ that has the smallest $2$-norm.  This corresponds to selecting a subvector of $y$ with at most one nonzero component. 
Since $y \in \Re^n_-$, the absolute value of a smallest component of $y$ is given by $|\max _{i \in \{1,\ldots, n\}} y_i|$.
Overall
\[
-\dist(y, (\Re^n_-)^c) = - \dist(y, \partial \Re^n_-) = \max _{i \in \{1,\ldots, n\}} y_i.
\]
\end{proof}
As CVXPY does not implement signed distance functions natively, we implement  $\sdist( \ \cdot \ , \Re^n_-)$ using the support function feature provided by CVXPY. We now review some convex analysis concepts related to that. 

Given a set $S \subseteq \Re^n$, we define the \emph{support function of $S$} by $\sigma_S(y) \coloneqq \sup  \{\langle y,z \rangle \mid z \in S\}$, where 
$\langle y,z \rangle$ indicates the usual Euclidean dot product between $y$ and $z$.

The signed distance function to a convex set is convex, e.g., see Theorem 10.1 in Chapter 7 of \citep{DZ11} or Section 3.3 of \citep{LWL18}.
In particular, since $\Re^n_-$ is a convex cone\footnote{A convex cone $\mathcal{K} \subseteq \Re^n$ is a convex set satisfying $\alpha y \in \mathcal{K}$ for all $\alpha \geq 0$ and all $y \in \mathcal{K}$.}, the 
function $\sdist(\cdot, \Re^n_-)$ is convex and positively homogeneous\footnote{That is, $\sdist(\alpha y, \Re^n_-) = \alpha \sdist(y, \Re^n_-)$ holds, for all $\alpha \in \Re_+$, $y \in \Re^n$.}.
Furthermore, $\sdist(\cdot, \Re^n_-)$  is finite everywhere.
Then, it follows from Corollary~13.2.2 in \citep{Rt97} that $\sdist(\cdot, \Re^n_-)$ can be expressed as the support function of a certain convex set. 
That is, there exists a convex set $C \subseteq \Re^n$ such that $\sdist(y,\Re^n_-) = \sigma_C(y)$ holds for every $y \in \Re^n$.
We implement $\sdist( \ \cdot \ , \Re^n_-)$ in our code by expressing it as the support function of  a certain convex set.

Let $B_n \coloneqq \{y \in \Re^n \mid \norm{y}_2 \leq 1\}$ denote the unit ball in $\Re^n$ and 
$P_n \coloneqq \{y \in \mathbb{R}^n \mid y \geq 0, y_1 + \cdots y_n \geq 1\}$.
In what follows, given $y,z \in \Re^n$, we use $y \leq z$ to indicate that $y_i \leq z_i$ holds for $i \in \{1,\ldots,n\}$.
With that, the next proposition tell us precisely how to obtain 
the signed distance function to $\Re^n_-$ as the support function of a convex set.

\begin{proposition}\label{prop:sup_sign}
Let $C_n = B_n\cap P_n$. For every 
$y \in \Re^n$ the following hold.
\begin{align*}
\sdist(y,\Re^n_-) &= \sigma_{C_n}(y) \\
 &=  \inf_{z\in \mathbb{R}^n}\{\norm{z}_2 + \max_{i\in\{1,\ldots,n\}}(y_i-z_i) \mid y\leq z \},
\end{align*}
\end{proposition}
\begin{proof}
The result is straightforward for 
$n =1$, so henceforth we assume 
that $n \geq 2$.
Before we proceed, we need some extra convex analysis 
preliminaries, for more details 
see \citep{Rt97,HU93I,HU93II}.

For a set $S \subseteq \mathbb{R}^n$ denote its indicator function by $\delta_S:\mathbb{R}^n \to \mathbb{R}^n \cup\{+\infty\} $.  By 
definition, we have
\[
\delta_S(y) \coloneqq \begin{cases}
0 & \text{ if } y \in S \\
+\infty & \text{ if } y  \not \in S,
\end{cases}, \qquad \sigma_S(y) \coloneqq \sup_{z \in S}\,\inProd{z}{y}.
\]
Next, let $f: \mathbb{R}^n \to \mathbb{R} \cup \{+\infty\}$ be a convex function. Its conjugate function is defined 
as
\[
f^*(s) \coloneqq \sup _{x \in \mathbb{R}^n} (\inProd{s}{x} - f(x)).
\]
We note that $\delta_S^* = \sigma_S$, for $S \subseteq \mathbb{R}^n$.

Since $C_n = P_n \cap B_n$, we  have
$\delta_{B_n} + \delta_{P_n} = \delta_{C_n}$.
Then, for every $y \in \mathbb{R}^n$, we have 
\begin{align*}
\sup_{x \in C_n} \, \inProd{y}{x} &= \sup _{x \in \mathbb{R}^n} \, (\inProd{y}{x} - \delta_{B_n}(x) -\delta_{P_n}(x)) \\
&= (\delta_{B_n}+\delta_{P_n})^*(y).
\end{align*}

Since $n \geq 2$, the (topological) interiors of $P_n$ and $B_n$ intersect. For example, $(0.6,0.6,\epsilon,\ldots, \epsilon)$ belongs to the interior of both sets for sufficiently small $\epsilon > 0$ when $n \geq 3$.
For $n =2$, it is enough to take $(0.6,0.6)$.
Under this condition,  a theorem from convex analysis says that $(\delta_{B_n}+\delta_{P_n})^*$ is the exact infimal convolution between 
$\sigma_{B_n}$ and $\sigma_{P_n}$, e.g., see Theorem~2.3.2 of Chapter X in \citep{HU93II}.
This means that 
\begin{equation}\label{eq:aux}
\sigma_{C_n}(y) = \sup_{x \in C_n} \, \inProd{y}{x} = \inf_{z \in \mathbb{R}^n}\{\sigma_{B_n}(z) + \sigma_{P_n}(y-z) \}
\end{equation}
and the infimum is attained for every $y$.
Now, for $z \in \mathbb{R}^n$ we have
\[
\sigma_{B_n}(z) = \norm{z}_2,
\]
which follows from the Cauchy-Schwarz inequality.
Defining $e\coloneqq (-1,-1,\ldots,-1)$, we have  for $w \in \mathbb{R}^n$
\begin{align*}
\sigma_{P_n}(w) &= \sup _{0\leq u, 1\leq u_1+\cdots u_n } \inProd{u}{w} = \inf_{0\leq et-w, 0\leq t } -t \\
&= \begin{cases}
\max _{i \in \{1,\ldots,n\}} w_i & \text{ if } w \in \mathbb{R}^n_-\\
+\infty & \text{ otherwise }
\end{cases},
\end{align*}

where the second equality follows from linear programming duality. The third equality holds because the constraint $0 \leq et - w $ implies that 
$-t \geq w_i$ for all $i$. So minimizing $-t$ under this constraint and the constraint that $t \geq 0$ leads to $\max _{i \in \{1,\ldots,n\}} w_i$ if $w \in \mathbb{R}^n_-$. If some component of $w_i$ is positive, then the problem is infeasible, so the infimum is $+\infty$.

Plugging the expressions for $\sigma_{B_n} $ and 
$\sigma_{P_n}$ into Equation~\eqref{eq:aux} leads to 
 \begin{equation*}
\sigma_{C_n}(y) = \inf_{z\in \Re^n}\{\norm{z}_2 + \max_{i\in\{1,\ldots,n\}}(y_i-z_i) \mid y\leq z \}.
 \end{equation*}

Therefore, in order to show that the proposition holds, it is enough to construct 
$x^* \in C_n$ and $z^* \in \mathbb{R}^n$ satisfying  $y \leq z^*$ and
\begin{align*}
\inProd{y}{x^*} &= \norm{z^*}_2 + \max_{i\in\{1,\ldots,n\}}(y_i-z_i^*) \\& = \begin{cases}
\norm{y^+}_2 & \text{ if } y \not \in \mathbb{R}^n_- \\
\max_{i \in \{1,\ldots,n\}}{y_{i}} & \text{ if } y  \in \mathbb{R}^n_-.
\end{cases}
\end{align*}
This can be done constructively case-by-case as follows.

\fbox{$(i)$ Suppose $y \not \in \mathbb{R}^n_-$.}
Then, at least one component of $y$ is positive, so $y^+ \neq 0$. 
Let $x^* \coloneqq  \frac{y^+}{\norm{y^+}_2}$ and
$z^* \coloneqq y^+$.

Then $x^*$ belongs to $C_n$, because $\norm{x^*}_2 = 1$ and $x_1^* + \cdots + x_n^* = \norm{x^*}_1 \geq \norm{x^*}_2 = 1$.
We also have $y - z^* = y^-\leq 0$, where $y^-$ is the nonpositive part of $y$.
Also, since at least one component of $y$ is positive
$y - z^*$ is a nonpositive vector with at least one entry equal to zero. So $\max_{i\in\{1,\ldots,n\}}(y_i-z_i^*)  = 0$.
Overall,
\[
\inProd{y}{x^*} = \norm{z^*}_2 + \max_{i\in\{1,\ldots,n\}}(y_i-z_i^*) = \norm{y^+}_2.
\]

\fbox{$(ii)$ Suppose $y \in \mathbb{R}^n_-$.}
Let $j$ be an index of $y$ associated to 
its largest component\footnote{There may be multiple $j$'s, but any will work.}. 
We let $x^* \in \mathbb{R}^n$ be such that $x^*_j \coloneqq 1$ and $x^*_k \coloneqq 0$ for $k \neq j$. 
We have $\norm{x^*}_2 = 1$ and $x^*_1 + \cdots x^*_n = x^*_j =1$, so $x^* \in C$.
Finally, let $z^* \coloneqq 0$. With that, since $y \in \mathbb{R}^n_-$, we have $y -z^* = y \leq 0$ and
\[
y_j = \inProd{y}{x^*} = \norm{z^*}_2 + \max_{i\in\{1,\ldots,n\}}(y_i-z_i^*) = \max_{i\in\{1,\ldots,n\}}y_i. 
\]

\end{proof}

\subsection{Convex Optimization }\label{sec:convex}
Let $\Kmc = (\Tmc,\Amc)$ %much more common to see \Tmc,\Amc than \Amc,\Tmc in the literature
be a DL-Lite$^\Hmc$ KB, %without role disjointness 
 let $d$ be the %\todo{AO: word target is needed here? not sure if this was explained bef}
embedding dimension and let $\worldSize$ and the $\epsilon > 0$ be as in Section~\ref{sec:boxSemantics}.
Following Section~\ref{sec:boxSemantics}, a given box interpretation $\eta$ associates to each individual name $a\in\NI$, each concept name $C\in\NC$ and each role name $R\in\NR$ the following objects: two vectors $(\posAssign{a}, \bumpAssign{a}) \in 
\World\times\World$,
%$\Re^d\times \Re^d$, 
a box $\cAssign{C}$ and three boxes $(\headAssign{R}, \tailAssign{R}, \bumpBoxAssign{R})$, respectively. 
Each box   is parameterized by two vectors in $\Re^d$ representing lower and upper bounds.
%Besides $d$, the geometric model has two extra constants that must be selected, $\worldSize$ and the $\epsilon > 0$, see Section~\ref{sec:boxSemantics}. %that appears in \eqref{eq:concept_cons} and \eqref{eq:indUn}. \todo{Update this part}
%The choice of $\worldSize$ is less consequential, since changing it merely affects the scale of the model.
%We will discuss the choice of $\epsilon$ shortly.

%Given $d, \epsilon, \worldSize$, a box interpretation $\eta$ is determined by the parameters in the geometric model (lower/upper bounds of boxes, positions of individuals, etc.).
%Suppose we concatenate all the parameters of the box interpretation into a single vector $z$ of dimension 
%$n := (2|\NI|+2|\NC|+6|\NR|)d$, where $|S|$ denotes the cardinality of a set $S$. 
%In this way, we can reconstruct $\eta$ from $z$ and vice-versa.

 We recall that we concatenate all the parameters of $\eta$ into a single vector $z$ of dimension 
$n := (2|\NI|+2|\NC|+6|\NR|)d$.
With that, let $\convexSet_{\Kmc} \subseteq \Re^n$ be the set of $z$'s such that the constraints defined in Section~\ref{sec:ProblemFormulation} are satisfied. That is, $z \in \convexSet_{\Kmc}$ if and only if the $\eta$ corresponding to $z$ is such that:
%$(a)$ better to avoid (a) in mathmode because a in mathmode is used as ind name 
for each TBox axiom the corresponding inequalities are satisfied; and %$(b)$ 
the box consistency and universe constraints are satisfied.
%\begin{enumerate}[$(a)$]
  %  \item 
 %   \item the box consistency and universe constraints are satisfied (see Section~\ref{sec:consistency}), i.e., individuals within the universe (see \eqref{eq:indUn} and \eqref{eq:conUn}), {and concepts satisfy the inequality in~\eqref{eq:concept_cons}}.
%\end{enumerate}
Also, let
$f_{\hyperP}:\Re^n\to \Re$ be the function that maps $z$ to the objective value described in Section~\ref{sec:Scoring} for a given choice of 
nonnegative hyperparameters $\hyperP = (\lambda_1,\lambda_2,\lambda_3) \in \Re^3_+$.%, where $\lambda_1$ is the hyperparameter associated to the loss terms and $\lambda_2,\lambda_3,\lambda_4, \lambda_5, \mu$ are associated to regularization terms.

In what follows, we recall that a set in $\Re^n$ is said to be \emph{polyhedral} if it  can be written as the set of solutions of finitely many linear equalities/inequalities.
%In particular, polyhedral sets are convex.

\theoremConvex*
\begin{proof}
As described in Section~\ref{sec:ProblemFormulation}, each   inclusion in the TBox of the KB is translated into finitely many linear inequalities in terms of the parameters of the box interpretation, which are exactly the components of $z$. 
Similarly, for each individual, concept and role, the box consistency and universe constraints in Section~\ref{sec:ProblemFormulation} are translated to finitely many linear inequalities in $z$.
{Since \NI, \NC, \NR, and the TBox  are all finite}, $\convexSet_{\Kmc}$ is the intersection of solution sets of finitely many linear inequalities. 
Thus, $\convexSet_{\Kmc}$ is a polyhedral set in $\Re^n$.

Next, we move on to the convexity of $f_{\hyperP}$ . First, we recall that the composition of a convex function with an affine function is still convex. Also, sums of convex functions are convex. Similarly, 
the maximum of convex functions is also convex.
% and adding/subtracting a \emph{linear} function to a convex function  preserves convexity. 
See \cite[Section~3.2]{BV04} for a review of calculus rules for convex functions.

First, the signed distance function 
$\sdist(\cdot, \Re^{2d}_-)$ is convex, see Theorem~10.1 in Chapter 7 of \citep{DZ11} or
pg.154 in \citep{HU93I}.
With that, 
the ${\sf dist}$ function is also convex, as it is the composition of a convex function with an affine function.
This implies that 
the concept assertion loss $\mathcal{L}_{\sf concept}$ is convex as well as it is again a composition of a convex function with an affine function.
Similarly, 
both the role assertion loss $\mathcal{L}_{\sf role}$ and the negative concept regularization $\mathcal{L}_{\sf negative}$ are convex, as they are the maximum of finitely many convex functions.
Finally, the box width regularization $\mathcal{R}_{\sf width}$ is also convex, since it is the composition of a convex function (the $2$-norm) with a linear function.

Overall the objective function $f_{\hyperP}$ is convex as it is obtained by taking sums and maximums of finitely many convex functions.

So far, we have shown that the problem in \eqref{eq:opt_ont} is a convex optimization problem with polyhedral constraints. Next, we move on to the proofs of the items.

\fbox{$i)$} %
%Rigorously speaking, $\convexSet_{\Kmc}$ is also a function of $d, \epsilon$ and $\worldSize = {\bf{4}}_d \coloneqq (4,\ldots, 4)$, so for the remainder of the proof of this item, we will use the extended notation $\convexSet_{\Kmc}(d,\epsilon, {\bf{4}}_d)$ to convey this fact.
%With that, we observe that for a fixed $\epsilon$, if $\convexSet_{\Kmc}(d,\epsilon, {\bf{4}}_d)$ is nonempty for a certain $d$, then
%$\convexSet_{\Kmc}(\hat d,\epsilon, {\bf{4}}_{\hat d})$ is also nonempty for $\hat d > d$, since we are just adding dimensions to the model. 
%{In particular, a solution to the latter can be obtained from a solution to the former by adding zeroes.}
%Therefore, it is enough to prove item~$(i)$ for $d \coloneqq (2|\NR| + 1)(|\NC|+2|\NR|)$, which we will do as follows.
Let $\tilde\Kmc$ be the KB which coincides with $\Kmc$ except for the fact that its ABox is empty.
By item~$i)$ of Corollary~\ref{cor:ConsistencyDim}, for every $d \geq d_{\min}$, there exists a box interpretation $\eta$ that is weakly TBox faithful. 
Because of the faithfulness of the embedding, $\eta$ satisfies the constraints described in 
Section~\ref{sec:ProblemFormulation}. 

%Next, we check that $\eta$ also satisfies the constraints described in  
%Section~\ref{sec:ProblemFormulation}.
First we recall that the proof of Corollary~\ref{cor:ConsistencyDim} is done by constructing $\eta$ following Definition~\ref{def:mapping} with $\worldSizeScalar = 4$ (as in the proof of Theorem~\ref{thm:faith_bool_alc}), see also Table~\ref{tab:RoleBoxes}.
This allows 
us to invoke Theorem~\ref{thm:propembedding}, which ensures that $\eta$ satisfies 
\eqref{eq:concept_cons} in Section~\ref{sec:ProblemFormulation}.
Furthermore, because $\eta$ is a box interpretation, the inequalities in \eqref{eq:conUn} and \eqref{eq:indUn} (Section~\ref{sec:ProblemFormulation}) must be satisfied, see Definition~\ref{def:boxinterpretation}.
This is because Definition~\ref{def:boxinterpretation} imposes the same restrictions on the widths of the boxes.
%all the boxes associated to concept names and role names must satisfy \eqref{eq:conUn}.
%Similarly, the positions of the individuals and their bumps must satisfy \eqref{eq:indUn}.

%First, we invoke Theorem~\ref{thm:faith_bool_alc} for the KB $\tilde\Kmc$, from which we can infer the existence of a weakly TBox
%faithful (with respect to $\tilde \Kmc$) box interpretation $\eta$ for any $d$ as in item~$(i)$ of Corollary~\ref{cor:ConsistencyDim}. 
%The proof of \todo{double check with current}Theorem~\ref{thm:faith_bool_alc} is done for $\worldSize = {\bf{4}}_d$ and $\epsilon \in (0, \epsilonMax]$, see also Tables~\ref{tab:RoleBoxes} and Section~\ref{sec:boxSemantics}.
%Because of the faithfulness of the embedding, the constraints described in 
%Section~\ref{sec:TBox_Translation} are satisfied. 
 %{
%Next, we check that each constraint in Section~\ref{sec:consistency} is satisfied.
%The constraints in \eqref{eq:concept_cons}, \eqref{eq:conUn}, and \eqref{eq:indUn} are satisfied because of Theorem~\ref{thm:propembedding}.}

% \begin{itemize}
    % \item The constraint in \eqref{eq:concept_cons} is satisfied because for each $C \in \NCE$, the corresponding lower bound and upper bound at dimension $i_C$ is set to $-4$ and $-0.5$ respectively, see Table~\ref{tab:ConceptBoxes}.
    % So, we have $\frac{\mathbf{L_{C}}[i_C] + \mathbf{U_{C}}[i_C]}{2} = -2.25 \leq -\worldSizeScalar/2 = -2 $
    % \item The constraints in \eqref{eq:concept_cons}, \eqref{eq:conUn}, and \eqref{eq:indUn} are satisfied because of Theorem~\ref{thm:propembedding}.
% \end{itemize}
Finally, aggregating 
the parameters of $\eta$ into a single vector $z \in \Re^n$ in an appropriate order, 
we will have $z \in \convexSet_{\Kmc}$, which implies 
that $\convexSet_{\Kmc}\neq \emptyset$.

\fbox{$ii)$}
The proof of item~$ii)$ follows from the definition of $\convexSet_{\Kmc}$. We recall that $z \in \convexSet_{\Kmc}$ if and only if the corresponding inequalities in Section~\ref{sec:ProblemFormulation} are satisfied. 
As in the proof of Theorem~\ref{thm:propembedding}, the inequality in \eqref{eq:concept_cons} (Section~\ref{sec:ProblemFormulation}) implies box consistency. 
We conclude that the box interpretation corresponding to a given $z \in {\convexSet}_{\Kmc}$ is box consistent and satisfies all the TBox axioms of the underlying KB. 
By Proposition~\ref{lem:tboxfaith}, the {box interpretation}
associated to $z$ is  TBox faithful, which concludes the proof.

\fbox{$iii)$}
Since $z^*$ is assumed to be an optimal solution, we have $z^* \in \convexSet_{\Kmc}$. By item~$ii)$, the box interpretation corresponding to $z^*$ satisfies all the TBox axioms and must be box consistent.
Because $f_{\hyperP}(z^*) \leq 0$ holds, $\lambda_1$ is zero
and the regularization terms associated to $\lambda_2, \lambda_3$ are nonnegative, we conclude that the first max term of the objective function term is nonpositive. 
 In particular, all the $\mathcal{L}_{\sf concept}$ and $\mathcal{L}_{\sf role}$ terms inside the max must be nonpositive.
 By the definition of the signed distance function $\sdist(\cdot, \Re^{2d}_-)$, this implies that the corresponding ABox axioms are satisfied as well. By Proposition~\ref{lem:kbfaith}, the box interpretation associated with $z^*$  is    KB faithful.

\fbox{$iv)$} %
%For this item, we use the same notation as in the proof of item~$(i)$. 
%Similarly, it is enough to prove 
%the result for $d \coloneqq |\NC| + 2|\NR| + |\NR| (|\NC| + 2|\NR| + |\NI|)$. The reason is that if $z^*$ is an optimal solution with value zero for the problem of minimizing $f_{0,0}$ over $\convexSet_{\Kmc}(d,\epsilon, {\bf{4}}_d)$, then 
%if we replace $d$ with $\hat d \geq d$, we can  obtain from $z^*$ a new optimal solution $\bar{z}^*$ with value zero by {\color{red}padding the extra dimensions with zeroes}. 
%\todo[inline]{Invoke %Corollary~\ref{corollary:FaithfulnessDim} and Theorem~\ref{thm:propembedding}}
%First we invoke Theorem~\ref{thm:faith_bool_alc} for the KB $\Kmc$, from which infer the existence of a weakly faithful box interpretation $\eta$ for any  dimension
By item~$i)$ of Corollary~\ref{corollary:FaithfulnessDim}, there exists a weakly KB faithful box interpretation $\eta$ of $\Kmc$ for any 
$d \geq d_{\min}$.
As in the proof of item~$i)$, $\eta$ satisfies 
the constraints described in Section~\ref{sec:ProblemFormulation} for $\worldSizeScalar = 4$, which comes as consequence of Theorem~\ref{thm:propembedding}, the proof of Corollary~\ref{corollary:FaithfulnessDim} and the definition of box interpretation in Definition~\ref{def:boxinterpretation}.

%$d $ as in . 
%The proof of Theorem~\ref{thm:faith_bool_alc} is done for $\worldSize = {\bf{4}}_d$ and $\epsilon \in (0, \epsilonMax]$, see also Tables~\ref{tab:RoleBoxes} and Section~\ref{sec:boxSemantics}.
%As in the proof of item~$(i)$, $\eta$ satisfies all the constraints in Sections~\ref{sec:TBox_Translation} and \ref{sec:consistency}, see also Theorem~\ref{thm:propembedding}.

Aggregating 
the parameters of the box interpretation 
$\eta$ %obtained in the proof Theorem~\ref{thm:faith_bool_alc} \todo{AO: maybe we need to say something else here because dim of each item of cor change a bit}
into a single vector $z^* \in \Re^n$ in an appropriate order, 
we have $z^* \in \convexSet_{\Kmc}$.
Because $\eta \models \Kmc$, all the ABox axioms are satisfied, so the the $\mathcal{L}_{\sf concept}$ and $\mathcal{L}_{\sf role}$ terms in 
in first max term of the objective function must be nonpositive.
As a consequence, if the regularization parameters $\lambda_1,\lambda_2, \lambda_3$ in the objective function term  are zero as well, then    $f_{\hyperP}(z^*) \leq 0$ holds.
\end{proof}

%Informally, Theorem~\ref{theo:convex} states that we can find a geometric embedding for a {satisfiable  DL-Lite$^\Hmc$ KB}  via convex optimization over a polyhedral set.  
%Any feasible solution $z$ to \eqref{eq:opt_ont} (whether optimal or not) will correspond to a box interpretation that is guaranteed to be at least weakly TBox faithful and also box consistent. 
%Item~$(i)$ gives an upper bound to the minimum required $d$ to ensure that $\convexSet_{\Kmc}$ is nonempty, but this estimate seems to be quite conservative.
%Furthermore, items~$(iii)$ and $(iv)$ imply that  if we wish to find a solution that is weakly KB faithful, this can be done by setting the hyperparameters associated to regularization terms to $0$, setting $d$ to be sufficiently large and solving the corresponding convex optimization problem in \eqref{eq:opt_ont}.
\paragraph{Second-order cone representability.}
For those  familiar with conic optimization, the proof of Theorem~\ref{theo:socp} is routine and can be summarized as follows. We employ the notion of the \emph{epigraph} of a function $f: \Re^s \to \Re\cup \{+\infty\}$, defined as 
the set $\{(x,t) \in \Re^s \times \Re \mid f(x) \leq t\}$.
With that, the  {epigraphs} of the functions corresponding to each of the terms appearing in
$f_{\hyperP}$ are \emph{second-order cone representable} (SOCr \citep{LVBL98}, also called CQr in \cite[Lecture~3]{BtN01}). Furthermore, {second-order cone representability} is preserved by adding functions, composition with affine functions and taking a maximum of finitely many SOCr functions, see \cite[Section~3.3]{BtN01}.
Overall, applying appropriate calculus rules, we see that the epigraph of $f_{\hyperP}$ is SOCr, which implies in particular that the problem in \eqref{eq:opt_ont} has a SOCP formulation.
For the sake of self-containment, we present a detailed proof.

%\todo[inline]{Ana: da uma lida ai}
\theoremSOCP*
``Equivalent'' in the statement of Theorem~\ref{theo:socp} means that the optimal value and optimal solutions from the former can be recovered from the optimal value and optimal solutions to the latter and vice-versa.

\begin{proof}
In its most general form, a SOCP  can be written as
\begin{align}
\min_{y \in \Re^s} &\quad c^T y    \label{eq:socp}\\
\text{subject to}&\quad y \in \mathcal{P} \notag\\
&\quad b_i +A_iy \in \POC{2}{n_i}, \qquad i = 1\ldots, m,\notag
\end{align}
where each $\POC{2}{n_i} \coloneqq \{(t,z) \in \Re \times \Re^{n_i-1} \mid t \geq \norm{z}_2 \}$ is the second-order cone in $\Re^{n_i}$,
$\mathcal{P} \subseteq \Re^s$ is a polyhedral set described via finitely many linear equalities/inequalities\footnote{Depending on the reference, the ``standard form'' of SOCPs may not include linear inequalities directly, but this can be bypassed easily since a linear inequality of the form ``$a^T y \geq d$'' is equivalent to the SOC constraint ``$a^T y -d \in \POC{1}{2}$''. }, the $b_i$'s are vectors, $A_i$'s are linear maps, $c \in \Re^s$ is a fixed vector and $c^T y $ indicates the Euclidean inner product between $c$ and $y$ so that $c^T y = \sum _{i=1}^s c_i y_i$ holds.

%We now give a more detailed explanation.
Comparing the SOCP in \eqref{eq:socp} with the problem in \eqref{eq:opt_ont} and recalling that $\mathcal{C}_{\Kmc}$ is polyhedral, we see that \eqref{eq:opt_ont} is not a SOCP only because its objective function is nonlinear. 
Here, we will use the common optimization trick of ``dropping the objective to the constraints'', e.g., see \cite[Section~2.5]{LVBL98} or \cite[Chapter~3]{BtN01}. 
The idea is as follows, if $g:\Re^s\to \Re$ is a real function and $S \subseteq \Re^s$, then the problem ``$\min_{y} g(y)\,\, \text{subject to}\,\, y \in S$'' is equivalent to ``$\min _{t,y} t \,\,\text{ subject to }\,\, g(y) \leq t, y \in S $''.
Similarly, if there were several functions $g_j$ 
we would have that 
the problem  \[\min_{y} \sum_{j=1}^{\ell}g_j(y)\,\, \text{subject to}\,\, y \in S\] is equivalent to 
\begin{equation}\label{eq:socp_example}
\min _{t_1,\ldots, t_\ell,y} \sum _{i=1}^\ell t_j \,\,\text{ subject to }\,\, g_j(y) \leq t_j \,\,(j = 1, \ldots, \ell),\, y \in S
\end{equation}
Here, the constraints ``$g_j(y) \leq t_j$'' simply mean that $(y,t_j)$ belongs to the epigraph of $g_j$.
In particular, if $S$ and the epigraphs of $g_j$ can be represented via finitely many linear equalities/inequalities and SOC constraints, then the problem in \eqref{eq:socp_example} can reformulated as a SOCP since its objective function is linear. 
Informally, we say that a function is SOCr if its epigraph can be represented via finitely many equalities/inequalities and SOC constraints (adding auxiliary variables if necessary).

This discussion provides a blueprint for proving that 
the problem in \eqref{eq:opt_ont} can be reformulated as a SOCP. 
The objective function $f_{\hyperP}$ is a sum of four terms: the loss terms associated
to concept and role assertions which are aggregated with a max, the negative sampling component, and two terms for width regularization.
For each term that appear, we add one auxiliary variable $t_j$, we ``drop the objective function terms to the constraints'' and we   argue that each resulting constraint can be written in terms of linear equalities/inequalities and SOC constraints, i.e., the epigraphs of the functions are SOCr. 
Naturally, if $g$ is SOCr then the composition of $g$ with an affine function is SOCr, e.g., see \cite[Remark~3.3.1]{BtN01}\footnote{It is enough to observe that if $h$ is an affine function of the form $h(x) = Bx + d$, for $B$ a linear map and $d$ is a vector, we can obtain a SOC representation of the epigraph of $g \circ h$ by adding the 
linear constraint $y = Bx+ d$ to a SOC representation of $\{(y,t) \mid g(y) \leq t\}$ (the epigraph of $g$).}.

The building blocks for the loss term and regularization terms in $f_{\hyperP}$ boil down to two functions: 
the signed distance function $\sdist(\cdot, \Re^{2d}_-)$ and the 2-norm function 
$\norm{\cdot}_2$.
All the four terms in $f_{\hyperP}$ are obtained by composing copies of  $\sdist(\cdot, \Re^{2d}_-)$ and $\norm{\cdot}_2$ with affine functions and either adding them together or taking maximums.
In view of our discussion so far, it is enough to establish that the epigraphs of these two functions can be written in terms of linear equalities/inequalities and SOC constraints. In other words, we need to show that both functions are SOCr.

First, we will show that $\sdist(\cdot, \Re^{2d}_-)$ is SOCr. For that, we define 
the auxiliary function $\psi:\Re^{2d} \times \Re^{2d} \to \Re \cup \{+\infty\}$ such 
that 
\[
\psi(y,z) \coloneqq \begin{cases}
    \norm{z}_2 + \max_{i\in\{1,\ldots,2d\}}(y_i-z_i) & \text{if } y \leq z\\
    +\infty & \text{otherwise.}
\end{cases}
\]
The function $\psi(y,z)$ is SOCr because 
$\psi(y,z) \leq t$ holds for $t \in \Re$ if and only if 
there exists $t_1,t_2 \in \Re$ such that
\begin{align*}
t_1 + t_2 &\leq t,\\
(t_1,\norm{z}_2) &\in \POC{2}{2d+1},\\
y_i-z_i & \leq t_2, \quad \forall i \in \{1,\ldots,2d\}\\
y &\leq z.
\end{align*}
Furthermore, Proposition~\ref{prop:sup_sign}
implies that $\sdist(\cdot,\Re^{2d}_-)$ is the partial minimization of $\psi$ with respect to the second argument, i.e., 
\[
\sdist(y, \Re^{2d}_-) = \inf_{z \in \Re^{2d}}\psi(y,z)
\]
holds for every $y \in \Re^{2d}$. Furthermore, the proof of Proposition~\ref{prop:sup_sign} shows that
for every $y$, there exists $z_y$ such that $\sdist(y, \Re^{2d}_-) = \psi(y,z_y)$ holds, i.e., the infimum is achieved for every $y$.
The property of being SOCr is preserved by partial minimization assuming that for every $y$ the infimum is achieved, e.g., see Section~3.3 in \citep{BtN01}. Therefore, $\sdist(\cdot, \Re^{2d}_-)$ is SOCr as well.

Next, the function $\norm{\cdot}_2$ is also SOCr, since $\norm{y}_2 \leq t$ holds for $y \in \Re^d, t \in \Re$ if and only if $(t,y)  \in \POC{2}{d}$.

In conclusion, all the functions used to build the objective function $f_{\hyperP}$ are SOCr, so overall, 
the problem in \eqref{eq:opt_ont} has a SOCP formulation.
\end{proof}

%Writing down the precise SOCP formulation of \eqref{eq:opt_ont}, although straightforward, is quite tedious. 
%Fortunately, 

\color{black}

\section{Size of the final optimization problem}
\label{app:additional_analyses}
As mentioned in Section~\ref{sec:exp}, a problem modelled through CVXPY is first compiled and then sent to a solver such as MOSEK.
Here we report in Table~\ref{tab:Problem_Characteristics} the number of variables and constraints of the final optimization problem that CVXPY outputs to MOSEK.

\begin{table}[h!]
\centering
\begin{tabular}{ccccc}
\toprule
Dataset & \#Variables & \#Constraints \\
\midrule
F\_v1        & 461k  & 300k \\
F\_v2        & 725k  & 475k   \\
F\_v3        & 1312k  & 863k   \\
F\_v4        & 3584k & 2367k  \\
%F\_v5        & 327k & 394k & 1,632k  \\
\bottomrule
\end{tabular}
\caption{Problem sizes of the final optimization problem solved by MOSEK split by dataset.}
\label{tab:Problem_Characteristics}
\end{table}

\section{Experimental Details}
\label{app:expDetails}

This section describes our experimental setup, created benchmark datastes, and evaluation protocol in detail.
In particular, Section~\ref{app:ImplementationAndReproducibility} contains details about the implementation of \modelName. 
Furthermore, Section~\ref{app:BenchData} discusses the properties of the created benchmark datasets F\_v1-4. 
Next, Section~\ref{app:ExperimentSetup}, describes our experimental setup, including a detailed description of the learning setup, used hardware, and selected hyperparameters.
Continuing from that, Section~\ref{app:Metrics} discusses the evaluation protocol and metrics in detail. 
Moreover, Section~\ref{app:time} shows the runtime of BoxLitE and the SGD solvers over the various Family dataset subsets.
Finally, Section~\ref{app:SGD_Hyperparameters} details the hyperparameters used for the SGD models we have shown in our experiments.

\subsection{Implementation   \& Reproducibility}
\label{app:ImplementationAndReproducibility}

We   implemented the \modelName (convex) second-order cone optimization problem in Python 3.12 using CVXPY %1.4.1 
for formulating the problem and MOSEK %Gurobi 10.0.2 
for optimizing it. %To make our findings reproducible, we include  \modelName's code, the created datasets F\_v1-4, a ReadMe.md file listing library dependencies, installation, and execution instructions as part of the supplemental material. 
The seed used for the SGD models is $6934$. %We will make all provided supplementary materials (including \modelName's code base and the created datasets) publicly available %on GitHub 
%upon acceptance of our paper.

\subsection{Details on F\_v1-4}
\label{app:BenchData}

This section contains details about the created benchmark datasets F\_v1-4 of Section~\ref{sec:exp}. 
In particular, we have created a set of datasets (F\_v1-4) of varying sizes from the family dataset \citep{FamilyDataset}. 
We derived these datasets by sampling approximately $k \in \{300, 500, 1000, 3000\}$ %$, 5000\}$ 
assertions of the family dataset's ABox with forest fire sampling \citep{ForestFireSampling}, 
a popular sampling technique for large graphs, using a forward burning probability (${\sf pf} = 0.7$) and a backward burning probability (${\sf bf} = 0$). Furthermore, since the family dataset solely provides role assertions in its ABox, 
we selected any TBox inclusion of the family dataset that includes solely roles and extended the TBox by the disjointness axioms $\exists {\sf hasFather^-} \sqsubseteq \neg \exists {\sf hasMother^-}$. Figure~\ref{fig:familyTBox} lists the TBox of the created datasets F\_v1-4.

Next, we created a set of inferred assertions by $(i)$ adding any assertion that logically follows from each dataset and $(ii)$ removing any assertion that occurs in the ABox. We randomly split this set of inferred assertions 
into a validation set (20\%), used for model selection, and a test set (80\%), used for evaluating the performance of the selected model. 
%We provide the train, validation, and test splits for datasets F\_v1-5 in the supplementary material. 

% We evaluate \modelName's performance on $(i)$ the ABox, revealing how well the embedding solutions fit the training data (for Q1) and $(ii)$ the test set, testing how well the embedding solutions learn to reason over the ontology (for Q2). We use the standard evaluation setting for KB completion\footnote{%We optimize our model on the train, select the best model using the validation, and evaluate its performance on the test set. 
% The evaluation of a KB embedding model typically needs a set of true and corrupted role assertions. True role assertions $R(a,b)$ of the KB are corrupted by replacing $a$ or $b$ by any $c \in \NI$ such that the corrupted assertion is not within the KB. %KB embeddings are optimized to score true assertions higher than false ones, thereby estimating a given assertion's plausibility. 
% The performance of KB embedding models is typically measured using the filtered versions \citep{TransE} of the \emph{mean reciprocal rank} (MRR), the average of inverse ranks ($1/\textit{rank}$) and H@k, the proportion of true assertions within the predicted assertions whose rank is at maximum k.} \citep{BoxE,ExpressivE,BoxEL}.

\subsection{Experimental Setup}
\label{app:ExperimentSetup}

\textbf{Training Setup.} 
%We implemented \modelName's (convex) second-order cone optimization problem in Python using CVXPY \citep{DB16,AVDB18} for modeling and Gurobi \citep{Gurobi23} for solving it. We ran each of our experiments on an Intel(R) Xeon(R) Silver 4314 CPU @ 2.40GHz of our internal cluster, using up to $32$ threads. In particular, 
During the training phase, we optimized the loss described in \ref{sec:ProblemFormulation} using MOSEK on the train set. After retrieving an embedding solution from MOSEK, we evaluated its performance on the validation set, which we used for selecting the best embedding solution (see Section~\ref{app:Metrics} for more details on the evaluation protocol). %Finally, we evaluated \modelName's best embedding solutions on the train set (the ABox), to quantify how well they learn to represent each KB; and on the test set, to quantify how well they learn to reason over the assertions of each KB (see Section~\ref{sec:Exp_Assertion_Experiments} for the results). 
We now discuss \modelName's hyperparameter optimization.

\textbf{Hyperparameter Optimization.} 
We set  $\worldSizeScalar = 1$ and $\epsilon = 10^{-2}$ for all of our experiments. 
For the benchmark results on datasets F\_v1-4, 
%of Sections~\ref{sec:Exp_Assertion_Experiments}, \ref{app:Compile_Solve_Time}, and \ref{sec:Problem_Characteristics}, 
we %manually \todo{what about: We tuned the hyperparameters...?}
set $d=32$ and
tuned the hyperparameters within the following ranges: %$(i)$ %\todo{do we need these i, ii , etc. here? it makes reading more difficult}
$\lambda_1,\lambda_2, \lambda_3 \in \{0, 0.001, 0.003,0.1, 0.3, 1, 3\}$. 
%\{10^p \mid p \in \{0, 2, 4, 6\} \}$, %$(ii)$
%$\lambda_2 \in \{10^p \mid p \in \{-2, 0, 2, 4, 6\}\}$, and %$(iii)$
%$\lambda_3 \in \{10^p \mid p \in \{0, 2, 4, 6\} \}$. %$(iv)$
%$\lambda_4 \in \{10^p \mid p \in \{0, 1\}\}$, %$(v)$
%$\lambda_5 \in \{10^p \mid p \in \{0, 1\}\}$, %$(vi)$
%$\lambda_6 \in \{0, 1\}$, 
%$\alpha \in \{10^p \mid p \in \{-2, -1, 0\}\}$
%and %$(vii)$
%$\mu \in \{10^p \mid p \in \{-1, 0, 1, 2\} \}$. 
We list the best found hyperparameters for each of the datasets in Table~\ref{tab:HyperParameters}. 
%\todo{update with new values for fv4}

\begin{table}[h!]
    \centering
    %\resizebox{\columnwidth}{!}
    {%
    \begin{tabular}{clllllllll}
        \toprule
         Dataset & $\lambda_1$ & $\lambda_2$ & $\lambda_3$ \\
         %& %$\lambda_4$ & $\lambda_5$ & $\lambda_6$ & $\alpha$ & $\mu$ \\
         \midrule
         F\_v1 & $.001$ & $3$ & $.001$
 %        0.003	0.1	0.003
         \\
         %& $10^{1}$ \\
         %& $10^{1}$ & $0$ & $10^{-2}$ & $10^{-1}$ \\
         F\_v2 & $.001$ & $1$ & $.003$ \\
         %& $10^{1}$ & $10^{1}$ & $0$ & $10^{-2}$ & $10^{-1}$ \\
         F\_v3 & $.003$ & $0.1$  & $.003$ \\
         %& $10^{1}$ & $10^{1}$ & $0$ & $10^{0}$ & $10^{-1}$ \\
        F\_v4 & $.001$ & $3$ & $.003$ \\
         %& $10^{1}$ & $10^{1}$ & $0$ & $10^{-2}$ & $10^{2}$ \\
        % F\_v5 & $10^{6}$ & $10^{2}$  & $10^{2}$ & $10^{1}$ & $10^{1}$ & $0$ & $10^{-2}$ & $10^{0}$ \\
        \bottomrule         
    \end{tabular}
    }
    \caption{Best found hyperparameters   for \modelName.}
    \label{tab:HyperParameters}
\end{table}

\begin{table}[h!]
    \centering
    %\resizebox{\columnwidth}{!}
    {%
    \begin{tabular}{clllllllll}
        \toprule
         Dataset & $\lambda_1$ & $\lambda_2$ & $\lambda_3$ \\
         %& %$\lambda_4$ & $\lambda_5$ & $\lambda_6$ & $\alpha$ & $\mu$ \\
         \midrule
         F\_v1 & $.0$ & $3$ & $.001$
 %        0.003	0.1	0.003
         \\
         %& $10^{1}$ \\
         %& $10^{1}$ & $0$ & $10^{-2}$ & $10^{-1}$ \\
         F\_v2 & $.0$ & $3$ & $.03$ \\
         %& $10^{1}$ & $10^{1}$ & $0$ & $10^{-2}$ & $10^{-1}$ \\
         F\_v3 & $.0$ & $.3$  & $.0$ \\
         %& $10^{1}$ & $10^{1}$ & $0$ & $10^{0}$ & $10^{-1}$ \\
         F\_v4 & $.0$ & $.3$  & $0$ \\
         %& $10^{1}$ & $10^{1}$ & $0$ & $10^{-2}$ & $10^{2}$ \\
        % F\_v5 & $10^{6}$ & $10^{2}$  & $10^{2}$ & $10^{1}$ & $10^{1}$ & $0$ & $10^{-2}$ & $10^{0}$ \\
        \bottomrule         
    \end{tabular}
    }
    \caption{Best found hyperparameters  for BoxLitE1.}
    \label{tab:HyperParameters1}
\end{table}

\begin{table}[h!]
    \centering
    %\resizebox{\columnwidth}{!}
    {%
    \begin{tabular}{clllllllll}
        \toprule
         Dataset & $\lambda_1$ & $\lambda_2$ & $\lambda_3$ \\
         %& %$\lambda_4$ & $\lambda_5$ & $\lambda_6$ & $\alpha$ & $\mu$ \\
         \midrule
         F\_v1 & $.003$ & $.0$ & $.0$
 %        0.003	0.1	0.003
         \\
         %& $10^{1}$ \\
         %& $10^{1}$ & $0$ & $10^{-2}$ & $10^{-1}$ \\
         F\_v2 & $.01$ & $.0$ & $1$ \\
         %& $10^{1}$ & $10^{1}$ & $0$ & $10^{-2}$ & $10^{-1}$ \\
         F\_v3 & $.01$ & $.0$  & $.003$ \\
         %& $10^{1}$ & $10^{1}$ & $0$ & $10^{0}$ & $10^{-1}$ \\
          F\_v4 & $.3$ & $.0$  & $.003$ \\
         %& $10^{1}$ & $10^{1}$ & $0$ & $10^{-2}$ & $10^{2}$ \\
        % F\_v5 & $10^{6}$ & $10^{2}$  & $10^{2}$ & $10^{1}$ & $10^{1}$ & $0$ & $10^{-2}$ & $10^{0}$ \\
        \bottomrule         
    \end{tabular}
    }
    \caption{Best found hyperparameters   for BoxLitE2.}
    \label{tab:HyperParameters2}
\end{table}

\begin{table}[h!]
    \centering
    %\resizebox{\columnwidth}{!}
    {%
    \begin{tabular}{clllllllll}
        \toprule
         Dataset & $\lambda_1$ & $\lambda_2$ & $\lambda_3$ \\
         %& %$\lambda_4$ & $\lambda_5$ & $\lambda_6$ & $\alpha$ & $\mu$ \\
         \midrule
         F\_v1 & $.003$ & $.1$ & $.0$
 %        0.003	0.1	0.003
         \\
         %& $10^{1}$ \\
         %& $10^{1}$ & $0$ & $10^{-2}$ & $10^{-1}$ \\
         F\_v2 & $.003$ & $1$ & $.0$ \\
         %& $10^{1}$ & $10^{1}$ & $0$ & $10^{-2}$ & $10^{-1}$ \\
         F\_v3 & $.003$ & $0.1$  & $.0$ \\
         %& $10^{1}$ & $10^{1}$ & $0$ & $10^{0}$ & $10^{-1}$ \\
         F\_v4 & $.003$ & $.1$  & $.0$ \\
         %& $10^{1}$ & $10^{1}$ & $0$ & $10^{-2}$ & $10^{2}$ \\
        % F\_v5 & $10^{6}$ & $10^{2}$  & $10^{2}$ & $10^{1}$ & $10^{1}$ & $0$ & $10^{-2}$ & $10^{0}$ \\
        \bottomrule         
    \end{tabular}
    }
    \caption{Best found hyperparameters  for BoxLitE3.}
    \label{tab:HyperParameters3}
\end{table}
%Furthermore, for the embedding (Figure~\ref{fig:StrongFaithFulnessPlot}) of the prototypical ontology (Figure~\ref{fig:familyOnto}), we set $d=30$ and used the following hyperparameters: $\lambda_1 = 0$, $\lambda_2 = 0$, $\lambda_3 = 10^{-10}$, $\lambda_4 = 10^{-9}$, $\lambda_5 = 10^{-8}$, $\lambda_6 = 10^{-6}$, and $\mu = 0$.

\subsection{Evaluation Protocol}
\label{app:Metrics}

We have evaluated \modelName, by following the standard evaluation setting for KB completion as described by \citep{BoxE,ExpressivE,BoxEL}. In particular, this includes measuring the ranking quality of each role assertion $R(a, b)$ in the test set over any possible individual in the first position of the assertion, i.e., $R(a', b)$ for all $a' \in \NI$, and the second position of the assertion, i.e., $R(a, b')$ for all $b' \in \NI$. Furthermore, we used the standard metrics for KB completion, namely, the mean reciprocal rank (MRR) and hits at k (H@k). 

As typically done in the literature \citep{TransE,RotatE,BoxE,ExpressivE,BoxEL}, we presented the filtered versions of these merics introduced by \citep{TransE}. This means in particular that for hyperparameter tuning, we evaluated each of the found \modelName embedding solutions on the validation set and excluded any assertion from the ranking that occurs in the train or validation set (apart from the validation assertion whose score shall be computed). We selected those embedding solutions that reached the highest scores on the validation set. We followed the filtered setting \citep{TransE} also during the final evaluation, i.e., we evaluated the selected embedding solutions on the test set and excluded any assertion from the ranking that occurs in the train, validation, or test set (apart from the test assertion whose score shall be computed). 

The intuition of the filtered setting is that we exclude assertions from the ranking that are during the current evaluation stage known to be true, as assigning a high score to these assertions does not indicate a wrong inference. Specifically, during hyperparameter tuning on the validation set, the train and validation assertions are known to be true and thus need to be excluded; while during the final evaluation on the test set, the train, validation, and test assertions are known to be true and thus need to be excluded from the ranking. Finally, we briefly review the definition of H@k and the MRR: H@k reflects the proportion of true assertions within the predicted assertions whose rank is at most $k$, whereas the MRR represents their average of inverse ranks ($1/\textit{rank}$).

\subsection{Running Time}
\label{app:time}

For each choice of the hyperparameters $\lambda_1,\lambda_2,\lambda_3$, the evaluation of \modelName takes less than a minute for F\_v1 and
F\_v2. It takes 2 minutes and 38 seconds for F\_v3 and more than 20 minutes for F\_v4. Improving our evaluation procedure would contribute for  scalability. 
\begin{table}[h!]
\centering
\begin{tabular}{cccc}
\toprule
Dataset &  BoxE  & RotatE & ComplEx\\
\midrule
F\_v1        &  8m38s  &  5m56s  &  7m19s\\
F\_v2        &  8m40s   &  4m39s &  8m37s\\
F\_v3        &  44m44s   &   4m24s  & 8m30s\\
F\_v4        &  34m37s   &   13m14s & 16m27s\\
%F\_v5        &  8h   &   815s  \\
\bottomrule
\end{tabular}
\caption{Training time  split by dataset for SGD methods.}
\label{tab:timeRessgd}
\end{table}

Regarding SGD methods, 
 we provide in Table~\ref{tab:timeRessgd} the training time for the SGD methods with $100$ trials and $500$ epochs for each method. We used the same dimensionality $32$ as in \modelName (see \ref{app:ExperimentSetup}).

 \subsection{SGD setup}
 \label{app:SGD_Hyperparameters}
 The SGD models shown in our experiments were trained using PyKEEN's implementation of BoxE, RotatE, and ComplEx. We trained each model on the different Family subsets for 100 trials over 500 epochs. The optimizer used is Adagrad, together with early stopping. The frequency, patience, and relative delta of the stopper are $10$, $10$, and $0.01$, respectively. Lastly, the embedding dimension, as in \modelName, is set to $32$. The remaining parameters are set to the defaults of the KGE implementations in PykEEN \citep{ali2021pykeen}. 

%\todo{I am not sure how this part should be} I put it in the main text

%\todo{...}
%\section{Nonconvexity in previous works}\label{sec:nonconvex}

% \newpage

% In what follows, we recall 
% The signed distance function to a convex cone is a sublinear function. Because of that,  convex function. 
% It is, in fact, a sublinear function 
% AsThe signed distance to the nonnpositive orthant is not natively implemented in CVXPY, but it we can still solve convex optimization problems involving the signed distance function using 
%However, we can implement. \todo{continue this }

\else
\fi

\end{document}